\pdfoutput=1
\documentclass[a4paper,fleqn]{cas-sc}
\usepackage[utf8]{inputenc}
\usepackage{url}
\usepackage{nicefrac}
\usepackage{microtype}
\usepackage{amsthm}
\usepackage{multirow}
\usepackage{algorithm}
\usepackage{algorithmic}
\usepackage[authoryear]{natbib}
\usepackage{placeins}
\hypersetup{hypertexnames=false,bookmarksopen=true,bookmarksdepth=3,pdfpagemode=UseOutlines}
\DeclareUnicodeCharacter{2019}{'}

\makeatletter
\newif\ifappendixtoc
\let\appendixtoc@addcontentsline\addcontentsline
\renewcommand{\addcontentsline}[3]{%
  \appendixtoc@addcontentsline{#1}{#2}{#3}%
  \ifappendixtoc
    \begingroup
      \def\appendixtoc@target{#1}%
      \def\appendixtoc@toc{toc}%
      \ifx\appendixtoc@target\appendixtoc@toc
        \appendixtoc@addcontentsline{apc}{#2}{#3}%
      \fi
    \endgroup
  \fi
}
\newcommand{\appendixcontentsname}{Appendix Contents}
\newcommand{\appendixtableofcontents}{%
  \section*{\appendixcontentsname}%
  \begingroup
    \setcounter{tocdepth}{3}%
    \@starttoc{apc}%
  \endgroup
}
\newcommand{\startappendixtoc}{\appendixtoctrue}
\newcommand{\stopappendixtoc}{\appendixtocfalse}
\makeatother
\ExplSyntaxOn
\cs_if_exist:NF \vbox_unpack_clear:N
  { \cs_new_eq:NN \vbox_unpack_clear:N \vbox_unpack_drop:N }
\RenewDocumentCommand \printorcid { }
  {
    \seq_if_empty:NF \g_stm_orcid_seq
      {
        \group_begin:
          \tex_let:D \thefootnote \relax \footnotetext
            {
              \raggedright
              \textsc{orcid}(s):\c_space_token
              \seq_use:Nn \g_stm_orcid_seq { ;~ }
            }
        \group_end:
      }
  }
\cs_set:Npn \__first_footerline:
  { \rmfamily \itshape Preprint }
\cs_set:Npn \__cas_head: { }
\ExplSyntaxOff

\def\mr{\mathrm}
\def\bm{\boldsymbol}

\def\be{\begin{equation}}
\def\ee{\end{equation}}
\def\ba{\begin{aligned}}
\def\ea{\end{aligned}}
\def\bp{\begin{pmatrix}}
\def\ep{\end{pmatrix}}

\newtheorem{theorem}{Theorem}

\newtheorem{proposition}{Proposition}

\newtheorem{assumption}{Assumption}

\begin{document}
\let\WriteBookmarks\relax
\def\floatpagepagefraction{1}
\def\textpagefraction{.001}
\shorttitle{A Theory-grounded Hybrid Neural Network for Stable Visual Object Tracking}
\shortauthors{Zhou et al.}
\title[mode=title]{A Theory-grounded Hybrid Neural Network Integrating Complementary Estimation Mechanisms for Stable Visual Object Tracking}

\author[1,3]{Yancheng Zhou}[orcid=0009-0008-7537-3238]
\credit{Conceptualization, Methodology, Investigation, Data curation, Validation, Formal analysis, Writing - original draft}

\author[1]{Hanle Zheng}[orcid=0009-0002-9622-780X]
\credit{Methodology, Investigation}

\author[1]{Lei Deng}[orcid=0000-0002-5172-9411]
\credit{Conceptualization, Methodology, Resources, Writing - original draft}

\author[1,2]{Yujie Wu}[orcid=0000-0002-8950-6803]
\cormark[1]
\ead{wu-yj16@tsinghua.org.cn}
\credit{Conceptualization, Methodology, Project administration, Writing - review \& editing, Supervision}

\affiliation[1]{
  organization={Center for Brain Inspired Computing Research (CBICR), Department of Precision Instrument, Tsinghua University},
  city={Beijing},
  country={China}
}

\affiliation[2]{
  organization={Department of Computing, The Hong Kong Polytechnic University},
  city={Hong Kong},
  country={China}
}

\affiliation[3]{
  organization={Weixian College, Tsinghua University},
  city={Beijing},
  country={China}
}

\cortext[1]{Corresponding author}
\begin{abstract}
Hybrid neural networks (HNNs) that integrate artificial neural networks (ANNs) with brain-inspired neural networks have achieved broad success across perception and control tasks. However, much of the current success is confined to neuron-scale hybridization, where discrete, spike-based coding fundamentally limits applicability to continuous-state estimation tasks.
In neuroscience, continuous attractor neural networks (CANNs) represent continuous states through neural ensembles, pointing to a population-scale route for HNNs to address this limitation. 
Yet, principled methodologies for ANN--CANN integration remain largely underexplored.
In this work, we propose a theory-grounded ANN-CANN hybridization framework and instantiate it as a hybrid tracking neural network (HTNN) for visual object tracking, a representative continuous-state estimation task.
The framework aligns ANN response maps with CANN dynamics in the same state space, enabling the two heterogeneous branches to interact through the shared state representation. 
Furthermore, we uncover a functional bias-variance complementarity: data-driven ANNs provide asymptotically unbiased estimates, while CANN estimates are low-variance but temporally lagged. 
By operationalizing this complementarity, HTNN achieves stable and accurate tracking across nine visual tracking benchmarks, consistently outperforming single-network baselines and existing hybrid models.
Notably, these performance gains are robustly maintained even under diverse environmental variations, including occlusion, motion blur, and background interference. 
Through this proof-of-concept study, our framework offers a generalizable foundation for advancing HNNs toward population-scale hybridization.

\end{abstract}
\maketitle
\hypersetup{
  pdftitle={A Theory-grounded Hybrid Neural Network Integrating Complementary Estimation Mechanisms for Stable Visual Object Tracking},
  pdfauthor={Yancheng Zhou, Hanle Zheng, Lei Deng, Yujie Wu},
  pdfkeywords={Hybrid neural networks, Brain-inspired computing, Continuous attractor dynamics, Neuromorphic computing, Visual object tracking}
}

\section{Introduction}
The integration of computer-science-oriented and neuroscience-oriented computational models has long been a pivotal direction in the development of intelligent systems~\citep{AIandBI2,AIandBI3}.
Hybrid Neural Networks (HNNs) instantiate this direction by integrating Artificial Neural Networks (ANNs) with brain-inspired computational models within unified architectures~\citep{HNN1_Tianjic, HNN14_Review}.
Current HNN research has mainly followed a neuron-scale hybridization route, typically integrating ANNs with network architectures based on neuronal spiking dynamics, such as spiking neural networks (SNNs)~\citep{HNN11_Framework,HNN12_SNNRNN}.
Through the exploitation of their complementary advantages in accuracy, robustness, and energy efficiency, this route has achieved broad success across perception, control, and neuromorphic sensing tasks~\citep{HNN5_DeepLNeurodynamics,HNN6_SpikeFlowNet,HNN13_TopDownAttention,HNN16_DanceSNNANN,HNN18_STAttention}.
However, this paradigm struggles with tasks requiring continuous-state estimation, because continuously evolving variables are difficult to maintain through its discrete spiking mechanisms, as shown in Figure~\ref{fig:intro_motivation}(A)~\citep{Compare_SNN_and_RNN,HNN11_Framework}.

Continuous attractor neural networks (CANNs) model an important neural mechanism for continuous-state representation in neuroscience.
Through attractor dynamics shaped by local excitation and global inhibition, CANNs form a continuous family of bump states whose centers can move smoothly within the represented state space~\citep{SiWuCANN1}.
CANNs have been widely used to model continuous variables and cognitive states~\citep{Brain_Attractor}, including head direction~\citep{CANN_in_Head_Cells,CANN_in_Head_Cells_2}, spatial location~\citep{CANN_in_Grid_Cells,CANN_in_Grid_Cells2}, and working memory~\citep{CANN_in_Working_Memory,CANN_in_Working_Memory_2}.
Computational studies of CANNs further show that time-varying external inputs can drive activity bumps to track continuously evolving targets~\citep{SiWuCANN2,SiWuCANN3}.
Building on this property, CANNs have also been applied to visual object tracking (VOT), a representative continuous-state estimation task, and demonstrated that population-bump dynamics can support smooth visual target tracking~\citep{VOT_CANN}.
Together, this evidence motivates a population-scale hybridization route: integrating CANNs with ANNs to enable both data-driven, accurate estimation and dynamical maintenance of continuous states, as illustrated in Figure~\ref{fig:intro_motivation}(B).

Despite this promise, a critical bottleneck remains: effectively bridging these two heterogeneous models to exploit their complementary strengths. 
While recent heuristic attempts have incorporated CANNs into HNN designs~\citep{HNN4_FlyNet,HNN19_NeuroGPR}, their simplistic serial or parallel designs overlook the structural heterogeneity between CANNs and other neural modules. 
Equally important, none of them provided the theoretical analysis to ground the hybridization mechanism.

\begin{figure}[pos=!htbp]
    \centering
    \includegraphics[width=1.0\linewidth]{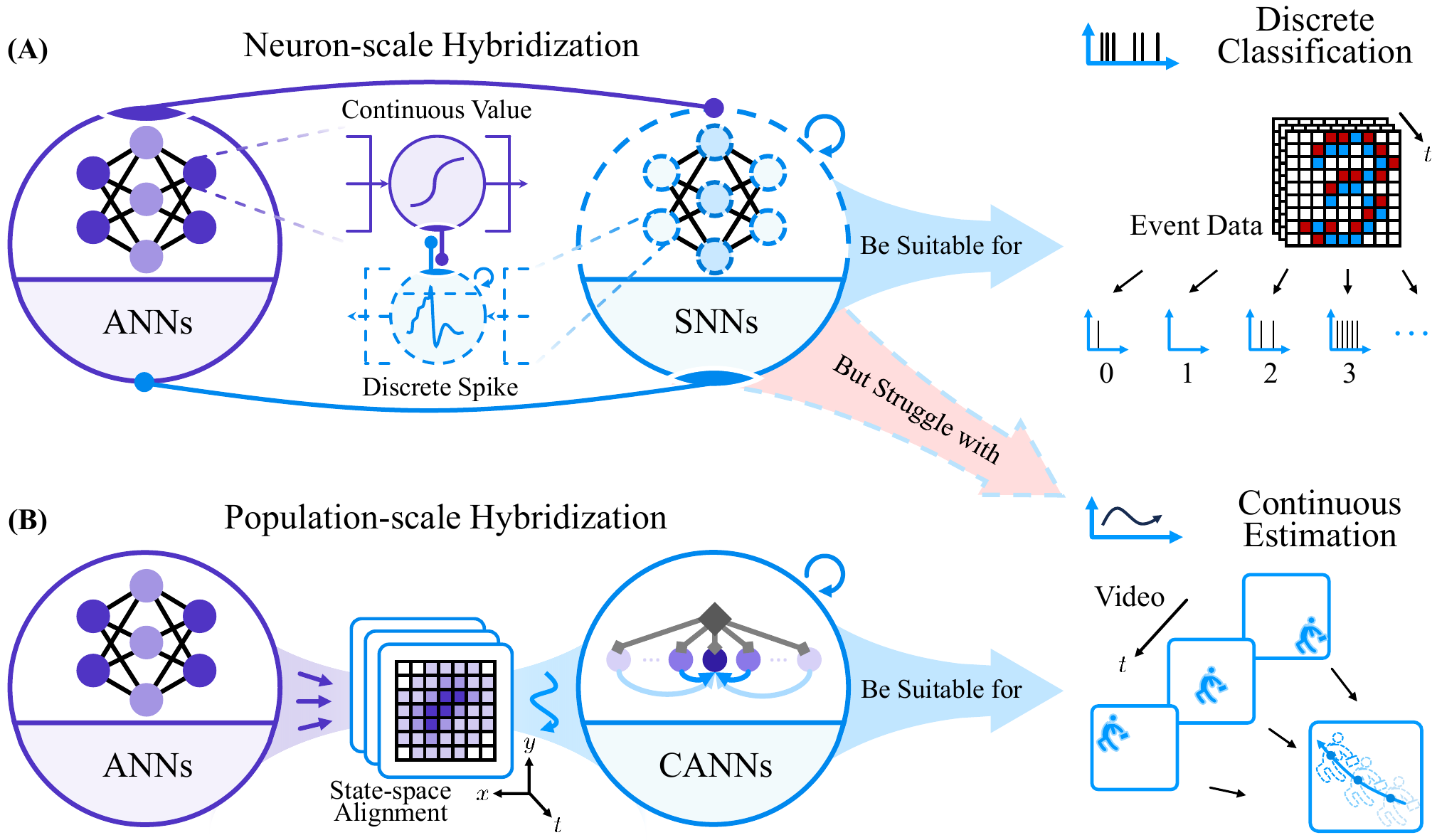}
    \caption{
    Motivation for population-scale hybridization in HNNs.
    (A) Neuron-scale hybridization integrates ANNs with architectures based on neuronal spiking dynamics, such as SNNs. 
    This route is suited to spatiotemporal tasks with discrete, event-driven data flow.
    (B) Population-scale hybridization aligns ANNs with neural population dynamics for continuous-state representation, such as CANNs, in the state space. 
    This route is suited to maintaining and estimating states that evolve continuously in space and time, such as tracking moving visual targets.
    }
    \label{fig:intro_motivation}
\end{figure}

Here, we present a theory-grounded framework that constructs hybrid tracking neural networks (HTNNs) by synergistically hybridizing an ANN and a CANN for VOT.
The framework aligns ANN- and CANN-based representations, characterizes their complementary estimation benefits, and realizes these benefits through the two-stage fusion design of the HTNN.
The ANN branch is instantiated as a classical and compact fully-convolutional Siamese Neural Network (SiamFC)~\citep{VOT_SiamFC}.
This choice preserves the ANN branch's target-localization capability while keeping the experimental focus on the fusion mechanism itself.
Systematic experiments show that HTNN stably and accurately tracks targets, outperforming single-network baselines, existing HNN baselines, and an ANN--CANN Direct Fusion baseline across multiple metrics.
These advantages remain consistent under challenging environmental variations, demonstrating the potential of HTNN for open-world object tracking applications.

The main contributions of this work are threefold.
\begin{itemize}

\item We propose a novel population-scale hybrid framework that integrates an ANN and a CANN for continuous-state estimation, using the VOT as a representative testbed. 
By formulating the two modules as complementary estimators over a shared state space, the framework establishes a principled design that naturally connects representation alignment, functional role characterization, and synergistic fusion of two heterogeneous branches.  

\item We mathematically characterize the complementary advantages of this hybrid framework from a bias--variance perspective.
Specifically, we demonstrate that the data-driven ANN branch provides an asymptotically unbiased estimate of the current state, while the CANN branch provides a lower-variance estimate with finite temporal lag.
Extending this to continuous-time dynamics reveals that the CANN branch supports stable tracking by continuous evolution and suppressing remote response peaks.
Building upon this insight, we develop an HTNN model with a two-stage architecture: a representation-fusion stage that refines ANN response maps using motion cues, and an estimation-fusion stage that additively combines both branches for robust final state estimation.

\item We comprehensively evaluate HTNN on nine visual tracking benchmarks comprising over 1,500 video sequences and 1.57M frames. 
Across 18 challenging environmental variations, HTNN outperforms single-network and HNN baselines in accuracy and stability. 
Compared with the ANN--CANN Direct Fusion baseline, HTNN achieves overall average gains of 77.3\% in Precision and 56.4\% in Success Rate.
Drift analyses indicate that HTNN obtains these empirical gains by exploiting the bias--variance complementarity between the two branches.
Visualizations of response maps and trajectories further provide qualitative evidence that HTNN suppresses unstable response fluctuations and maintains coherent target-state evolution.

\end{itemize}

\section{Related Work}

\subsection{Hybrid Neural Networks}

HNNs have emerged as a representative paradigm for integrating computer-science-oriented artificial neural networks with neuroscience-oriented brain-inspired models. 
Their central goal is to coordinate structurally and functionally heterogeneous modules within unified architectures, allowing different models to contribute complementary computational roles~\citep{HNN14_Review}. 
Existing HNN research has mainly followed a neuron-scale hybridization route.
In this paradigm, data-driven ANNs provide accurate continuous-valued feature representations.
They are integrated with architectures based on neuronal spiking dynamics at the levels of information coding, neural dynamics, and network design, exploiting complementary advantages in accuracy, robustness, and energy efficiency~\citep{HNN11_Framework, HNN12_SNNRNN}.
This route has been developed across system architectures~\citep{HNN2_SystemHiearchy}, general design and learning frameworks~\citep{HNN11_Framework}, biologically inspired neural modules~\citep{HNN5_DeepLNeurodynamics,HNN8_NonSpikeInRNN}, and neuromorphic hardware platforms~\citep{HNN1_Tianjic}.
It has also supported applications in event-based optical flow~\citep{HNN6_SpikeFlowNet}, visual attention~\citep{HNN13_TopDownAttention}, low-power visual perception~\citep{HNN15_SNNANN}, action recognition~\citep{HNN17_ReSpike}, and event-based object detection~\citep{HNN18_STAttention}. 
Yet the discrete nature of this paradigm makes it less naturally matched to continuous-state estimation, where the represented state must evolve continuously over time~\citep{Compare_SNN_and_RNN,HNN11_Framework}. 

This limitation motivates population-scale hybridization as an important direction for extending HNNs.
ANN--CANN hybridization provides a suitable starting point along this direction.
Early studies such as FlyNet and NeuroGPR have integrated CANNs with ANNs or other network modules through serial or parallel architectures for visual or robotic place recognition~\citep{HNN4_FlyNet,HNN19_NeuroGPR}. 
However, these heuristic architectures overlook the structural heterogeneity between CANNs and other network modules, and they lack a theoretical analysis that grounds the hybridization mechanism.
To address these issues, our work develops a theory-grounded ANN--CANN hybridization framework that aligns ANN and CANN representations in a shared state space, characterizes their complementary estimation roles, and realizes functional synergy through the two-stage fusion design of HTNN.

\subsection{Continuous Attractor Neural Networks}

CANNs provide a canonical neural population-dynamics mechanism for continuous-state representation~\citep{CANN_is_Great,Brain_Attractor}. 
Through recurrent local excitation and global inhibition, a CANN can stably maintain a localized activity bump whose center encodes a low-dimensional state~\citep{SiWuCANN1}. 
The classical form of this dynamical structure is detailed later in Section~\ref{sec:2d-cann}. 
This mechanism has been widely used to model continuous variables and cognitive states in neural systems, including head direction~\citep{CANN_in_Head_Cells,CANN_in_Head_Cells_2,CANN_in_Head_Cells_3}, spatial location~\citep{CANN_in_Grid_Cells,CANN_in_Grid_Cells2,CANN_in_Grid_Cells3}, and working memory~\citep{CANN_in_Working_Memory,CANN_in_Working_Memory_2,CANN_in_Working_Memory_3}. 

Beyond these representational accounts, computational studies of CANN tracking dynamics have shown that time-varying external inputs can drive activity bumps to track continuously evolving stimuli, with the resulting dynamics shaped by response delay and bump deformation~\citep{SiWuCANN2,SiWuCANN3}. 
Under appropriate input conditions, such as weak stimulus strength and a small offset from the bump center, translational modes dominate CANN dynamics and drive the bump smoothly toward the target, whereas deformation-related modes decay exponentially over time~\citep{SiWuCANN2}. 
This tracking capability has motivated applications of CANNs to VOT as a representative continuous-state estimation task. 
For example, a two-dimensional CANN implemented on a many-core neural network chip can track visual targets rapidly and smoothly~\citep{VOT_CANN}, and a Tianjic-based unmanned bicycle system further uses CANN dynamics for real-time following of a moving human in outdoor scenes~\citep{HNN1_Tianjic}. 

Despite these successful applications, how to integrate the advantages of CANNs with data-driven ANNs into a unified hybrid framework remains an open question and is still largely underexplored.

\subsection{Visual Object Tracking}

VOT is a representative continuous-state estimation task. 
Given an initial annotation in the first frame, a tracker must maintain the evolving target state from video, such as location, scale, or bounding box~\citep{VOT_Review}. 
VOT methods have evolved from correlation-filter-based trackers~\citep{VOT_Corr1,VOT_Corr2} to deep Siamese trackers~\citep{VOT_SiamFC,VOT_Siam1,VOT_Siam2} and Transformer-based trackers~\citep{VOT_Transformer1,VOT_Transformer2,VOT_Transformer3}. 
Among deep Siamese trackers, SiamFC established fully convolutional Siamese networks as a representative paradigm for generic object tracking~\citep{VOT_SiamFC}. 
Its simplicity and representative design make SiamFC a suitable carrier for evaluating our ANN--CANN hybridization framework: it preserves the discriminative target-state estimation capability of an ANN while reducing the influence of task-irrelevant architectural factors.
Driven by applications in embodied perception~\citep{VOT_Embodied}, robotics~\citep{VOT_Robot}, and autonomous systems~\citep{Dataset_UAV}, VOT benchmarks have also expanded toward long-term sequences, open-world categories, and diverse environmental variations~\citep{Dataset_TrackingNet,Dataset_LaSOT,Dataset_GOT-10k,Dataset_NfS}. 

While existing VOT models have advanced discriminative visual estimation through stronger matching, representation, or attention mechanisms, the explicit dynamical maintenance of continuously and stably evolving target states remains less systematically explored. 
Our ANN--CANN hybridization framework addresses this gap by aligning visual response maps with brain-inspired attractor-based state dynamics, providing a theory-grounded route toward accurate and temporally stable tracking under environmental variations.

\section{Preliminaries}

\subsection{A Canonical 2D Model of CANN}
\label{sec:2d-cann}

The defining dynamical properties of CANNs are translationally invariant recurrent interactions, local excitation, and global inhibition, which give rise to a continuous family of bump states and enable stable tracking of external stimuli. 
Here, we introduce the canonical 2D CANN model established in classical works~\citep{SiWuCANN1,SiWuCANN2}:

\paragraph{2D continuous dynamics.}
Let \( \bm x=(x,y)\in\mathbb R^2 \) denote a coordinate in the two-dimensional state space.
Denote by \( U(\bm x,t) \) the synaptic input at time \( t \) to neurons whose preferred state is \( \bm x \), and by \( r(\bm x,t) \) the firing rate of these neurons.
The standard 2D-CANN dynamics can be written as
\begin{equation}
\label{equ:CANN-2d-general}
\begin{aligned}
\tau_c \frac{\partial U(\bm x,t)}{\partial t}
&=
-U(\bm x,t)
+
I_{\mathrm{ext}}(\bm x,t)
+
\rho_{\mathrm{n}} \int_{\mathbb R^2} J(\bm x-\bm x')\, r(\bm x',t)\, \mathrm d\bm x', \\
r(\bm x,t)
&=
\frac{U^2(\bm x,t)}
{1 + k\rho_{\mathrm{n}} \int_{\mathbb R^2} U^2(\bm x',t)\, \mathrm d\bm x'}, \\
J(\bm x-\bm x')
&=
\frac{A}{2\pi a^2}
\exp\!\left(
-\frac{\lVert \bm x-\bm x' \rVert^2}{2a^2}
\right),
\end{aligned}
\end{equation}
where \( \tau_c\) is the time constant, \( I_{\mathrm{ext}}(\bm x,t) \) is the external input, \( \rho_{\mathrm{n}} \) is the neural density, \( k \) controls the strength of divisive normalization, and \( J(\bm x-\bm x') \) is a translationally invariant Gaussian interaction kernel with amplitude \( A \) and width \( a \). 
This 2D model admits a theoretically tractable analysis of bump dynamics and tracking behavior.

\paragraph{Stationary bump family.}
When the external input vanishes, i.e., \( I_{\mathrm{ext}}(\bm x,t)=0 \), the network admits a continuous family of stationary bump states indexed by any center \( \bm z\in\mathbb R^2 \):
\begin{equation}
\label{equ:CANN-2d-solution}
\begin{aligned}
U(\bm x \mid \bm z)
&=
U_0
\exp\!\left(
-\frac{\lVert \bm x-\bm z \rVert^2}{4a^2}
\right), \\
r(\bm x \mid \bm z)
&=
r_0
\exp\!\left(
-\frac{\lVert \bm x-\bm z \rVert^2}{2a^2}
\right),
\end{aligned}
\end{equation}
where \( U_0 \) and \( r_0 \) are constants determined by the network parameters~\citep{SiWuCANN2}. 
Because these stationary states exist for arbitrary \( \bm z \), the network can continuously represent target locations in the 2D state space while preserving a stable Gaussian bump profile under moderate perturbations.
External inputs can drive the bump to translate along the continuous family of stationary states, whereas non-translational modes that deform the bump shape decay over time.
This property underlies the reduced center dynamics used later in our theoretical analysis~\citep{SiWuCANN1,SiWuCANN2}.

\paragraph{Discrete implementation on a toroidal grid.}
For practical computation, the 2D continuous network is discretized onto an \( N\times N \) toroidal grid, yielding matrix-valued states \( \bm U(t) \) and \( \bm r(t) \). 
The discrete network dynamics are then given by
\begin{equation}
\label{equ:CANN-discrete-update}
\begin{aligned}
\tau_c \frac{\bm U(t+\Delta t)-\bm U(t)}{\Delta t}
&=
-\bm U(t)
+
\bm I_{\mathrm{ext}}(t)
+
\bm J \circledast \bm r(t), \\
\bm r(t)
&=
\frac{\bm U(t)\odot \bm U(t)}
{1 + k \lVert \bm U(t)\rVert_F^2}, \\
\bm J(\bm p,\bm p')
&=
\frac{A}{2\pi a^2}
\exp\!\left(
-\frac{\mathrm{dis}(\bm p,\bm p')^2}{2a^2}
\right),
\end{aligned}
\end{equation}
where \( \Delta t \) is the discrete time step, \( \odot \) denotes element-wise multiplication, \( \lVert\cdot\rVert_F \) is the Frobenius norm, \( \circledast \) denotes circular convolution, and \( \mathrm{dis}(\bm p,\bm p') \) is the toroidal distance between grid locations.

\subsection{Fully-convolutional Siamese Neural Networks}
\label{sec:prelim_siamfc}

\begin{algorithm}[!htbp]
\caption{SiamFC-style online inference at frame \(k\)}
\label{alg:siamfc_inference}
\begin{algorithmic}[1]
\REQUIRE Current frame \( \bm V^{(k)} \), template \( \bm V_{\mr{temp}}^{(0)} \), previous target center \( \bm c^{(k-1)} \), previous target size \( \bm b^{(k-1)} \), scale-index set \( \mathcal S \)
\ENSURE Updated target center \( \bm c^{(k)} \), updated target size \( \bm b^{(k)} \)
\FOR{each scale index \( s\in\mathcal S \)}
    \STATE Construct the scaled search image \( \bm V_{\mr{search},s}^{(k)} \) according to Eq.~\eqref{eq:siamfc_scaled_search_region}
    \STATE Compute the response map \( \bm S_s^{(k)} \) according to Eq.~\eqref{eq:siamfc_scaled_response}
    \STATE Compute the penalized peak score \( \pi_s \max_{(i,j)} \bm S_s^{(k)}(i,j) \)
\ENDFOR
\STATE Select the scale index \( s^{(k)} \) according to Eq.~\eqref{eq:siamfc_scale_select}
\STATE Estimate the response peak \( (i^{(k)},j^{(k)}) \) from \( \bm S_{s^{(k)}}^{(k)} \)
\STATE Convert the response peak into image-coordinate displacement \( \bm d^{(k)} \)
\STATE Update the target center according to Eq.~\eqref{eq:siamfc_center_update}
\STATE Update the target size according to Eq.~\eqref{eq:siamfc_scale_update}
\RETURN \( \bm c^{(k)} \), \( \bm b^{(k)} \)
\end{algorithmic}
\end{algorithm}

SiamFC~\citep{VOT_SiamFC} is a seminal tracker that uses Siamese matching for visual object tracking. It achieves substantial performance gains while maintaining real-time speed.
Given the target template \( \bm V_{\mr{temp}}^{(0)} \) extracted from the first frame and a search image from the \(k\)-th frame, SiamFC uses a shared convolutional backbone \( \varphi(\cdot) \) to extract their feature representations and computes a response map $\bm S^{(k)}$ by cross-correlation.
For a search image \( \bm V_{\mr{search}}^{(k)} \), the response map is defined as
\begin{equation}
    \bm S^{(k)}
    =
    \varphi\!\left(\bm V_{\mr{temp}}^{(0)}\right)
    \star
    \varphi\!\left(\bm V_{\mr{search}}^{(k)}\right),
\end{equation}
where \( \star \) denotes cross-correlation.
Each element \( \bm S^{(k)}(i,j) \) measures the matching score between the target template and the candidate region corresponding to response-grid location \( (i,j) \).
The target center is estimated from the peak displacement of the response map.

Beyond the response-map computation, SiamFC employs an online inference rule that restricts the search region according to the previous prediction and evaluates multiple candidate scales.
Let
\[
\bm c^{(k-1)}
=
(x_{\mr{center}}^{(k-1)}, y_{\mr{center}}^{(k-1)})
\]
denote the previously predicted target center, and let
\[
\bm b^{(k-1)}
=
(w^{(k-1)}, h^{(k-1)})
\]
denote the previous target size.
At frame \(k\), the nominal search region is centered at \( \bm c^{(k-1)} \), and its spatial extent is chosen proportional to \( \bm b^{(k-1)} \):
\begin{equation}
\bm V_{\mr{search}}^{(k)}
=
\mr{Crop}\!\left(
\bm V^{(k)},
\bm c^{(k-1)},
\varsigma\,\bm b^{(k-1)}
\right),
\label{eq:siamfc_search_region}
\end{equation}
where \( \bm V^{(k)} \) is the current frame, \( \mr{Crop}(\bm V,\bm c,\bm b) \) extracts an image region centered at \( \bm c \) with spatial size specified by \( \bm b \), and \( \varsigma>1 \) is a context factor.

To handle scale changes, SiamFC evaluates multiple scaled search regions in parallel.
Let the scale-index set be
\begin{equation}
\mathcal S = \{-m,\dots,-1,0,1,\dots,m\},
\label{eq:siamfc_scale_set}
\end{equation}
where \(m\in\mathbb N\) denotes the number of scale levels on each side of the nominal scale.
Let \(\eta>1\) denote the scale step, so that the scale factor corresponding to \(s\in\mathcal S\) is \(\eta^s\).
For each scale index \(s\), the scale-specific search image is constructed as
\begin{equation}
\bm V_{\mr{search},s}^{(k)}
=
\mr{Resize}_{R_{\mr{search}}}\!\left(
\mr{Crop}\!\left(
\bm V^{(k)},
\bm c^{(k-1)},
\varsigma\,\eta^s\,\bm b^{(k-1)}
\right)
\right),
\label{eq:siamfc_scaled_search_region}
\end{equation}
where \(R_{\mr{search}}\) denotes the fixed search-input resolution of SiamFC, and \( \mr{Resize}_{R_{\mr{search}}}(\cdot) \) resizes the cropped region to this resolution.
The response map at scale index \(s\) is then computed as
\begin{equation}
\bm S_s^{(k)}
=
\varphi\!\left(\bm V_{\mr{temp}}^{(0)}\right)
\star
\varphi\!\left(\bm V_{\mr{search},s}^{(k)}\right).
\label{eq:siamfc_scaled_response}
\end{equation}

The scale is selected according to the penalized peak response:
\begin{equation}
s^{(k)}
=
\arg\max_{s\in\mathcal S}
\left[
\pi_s
\max_{(i,j)}
\bm S_s^{(k)}(i,j)
\right],
\label{eq:siamfc_scale_select}
\end{equation}
where \( \pi_s \) is the scale-penalty factor for scale index \(s\).
After selecting \(s^{(k)}\), the response peak
\begin{equation}
(i^{(k)},j^{(k)})
=
\arg\max_{(i,j)}
\bm S_{s^{(k)}}^{(k)}(i,j)
\end{equation}
is converted into an image-coordinate displacement by the standard SiamFC coordinate mapping:
\begin{equation}
\bm d^{(k)}
=
\mr{MapToImage}\!\left(i^{(k)},j^{(k)},s^{(k)}\right),
\end{equation}
where \( \mr{MapToImage}(\cdot) \) denotes the conversion from response-grid displacement to image-coordinate displacement under the selected scale.
The target center is then updated as
\begin{equation}
\bm c^{(k)}
=
\bm c^{(k-1)}+\bm d^{(k)}.
\label{eq:siamfc_center_update}
\end{equation}
The target size is updated by damped interpolation:
\begin{equation}
\bm b^{(k)}
=
(1-\zeta)\bm b^{(k-1)}
+
\zeta\,\eta^{s^{(k)}}\bm b^{(k-1)},
\label{eq:siamfc_scale_update}
\end{equation}
where \( \zeta\in(0,1] \) is a damping factor.
Equivalently, the two components of \( \bm b^{(k)}=(w^{(k)},h^{(k)}) \) are updated as
\[
w^{(k)}
=
(1-\zeta)w^{(k-1)}
+
\zeta\,\eta^{s^{(k)}}w^{(k-1)},
\qquad
h^{(k)}
=
(1-\zeta)h^{(k-1)}
+
\zeta\,\eta^{s^{(k)}}h^{(k-1)}.
\]

Algorithm~\ref{alg:siamfc_inference} summarizes the SiamFC-style online inference procedure used in this paper.
For controlled comparison, the same online inference rule is used in the experimental evaluation.

From the perspective of ANN--CANN hybridization, SiamFC serves as a representative data-driven ANN estimator.
It infers the current target state from appearance matching and produces a discriminative response map over candidate target locations.
This frame-wise estimation mechanism provides accurate current-state estimation, while leaving temporally continuous state evolution to be modeled by the CANN branch.

\section{Method}

Here, we introduce a systematic ANN--CANN hybridization framework.
In Section~\ref{sec:framework_setting}, we first formulate VOT as a two-dimensional continuous-state estimation problem and align ANN and CANN branches over the same target-center state space.
In Section~\ref{subsec:bias_variance_complementarity}, we characterize the bias-variance complementarity between the two branches in an idealized one-step setting.
Furthermore, Section~\ref{subsec:mse_of_additive_hybridization} shows how convex additive fusion can utilize this complementarity to achieve a lower estimation error, thereby realizing functional synergy.
Section~\ref{subsec:continuous_time_response_map_dynamics} extends the analysis to continuous-time tracking.
In Section~\ref{sec:htnn_architecture}, we present HTNN as a concrete tracking architecture that combines a SiamFC-based ANN branch with a training-free CANN branch and implements the theoretical framework through representation-fusion and estimation-fusion.

\subsection{State-Space Alignment of ANNs and CANNs}
\label{sec:framework_setting}

\begin{figure}[pos=!htbp]
    \centering
    \includegraphics[width=1.0\linewidth]{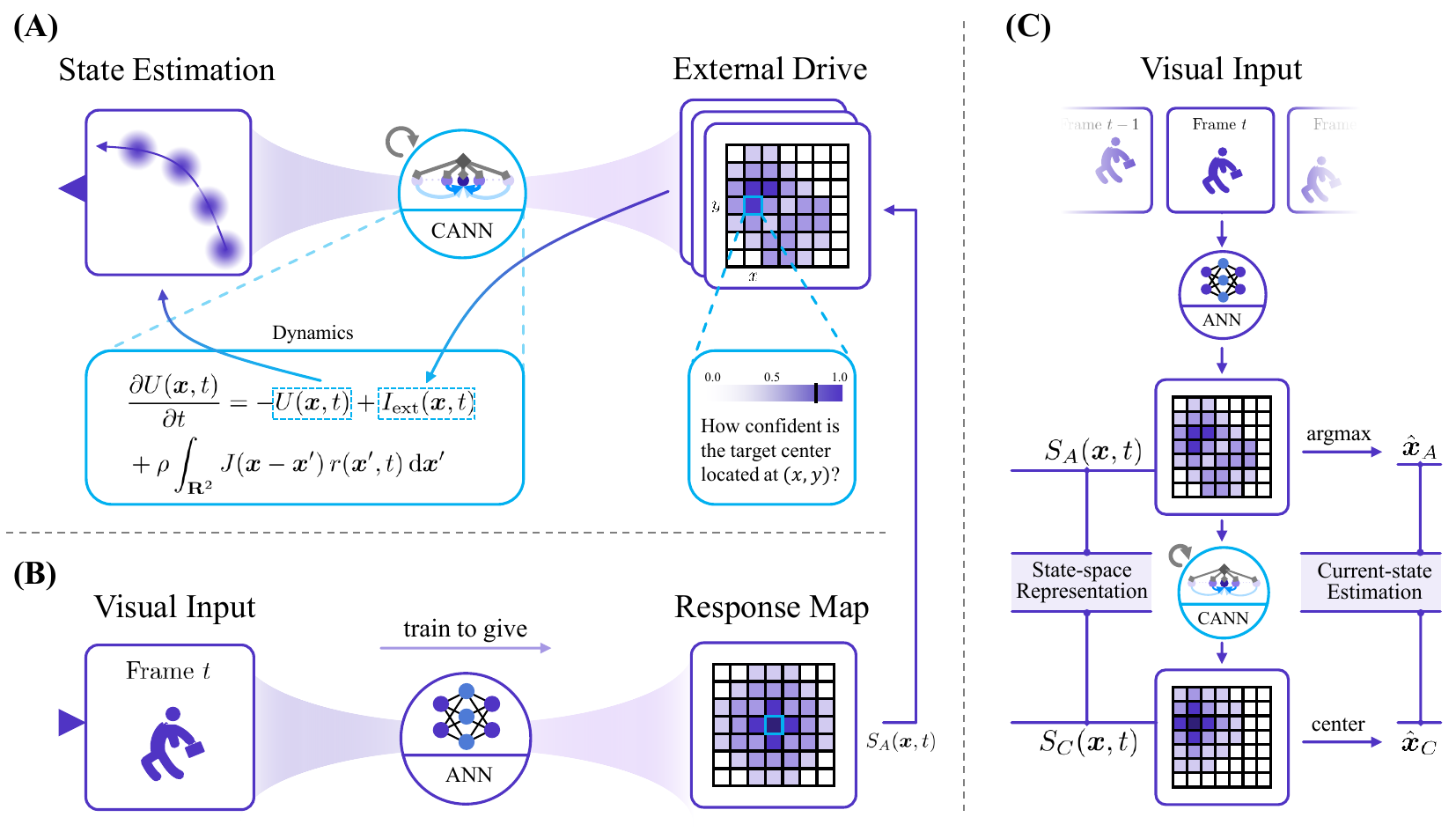}
    \caption{
    State-space alignment of ANNs and CANNs.
    (A) The CANN represents the target state through an activity bump of \(U(\bm x,t)\), whose evolution is driven by the external input \(I_{\mathrm{ext}}(\bm x,t)\).
    The value of \(I_{\mathrm{ext}}(\bm x,t)\) indicates how strongly candidate state \(\bm x\) is supported by the current visual evidence.
    (B) The ANN maps visual input to a response map \(S_A(\bm x,t)\) over candidate target states, where each value scores how likely the target center is to be located at \(\bm x\).
    (C) The aligned framework organizes the ANN and CANN branches over the same target-state space.
    The ANN response map \(S_A(\bm x,t)\) serves as the CANN external drive, while both \(S_A(\bm x,t)\) and the CANN activity bump, denoted as \(U(\bm x,t)\), can be used as state-space representations for current-state estimation and fusion before final decoding.
    }
    \label{fig:framework_setting}
\end{figure}

We first formulate VOT as the problem of estimating an evolving target state from a video stream.
In this work, we focus on the target center \(\bm x^*(t)\in\mathbb R^2\) as the continuous state in the two-dimensional state space.
The tracker is expected to estimate this state as the target moves over time.
As described by the 2D-CANN dynamics in Eq.~\eqref{equ:CANN-2d-general} and illustrated in Figure~\ref{fig:framework_setting}(A), the CANN branch represents a continuous state through an activity bump of \(U(\bm x,t)\).
To make this bump follow a moving target, the CANN requires an external input \(I_{\mathrm{ext}}(\bm x,t)\) to drive its evolution as the target state changes.
At time \(t\), the value of \(I_{\mathrm{ext}}(\bm x,t)\) should indicate how likely the target center is to be at candidate state \(\bm x\) according to the current visual evidence.
With this drive, the activity bump is guided toward the target-centered region, so that the bump center can be decoded as the CANN estimate of the current state.
Therefore, combining an ANN with a CANN requires the visual branch to provide evidence over the same target-state space.

The ANN branch can directly learn from visual data to predict the target center for each input frame.
To make this prediction compatible with the CANN input, the ANN can also be trained to output a response map \(S_A(\bm x,t)\) over candidate target states.
Each value of \(S_A(\bm x,t)\) lies in \([0,1]\) and indicates how likely the target center is to be at candidate state \(\bm x\) at time \(t\).
Through supervised learning with a Gaussian-shaped target response centered at the annotated target center, \(S_A(\bm x,t)\) can be optimized to be close to \(1\) near the ground-truth state and close to \(0\) at distant states, as shown in Figure~\ref{fig:framework_setting}(B).
In the basic aligned formulation, this response map can be directly used as the CANN external input, \(I_{\mathrm{ext}}(\bm x,t)=S_A(\bm x,t)\), allowing the ANN-derived visual evidence to drive the CANN bump as the target state changes.

Building on the two branches, Figure~\ref{fig:framework_setting}(C) illustrates how the ANN and CANN branches are organized within the aligned framework.
At each time \(t\), the ANN receives the current visual input and produces a response map \(S_A(\bm x,t)\), which serves as the CANN external drive \(I_{\mathrm{ext}}(\bm x,t)\).
The ANN response map can be decoded by \(\arg\max\) to obtain the ANN estimate \(\hat{\bm x}_A(t)\), while the CANN activity bump can be decoded by its center to obtain the CANN estimate \(\hat{\bm x}_C(t)\); these two estimates correspond to the same target state and can therefore be fused.
Meanwhile, the ANN response map \(S_A(\bm x,t)\) and the CANN activity bump \(U(\bm x,t)\) can both be viewed as state-space representations with time as an argument, namely functions from \(\mathbb R^2 \times \mathcal T\) to \(\mathbb R\), where \(\mathcal T\) denotes the time domain.
At each \((\bm x,t)\), their values indicate how strongly each branch supports candidate state \(\bm x\) at time \(t\), which makes them also compatible for fusion before decoding the final target state.
Together, this state-space alignment organizes ANN and CANN as two estimation branches for the same continuous target state, providing the basis for analyzing their functional complementarity and how fusion can exploit it.

\subsection{Bias--Variance Complementarity}
\label{subsec:bias_variance_complementarity}

\begin{figure}[pos=!htbp]
    \centering
    \includegraphics[width=1.0\linewidth]{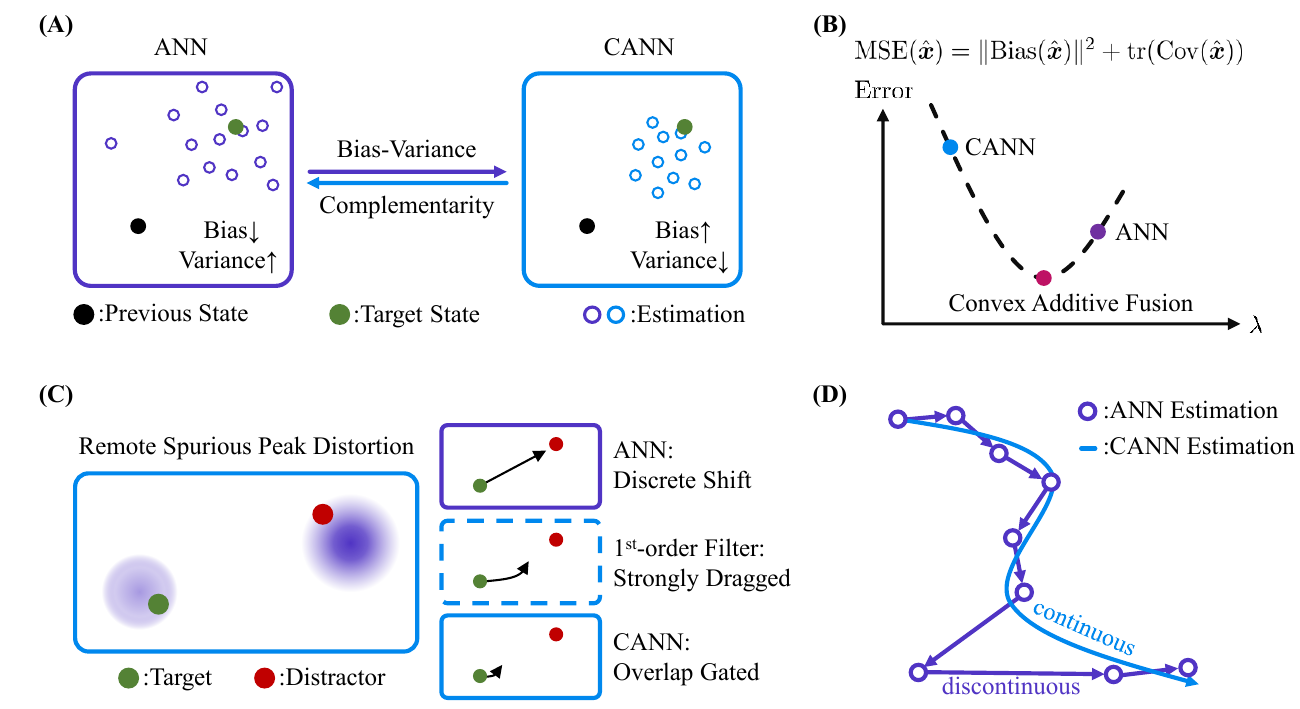}
    \caption{
    Intuition of bias--variance complementarity, convex additive fusion, and continuous-time tracking stability in ANN--CANN hybridization.
    (A) In the one-step setting, the ANN estimator provides an asymptotically unbiased current-state estimate but remains sensitive to response-map fluctuations, whereas the CANN estimator can reduce variance through attractor dynamics under suitable dynamical-parameter choices while introducing finite-response lag.
    (B) Under the bias--variance decomposition of MSE, convex additive fusion balances the complementary error profiles of the two estimators and can achieve a lower estimation error than either pure-estimator endpoint under the derived conditions.
    (C) Under a remote spurious peak distortion, ANN argmax estimation may jump to the distractor, and a first-order center filter driven by decoded coordinates can be strongly dragged toward the remote peak.
    In contrast, CANN dynamics receives the full response map, so the remote peak affects the CANN center through a spatial-overlap-gated velocity contribution and is attenuated when its overlap with the current activity bump is weak.
    (D) Even when ANN argmax estimates switch discontinuously between response peaks, the CANN center evolves through continuous-time dynamics, providing continuous trajectory-level state evolution.
    }
    \label{fig:theory_intuition}
\end{figure}

The proposed alignment bypasses direct structural coupling between the two heterogeneous modules.
However, structural heterogeneity still induces distinct functional roles in state estimation.
A non-trivial hybridization, therefore, requires characterizing their estimation properties and explaining how their advantages can be utilized synergistically.

To achieve this, in the following, we first formulate an idealized one-step setting and define the bias and covariance used to quantify their estimation error. After that, we analyze how the response-map peak of the ANN branch yields an asymptotically unbiased current-state estimate and how response-map fluctuations contribute to its variance. Similarly, for the CANN branch, we analyze how finite-time bump evolution introduces a finite-response bias while filtering response-map perturbations through attractor dynamics.
Finally, we characterize the complementary error structures of the two estimators as shown in Figure~\ref{fig:theory_intuition}(A), thereby identifying their distinct functional advantages.

\paragraph{One-step tracking setting.}
Because VOT videos are sampled by cameras at finite frame rates, the continuous evolution of a target is observed through discrete frame transitions.
We therefore start from an idealized one-step setting, which captures the elementary estimation problem faced by both branches at each frame transition.
Specifically, consider a target moving from the previous state \(\bm x_0\) to the current state \(\bm x_1=\bm x_0+\bm\Delta\), where \(\bm x_0,\bm x_1,\bm\Delta\in\mathbb R^2\), and define \(d:=\|\bm\Delta\|^2\).
The ANN and CANN branches are both treated as estimators of the current state \(\bm x_1\) under this single-step transition.
Expectations, biases, and covariances are taken conditional on the relevant true state or motion when the conditioning variable is clear from context.
For an estimator \(\hat{\bm x}\) of \(\bm x_1\), we write
\(\operatorname{Bias}(\hat{\bm x}) := \mathbb E[\hat{\bm x}]-\bm x_1\) and
\(\operatorname{Cov}(\hat{\bm x}) := \mathbb E[(\hat{\bm x}-\mathbb E[\hat{\bm x}])(\hat{\bm x}-\mathbb E[\hat{\bm x}])^\top]\).
For a vector-valued estimator, its scalar variance is given by the trace of its covariance matrix.

\paragraph{ANN estimator.}
We analyze how finite-sample fluctuations in the ANN response map affect the bias and covariance of the ANN estimator in the one-step setting.
Let \(\bm u=\bm x-\bm x_1\) be the local coordinate, where \(\bm x\in\mathbb R^2\) denotes a candidate target state.
The supervised heatmap is defined as
\(g_h(\bm x-\bm x^\ast)=\exp(-\|\bm x-\bm x^\ast\|^2/(2h^2))\), where \(\bm x^\ast\) is the annotated target center and \(h>0\) is the heatmap width.
We assume isotropic Gaussian annotation noise \(\bm x^\ast\mid\bm x_1\sim\mathcal N(\bm x_1,\sigma_{\mathrm{lab}}^2I_2)\), where \(\sigma_{\mathrm{lab}}^2\) is the annotation-noise variance and \(I_2\) is the \(2\times2\) identity matrix.
Under squared-loss regression, the Bayes response of the supervised heatmap regression problem is
\(f^\star(\bm u)=M\exp(-\|\bm u\|^2/(2\rho^2))\), where \(\rho^2=h^2+\sigma_{\mathrm{lab}}^2\) and \(M>0\) is the induced maximum amplitude.
This response represents the ideal local ANN response around the current target state, and the finite-sample analysis below studies how the learned response deviates from it.

For an effective training sample size \(n\), the learned ANN response is written as
\(S_A^{(n)}(\bm u)=f^\star(\bm u)+\delta_n(\bm u)\).
Here, \(S_A^{(n)}\) denotes the response map learned from \(n\) samples, and \(\delta_n\) denotes the finite-sample response fluctuation around the Bayes response.
Under standard regularity assumptions, \(\sqrt n\,\delta_n\Rightarrow G\) in \(C^2(\mathcal B)\), where \(\mathcal B\) is a local neighborhood of \(\bm 0\), \(C^2(\mathcal B)\) is the space of twice continuously differentiable functions on \(\mathcal B\), \(\Rightarrow\) denotes convergence in distribution, and \(G\) is a zero-mean Gaussian random field.
This assumption gives \(\sqrt n\nabla\delta_n(\bm 0)\Rightarrow\nabla G(\bm 0)\) with \(\mathbb E[\nabla G(\bm 0)]=\bm 0\).
For closed-form constants, we use the local covariance model
\(\operatorname{Cov}(G(\bm u),G(\bm v))=\sigma_G^2\exp(-\|\bm u-\bm v\|^2/(2\xi^2))\), where \(\sigma_G^2>0\) is the fluctuation magnitude and \(\xi>0\) is the correlation length.
We assume \(\rho>\xi\), meaning that the ideal response is broader than the local finite-sample fluctuation.

\begin{proposition}[Asymptotic unbiasedness and covariance of the ANN estimator]
\label{prop:ann_unbiased_2d_main}
Under the local ANN response model
\(S_A^{(n)}(\bm u)=f^\star(\bm u)+\delta_n(\bm u)\), with
\(f^\star(\bm u)=M\exp(-\|\bm u\|^2/(2\rho^2))\), let \(\hat{\bm u}_A\) be the local maximizer of \(S_A^{(n)}(\bm u)\) around \(\bm u=\bm 0\), and define the ANN estimator as \(\hat{\bm x}_A=\bm x_1+\hat{\bm u}_A\).
Here, \(M\) is the ideal response amplitude, \(\rho\) is the Bayes-response width, and \(\sigma_G^2\) and \(\xi\) are the fluctuation magnitude and correlation length in the local Gaussian-field covariance model.
Then
\begin{equation}
\hat{\bm x}_A
=
\bm x_1+\frac{\rho^2}{M}\nabla\delta_n(\bm 0)+o_p(n^{-1/2}),
\end{equation}
where \(o_p(n^{-1/2})\) denotes a stochastic remainder smaller than \(n^{-1/2}\) in probability.
Consequently,
\begin{equation}
\operatorname{Bias}_A
:=
\operatorname{Bias}(\hat{\bm x}_A)
=
o(n^{-1/2}),
\end{equation}
and
\begin{equation}
\Sigma_A
:=
\operatorname{Cov}(\hat{\bm x}_A\mid \bm x_1)
=
\frac{\sigma_G^2}{n}
\frac{\rho^4}{M^2\xi^2}I_2
+
o(n^{-1}).
\label{eq:sigma_A_2d_main}
\end{equation}
where \(o(n^{-1})\) denotes a deterministic remainder smaller than \(n^{-1}\).
Thus
\begin{equation}
s_A
:=
\operatorname{tr}(\Sigma_A)
=
\frac{2\sigma_G^2\rho^4}{nM^2\xi^2}
+
o(n^{-1}).
\end{equation}
\end{proposition}

Proposition~\ref{prop:ann_unbiased_2d_main} connects the finite-sample fluctuation of the ANN response map to the error of its argmax-based estimate.
Since the local maximizer is displaced by \(\nabla\delta_n(\bm 0)\), whose limiting distribution has zero mean, the relation \(\operatorname{Bias}_A=o(n^{-1/2})\) shows that the ANN estimator is asymptotically unbiased: as the effective training sample size \(n\) grows, the expected displacement of \(\hat{\bm x}_A\) from the current state \(\bm x_1\) vanishes.
At the same time, \(\Sigma_A\) quantifies how finite-sample response fluctuations perturb the argmax-based estimate.
Thus, in the one-step setting, the ANN branch provides an asymptotically unbiased current-state estimate, while its estimate remains sensitive to local response-map fluctuations.

\paragraph{CANN estimator.}
We next analyze how CANN dynamics transform the same response-map fluctuation into the bias and covariance of the CANN estimator.
Under the aligned formulation, the ANN response map serves as the CANN external drive, \(I_{\mathrm{ext}}(\bm x,t)=S_A^{(n)}(\bm x,t)\).
Since \(S_A^{(n)}=f^\star+\delta_n\), the CANN receives both the ideal target-centered response and the finite-sample fluctuation \(\delta_n\).
In a 2D-CANN, a small perturbation to the activity bump can either shift the bump center or deform its shape.
The center shift is governed by the perturbation components along the two translational modes, whereas non-translational deformation components are suppressed by recurrent dynamics in the local perturbation regime.
Let \(r_1\) and \(r_2\) denote the two translational modes associated with shifts along the two state-space coordinates.
We define
\(a_{1,n}:=\langle\delta_n,r_1\rangle\),
\(a_{2,n}:=\langle\delta_n,r_2\rangle\), and
\(\bm a_n=(a_{1,n},a_{2,n})^\top\), where \(\langle\cdot,\cdot\rangle\) denotes the \(L^2\) inner product over the state space.

Let \(\bm y(t)=(y_1(t),y_2(t))^\top\) denote the displacement of the CANN bump center relative to the current state \(\bm x_1\).
At the beginning of the response window, the bump is centered at the previous state \(\bm x_0\), so \(\bm y(0)=\bm x_0-\bm x_1=-\bm\Delta\).
The local perturbation dynamics of the center displacement can be written as
\(\dot y_i=-\beta y_i+\gamma_{\mathrm{dyn}}a_{i,n}+o_p(n^{-1/2})\) for \(i=1,2\).
Here,
\(\beta=8Ma^2\rho^2/[\tau_cU_0(\rho^2+2a^2)^2]\) and
\(\gamma_{\mathrm{dyn}}=\sqrt{2/\pi}/(\tau_cU_0)\), where \(a\) is the CANN bump width, \(U_0\) is the bump amplitude, and \(\tau_c\) is the CANN time constant defined in Eq.~\eqref{equ:CANN-2d-general}.
For a fixed response time \(T\), define \(\alpha=e^{-\beta T}\).
The CANN estimate after this response window is \(\hat{\bm x}_C=\bm x_1+\bm y(T)\).

\begin{proposition}[Finite-response bias and covariance of the CANN estimator]
\label{prop:cann_bv_2d_main}
Under the aligned formulation, the CANN branch is driven by the ANN response map
\(I_{\mathrm{ext}}=S_A^{(n)}\).
Around the current state \(\bm x_1\), write
\(S_A^{(n)}=f^\star+\delta_n\), where \(f^\star\) is the target-centered Bayes response and \(\delta_n\) is the finite-sample response fluctuation.
Let \(r_1\) and \(r_2\) be the two translational modes of the CANN bump, and set
\(a_{i,n}:=\langle \delta_n,r_i\rangle\) for \(i=1,2\), with
\(\bm a_n=(a_{1,n},a_{2,n})^\top\).
Let \(\hat{\bm x}_C\) be the CANN estimate obtained from the bump center after response time \(T\), with residual response factor \(\alpha=e^{-\beta T}\).
Then
\begin{equation}
\hat{\bm x}_C
=
\bm x_1-\alpha\bm\Delta
+
\frac{1-\alpha}{\beta}\gamma_{\mathrm{dyn}}\bm a_n
+
o_p(n^{-1/2}).
\label{eq:cann_estimator_expansion_main}
\end{equation}
Consequently,
\begin{equation}
\operatorname{Bias}_C
:=
\operatorname{Bias}(\hat{\bm x}_C)
=
-\alpha\bm\Delta
+
o(n^{-1/2}),
\end{equation}
and
\begin{equation}
\Sigma_C
:=
\operatorname{Cov}(\hat{\bm x}_C\mid \bm\Delta)
=
\frac{\sigma_G^2}{n}
(1-\alpha)^2
\frac{\xi^2(\rho^2+2a^2)^4}
{M^2\rho^4(\xi^2+4a^2)^2}
I_2
+
o(n^{-1}).
\label{eq:sigma_C_2d_main}
\end{equation}
Thus
\begin{equation}
s_C
:=
\operatorname{tr}(\Sigma_C)
=
\frac{2\sigma_G^2}{n}
(1-\alpha)^2
\frac{\xi^2(\rho^2+2a^2)^4}
{M^2\rho^4(\xi^2+4a^2)^2}
+
o(n^{-1}).
\end{equation}
\end{proposition}

Proposition~\ref{prop:cann_bv_2d_main} connects finite-time bump evolution and translational-mode projection to the error of the CANN estimate.
The term \(-\alpha\bm\Delta\) shows the finite-response bias: after a finite response time \(T\), the initial displacement from the previous state is attenuated but not fully eliminated.
The stochastic term is determined by \(\bm a_n\), which is the projection of the ANN response fluctuation onto the CANN translational modes.
Under the zero-mean fluctuation model above, this projected fluctuation does not introduce a persistent systematic displacement, but it contributes to the covariance \(\Sigma_C\).
Thus, in the one-step setting, the CANN branch provides a biased estimate with temporal lag, while its covariance is shaped by how attractor dynamics filter response-map fluctuations through the bump-center modes.

\paragraph{Bias--variance complementarity.}
The preceding results show that the ANN and CANN estimators have complementary error structures in the one-step setting.
In terms of bias, the ANN estimator has \(\operatorname{Bias}_A=o(n^{-1/2})\), so its expected estimate approaches the current state \(\bm x_1\) as the effective training sample size increases, while the CANN estimator carries the finite-response bias \(-\alpha\bm\Delta+o(n^{-1/2})\) caused by incomplete bump convergence within the response window.
In terms of variance, comparing Eq.~\eqref{eq:sigma_A_2d_main} and Eq.~\eqref{eq:sigma_C_2d_main} gives
\begin{equation}
\Sigma_C
=
\bar\kappa^2\Sigma_A
+
o(n^{-1}),
\qquad
s_C
=
\bar\kappa^2s_A
+
o(n^{-1}),
\end{equation}
where
\begin{equation}
\bar\kappa
=
(1-\alpha)
\frac{\xi^2(\rho^2+2a^2)^2}
{\rho^4(\xi^2+4a^2)}.
\label{eq:variance_ratio_main}
\end{equation}
Thus, when the CANN parameters satisfy \(\bar\kappa<1\), the CANN estimator has a smaller variance term than the ANN estimator in this local perturbation model.
The parameter \(a\) controls how the response-map fluctuation overlaps with the translational modes of the bump, and the response time \(T\) enters through \(\alpha=e^{-\beta T}\), which also determines the finite-response bias.

A simple sufficient condition for \(\bar\kappa<1\) is
\begin{equation}
a^2
<
\frac{\rho^2(\rho^2-\xi^2)}{\xi^2}.
\label{eq:variance_ratio_bound_main}
\end{equation}
Together with the factor \((1-\alpha)^2\), this condition specifies a regime in which CANN dynamics attenuates the variance induced by local response-map fluctuations.

Together, this idealized one-step analysis establishes the bias--variance complementarity illustrated in Figure~\ref{fig:theory_intuition}(A).
The ANN branch provides an asymptotically unbiased current-state estimate, whereas the CANN branch can provide a lower-variance estimate through dynamical-parameter design, at the cost of finite-response lag.
This complementary error structure identifies the functional role of each branch in a setting matched to frame-based VOT.
Detailed derivations of this one-step bias--variance analysis are provided in Appendix~\ref{app:bias_variance_complementarity}.

\subsection{Functional Synergy through Additive Fusion}
\label{subsec:mse_of_additive_hybridization}

Having identified the bias--variance complementarity between ANN and CANN estimators, we next show how convex additive fusion can utilize their complementary advantages to realize functional synergy in ANN--CANN hybridization.
We use the mean-squared error (MSE) of the current-state estimate as the performance measure, with the bias--variance decomposition \(\operatorname{MSE}(\hat{\bm x})=\|\operatorname{Bias}(\hat{\bm x})\|^2+\operatorname{tr}(\operatorname{Cov}(\hat{\bm x}))\).

In the following, we first derive the MSE of a convex combination of the ANN and CANN estimates.
We then examine how adding a small contribution from the other branch improves each endpoint estimator.
Finally, we show that, under suitable dynamical-parameter choices, a convex fusion weight can achieve a smaller MSE than either individual branch, as illustrated in Figure~\ref{fig:theory_intuition}(B).

\paragraph{MSE of additive fusion.}
We derive the MSE of a convex additive estimator that combines the ANN and CANN estimates in the shared target-state space.
Since \(\hat{\bm x}_A\) and \(\hat{\bm x}_C\) estimate the same current state \(\bm x_1\), we define
\begin{equation}
\hat{\bm x}_H(\lambda)
=
\lambda\hat{\bm x}_A+(1-\lambda)\hat{\bm x}_C,
\qquad
\lambda\in[0,1].
\label{eq:hybrid_estimator_main}
\end{equation}
Here, \(\lambda=1\) recovers the ANN estimator, \(\lambda=0\) recovers the CANN estimator, and \(\lambda\in(0,1)\) gives a convex fusion of the two estimates.

Because the same response-map fluctuation affects both the ANN argmax-based estimate and the CANN bump-center estimate, their stochastic errors are correlated.
Let
\(\tilde{\bm e}_A:=\hat{\bm x}_A-\mathbb E[\hat{\bm x}_A]\) and
\(\tilde{\bm e}_C:=\hat{\bm x}_C-\mathbb E[\hat{\bm x}_C]\)
denote the centered stochastic errors of the two estimators.
Define \(\Sigma_{AC}:=\mathbb E[\tilde{\bm e}_A\tilde{\bm e}_C^\top]\) and \(s_{AC}:=\operatorname{tr}(\Sigma_{AC})\).
Under the local fluctuation model above, the cross term satisfies
\(s_{AC}=\chi_{AC}\sqrt{s_As_C}+o(n^{-1})\), where
\(\chi_{AC}=\xi^2(\xi^2+4a^2)/(\xi^2+2a^2)^2\in(0,1)\).

Using the bias and covariance results derived above, the MSE of the convex additive estimator is
\begin{equation}
\operatorname{MSE}_H(\lambda)
=
(1-\lambda)^2\alpha^2d
+
\lambda^2s_A
+
(1-\lambda)^2s_C
+
2\lambda(1-\lambda)s_{AC}
+
o(n^{-1}).
\label{eq:hybrid_mse_main}
\end{equation}
The first term is the retained CANN finite-response bias, the next two terms are the scalar variances of the ANN and CANN estimators, and the last term accounts for their shared response-map fluctuation.
Thus, \(\lambda\) controls how the fusion balances the ANN estimator's low bias and the CANN estimator's dynamically shaped variance.

\paragraph{Endpoint improvement.}
We next examine the two pure-estimator endpoints of the convex fusion, namely \(\lambda=1\) for the ANN estimator and \(\lambda=0\) for the CANN estimator, and ask whether a small contribution from the other branch can reduce the MSE.
Near the ANN endpoint, let \(\lambda=1-\epsilon\) with \(0<\epsilon\ll1\).
Then
\begin{equation}
\operatorname{MSE}_H(1-\epsilon)-\operatorname{MSE}_H(1)
=
-2\epsilon(s_A-s_{AC})+O(\epsilon^2).
\end{equation}
Thus, a small CANN contribution improves the ANN endpoint when \(s_{AC}<s_A\).
This condition follows from the variance-reduction regime identified above: if \(\bar\kappa<1\), then \(s_C<s_A\), and together with \(s_{AC}=\chi_{AC}\sqrt{s_As_C}+o(n^{-1})\) and \(\chi_{AC}\in(0,1)\), it yields \(s_{AC}<s_A\).

Near the CANN endpoint, let \(\lambda=\epsilon\) with \(0<\epsilon\ll1\).
Then
\begin{equation}
\operatorname{MSE}_H(\epsilon)-\operatorname{MSE}_H(0)
=
-2\epsilon(\alpha^2d+s_C-s_{AC})+O(\epsilon^2).
\end{equation}
The ANN contribution reduces the retained CANN lag term while also changing the stochastic term.
A sufficient condition for local improvement, independent of the motion magnitude \(d\), is \(s_{AC}<s_C\).
This condition can be achieved by choosing the response time \(T\) such that
\begin{equation}
\frac{\xi^2+4a^2}{\xi^2+2a^2}
<
\sqrt{1-\alpha}
\frac{\rho^2+2a^2}{\rho^2}.
\label{eq:sac_less_sc_condition_main}
\end{equation}
When \(\rho^2<\xi^2+2a^2\), a sufficiently long response time makes \(1-\alpha=1-e^{-\beta T}\) close enough to \(1\) to satisfy Eq.~\eqref{eq:sac_less_sc_condition_main}.
These endpoint expansions show that neither pure estimator is locally fixed under the complementary error structure: a small CANN contribution can improve the ANN endpoint, and a small ANN contribution can improve the CANN endpoint under the stated conditions.
This motivates considering convex weights away from the two endpoints, which we formalize next.

\paragraph{Convex fusion yields synergy.}
The endpoint expansions show that, under suitable dynamical-parameter choices, moving away from either pure-estimator endpoint can reduce the MSE.
We now convert this local endpoint observation into a convex-fusion result by minimizing the quadratic part of Eq.~\eqref{eq:hybrid_mse_main}.

\begin{theorem}[Convex additive fusion under dynamical-parameter design]
\label{thm:optimal_lambda_main}
Let the squared one-step target displacement be \(d:=\|\bm\Delta\|^2\), and let the CANN finite-response factor after response time \(T\) be \(\alpha:=e^{-\beta T}\in(0,1)\).
Let the scalar variance and cross-covariance terms be
\(s_A:=\operatorname{tr}(\Sigma_A)\),
\(s_C:=\operatorname{tr}(\Sigma_C)\), and
\(s_{AC}:=\operatorname{tr}(\Sigma_{AC})\).
Define
\begin{equation}
Q_H:=\alpha^2d+s_A+s_C-2s_{AC}.
\end{equation}
Suppose the bump width \(a\) and response time \(T\) are chosen so that \(s_{AC}<s_A\) and \(s_{AC}<s_C\).
Then \(Q_H>0\), and the minimizer of the quadratic part of Eq.~\eqref{eq:hybrid_mse_main} is
\begin{equation}
\lambda^\star
=
\frac{\alpha^2d+s_C-s_{AC}}
{\alpha^2d+s_A+s_C-2s_{AC}}.
\label{eq:lambda_star_main}
\end{equation}
Moreover, \(\lambda^\star\in(0,1)\), and the MSE at \(\lambda^\star\) is smaller than the MSE at both pure-estimator endpoints \(\lambda=1\) and \(\lambda=0\), up to the \(o(n^{-1})\) remainder in Eq.~\eqref{eq:hybrid_mse_main}.
\end{theorem}

The inequalities \(s_{AC}<s_A\) and \(s_{AC}<s_C\) make both sides of the convex optimum positive: the numerator of \(\lambda^\star\) is \(\alpha^2d+s_C-s_{AC}>0\), and \(1-\lambda^\star=(s_A-s_{AC})/Q_H>0\).
They also give \(Q_H=(\alpha^2d+s_C-s_{AC})+(s_A-s_{AC})>0\), so the optimum lies strictly between the two endpoints.
Under the scale assumption \(\rho>\xi\), the feasible design range
\begin{equation}
\frac{\rho^2-\xi^2}{2}
<
a^2
<
\frac{\rho^2(\rho^2-\xi^2)}{\xi^2}
\label{eq:a_design}
\end{equation}
allows the two inequalities to be jointly satisfied when \(T\) is sufficiently large.

Theorem~\ref{thm:optimal_lambda_main} formalizes the intuition in Figure~\ref{fig:theory_intuition}(B).
The convex weight \(\lambda^\star\) balances the asymptotically unbiased ANN estimate with the lower-variance CANN estimate obtained under suitable dynamical-parameter choices, while accounting for finite-response lag and shared response-map fluctuations.
Thus, once the two branches are aligned in the same target-state space, additive fusion provides a simple mechanism for utilizing their complementary error structures.
In the idealized one-step setting, this realizes functional synergy in ANN--CANN hybridization.
The proof is given in Appendix~\ref{app:proof_hybrid_mse}.

\subsection{Continuous-Time Dynamics for Stable Tracking}
\label{subsec:continuous_time_response_map_dynamics}

The one-step analysis above shows how ANN and CANN estimates complement each other in bias and variance, and how convex additive fusion can utilize this complementarity at the estimation level.
Real tracking further requires the estimated state to evolve coherently over time and to remain stable when the ANN response is distorted by distractors.
We therefore extend the analysis to continuous-time dynamics.

In the following, we first show that, in a local constant-velocity regime, CANN center dynamics preserves the same bias--variance complementarity and supports synergistic fusion over time.
We then analyze a spurious remote peak regime, where CANN dynamics suppresses distractor-induced pulls and maintains continuous center evolution, as illustrated in Figure~\ref{fig:theory_intuition}(C) and (D).
Finally, we explain why stable estimation should fuse the ANN response map and the CANN activity bump before final decoding.

\paragraph{Local continuous-time complementarity and synergy.}
We first show that the ANN and CANN branch retains bias--variance complementarity and supports convex fusion in a local continuous-time regime.
Let \(\bm x^*(t)\in\mathbb R^2\) denote the ground-truth target trajectory, and write the ANN estimate as
\(\hat{\bm x}_A(t)=\bm x^*(t)+\bm\varepsilon(t)\), where \(\bm\varepsilon(t)\) is the ANN estimation error.
We assume
\(\mathbb E[\bm\varepsilon(t)]=\bm 0\),
\(\operatorname{Cov}(\bm\varepsilon(t))=\sigma_A^2I_2\), and
\(\operatorname{Cov}(\bm\varepsilon(t),\bm\varepsilon(s))=\sigma_A^2e^{-\eta_{\varepsilon}|t-s|}I_2\), where \(\sigma_A^2>0\) is the instantaneous error variance and \(\eta_{\varepsilon}>0\) is the temporal decorrelation rate.
Under the local constant-velocity approximation
\(\bm x^*(t)=\bm x^*(0)+\bm vt\), where \(\bm v\in\mathbb R^2\) is the local target velocity, the translational-mode reduction of the CANN center dynamics gives
\begin{equation}
\tau_{\mathrm{ct}}\dot{\hat{\bm x}}_C(t)+\hat{\bm x}_C(t)=\hat{\bm x}_A(t),
\end{equation}
where \(\hat{\bm x}_C(t)\) is the CANN center estimate and \(\tau_{\mathrm{ct}}>0\) is the effective center-dynamics time constant.

Taking bias relative to the ground-truth state \(\bm x^*(t)\), this local model gives
\begin{equation}
\begin{aligned}
\operatorname{Bias}_A(t)&=\bm 0,
&
\operatorname{Cov}(\hat{\bm x}_A(t))&=\sigma_A^2I_2,\\
\operatorname{Bias}_C(t)&=-\tau_{\mathrm{ct}}\bm v,
&
\operatorname{Cov}(\hat{\bm x}_C(t))&=
\frac{\sigma_A^2}{1+\eta_{\varepsilon}\tau_{\mathrm{ct}}}I_2 .
\end{aligned}
\label{eq:ct_bias_variance_result}
\end{equation}
Thus, the ANN branch provides an unbiased instantaneous estimate but retains the full temporal fluctuation, whereas the CANN branch reduces this fluctuation by the factor \((1+\eta_{\varepsilon}\tau_{\mathrm{ct}})^{-1}\) while introducing the velocity-dependent lag \(-\tau_{\mathrm{ct}}\bm v\).
This gives the continuous-time counterpart of the one-step bias--variance complementarity.

The same local model also shows how this complementarity can be utilized by convex fusion over time.
For
\(\hat{\bm x}_H(t;\lambda)=\lambda\hat{\bm x}_A(t)+(1-\lambda)\hat{\bm x}_C(t)\), with \(\lambda\in[0,1]\), the MSE in this local regime is
\begin{equation}
\operatorname{MSE}^{\mathrm{ct}}_H(\lambda)
=
(1-\lambda)^2\tau_{\mathrm{ct}}^2\|\bm v\|^2
+
\lambda^2s_A^{\mathrm{ct}}
+
(1-\lambda)^2s_C^{\mathrm{ct}}
+
2\lambda(1-\lambda)s_{AC}^{\mathrm{ct}},
\label{eq:ct_mse}
\end{equation}
where
\(s_A^{\mathrm{ct}}=2\sigma_A^2\),
\(s_C^{\mathrm{ct}}=2\sigma_A^2/(1+\eta_{\varepsilon}\tau_{\mathrm{ct}})\), and
\(s_{AC}^{\mathrm{ct}}=2\sigma_A^2/(1+\eta_{\varepsilon}\tau_{\mathrm{ct}})\).
The first term is the retained CANN lag, while the remaining terms describe how the ANN fluctuation and the CANN-smoothed fluctuation are combined.
Minimizing this quadratic MSE gives the convex weight
\begin{equation}
\lambda^\star_{\mathrm{ct}}
=
\frac{\tau_{\mathrm{ct}}^2\|\bm v\|^2}
{
\tau_{\mathrm{ct}}^2\|\bm v\|^2
+
2\sigma_A^2\eta_{\varepsilon}\tau_{\mathrm{ct}}/(1+\eta_{\varepsilon}\tau_{\mathrm{ct}})
}.
\label{eq:ct_lambda_star}
\end{equation}
When \(\|\bm v\|>0\) and \(\eta_{\varepsilon}\tau_{\mathrm{ct}}>0\), we have \(\lambda^\star_{\mathrm{ct}}\in(0,1)\), so the local continuous-time optimum uses both the ANN and CANN estimates.
Therefore, in the local continuous-time regime, convex fusion can balance the ANN branch's unbiased but fluctuating estimate with the CANN branch's smoother but lagged estimate.

\paragraph{Spurious remote-peak suppression.}
We next show that, beyond the local temporal smoothing captured by a first-order center filter, CANN dynamics supports stable tracking by suppressing remote spurious peaks and maintaining continuous center evolution.
Consider an ANN response map containing one target peak and one remote distractor peak:
\begin{equation}
S_A(\bm x,t)
=
A_\ast e^{-\|\bm x-\bm x^*(t)\|^2/(2\omega_\ast^2)}
+
A_s e^{-\|\bm x-\bm x_s(t)\|^2/(2\omega_s^2)}.
\label{eq:two_peak_response_map}
\end{equation}
Here, the first peak is centered at the target state \(\bm x^*(t)\) with amplitude \(A_\ast>0\) and width \(\omega_\ast>0\), while the second peak is centered at a remote distractor \(\bm x_s(t)\) with amplitude \(A_s>0\) and width \(\omega_s>0\).
The ANN argmax estimate \(\hat{\bm x}_A(t)=\arg\max_{\bm x}S_A(\bm x,t)\) may abruptly jump to the distractor when the remote peak becomes dominant.

The CANN branch receives the full response map rather than only the decoded ANN coordinate.
When the CANN maintains a localized activity bump, we use the single-bump approximation \(U(\bm x,t)\approx \tilde U(\bm x-\hat{\bm x}_C(t))\), where \(\tilde U(\bm u)=U_0e^{-\|\bm u\|^2/(4a^2)}\), \(\hat{\bm x}_C(t)\) is the CANN center estimate, \(U_0\) is the bump amplitude, and \(a\) is the bump width.
Let \(r_s(t)=\|\bm x_s(t)-\hat{\bm x}_C(t)\|\) denote the distance between the remote distractor peak and the current CANN center.
Projecting the spurious Gaussian peak \(A_s e^{-\|\bm x-\bm x_s(t)\|^2/(2\omega_s^2)}\) onto the CANN translational modes gives the following contribution to the CANN center dynamics:
\begin{equation}
\dot{\hat{\bm x}}_{C,s}(t)
=
\beta_s
\exp\!\left(
-\frac{r_s(t)^2}{2(\omega_s^2+2a^2)}
\right)
\bigl(\bm x_s(t)-\hat{\bm x}_C(t)\bigr),
\qquad
\beta_s=
\frac{8A_sa^2\omega_s^2}
{\tau_cU_0(\omega_s^2+2a^2)^2}.
\label{eq:remote_peak_pull_vector}
\end{equation}
Here, \(\dot{\hat{\bm x}}_{C,s}(t)\) denotes the component of the CANN center velocity induced by the spurious peak.
Its magnitude is
\begin{equation}
\|\dot{\hat{\bm x}}_{C,s}(t)\|
=
\beta_s r_s(t)
\exp\!\left(
-\frac{r_s(t)^2}{2(\omega_s^2+2a^2)}
\right).
\label{eq:remote_peak_pull}
\end{equation}
The factor \(r_s(t)\exp[-r_s(t)^2/(2(\omega_s^2+2a^2))]\) increases only up to \(r_s(t)=\sqrt{\omega_s^2+2a^2}\) and then decays as the distractor moves farther away from the current CANN center.
Thus, the remote spurious peak affects the CANN estimate through an overlap-gated velocity contribution.
Its influence is suppressed by the weak spatial overlap between the remote response peak and the current activity bump.

This mechanism differs from directly filtering the decoded ANN coordinate.
For a first-order center filter \(\dot{\hat{\bm x}}_F(t)=\beta_0(\hat{\bm x}_A(t)-\hat{\bm x}_F(t))\) with \(\beta_0>0\), an ANN jump to \(\bm x_s(t)\) produces a pull proportional to \(\|\bm x_s(t)-\hat{\bm x}_F(t)\|\).
In contrast, the CANN pull caused by the same remote peak is proportional to \(r_s(t)\exp[-r_s(t)^2/(2(\omega_s^2+2a^2))]\), and is therefore attenuated at large offsets, as illustrated in Figure~\ref{fig:theory_intuition}(C).

The same continuous-time dynamics also support stable tracking by enforcing continuous center evolution.
Even if the ANN argmax estimate \(\hat{\bm x}_A(t)=\arg\max_{\bm x}S_A(\bm x,t)\) switches abruptly between the target peak and a remote spurious peak, the CANN center is governed by a continuous-time dynamical system.
Under mild regularity conditions on the response amplitudes, peak locations, and mode-reduction remainder, \(\hat{\bm x}_C(t)\) is locally absolutely continuous.
That is, the CANN estimate remains continuous and differentiable almost everywhere, preventing the instantaneous coordinate jumps that may appear in ANN argmax estimation.
This continuous center evolution provides the second stability mechanism illustrated in Figure~\ref{fig:theory_intuition}(D).

\paragraph{Estimation fusion for stable tracking.}
We finally show that stable estimation should combine the ANN response map and the CANN activity bump before final decoding, rather than directly fusing their peak- or center-decoded estimates.
When both branches are locally concentrated around the target state and have comparable local shapes, this operation is consistent with the convex additive estimator analyzed above:
\begin{equation}
\hat{\bm x}_H(t)
=
\arg\max_{\bm x}
\left\{
\lambda S_A(\bm x,t)+(1-\lambda)U(\bm x,t)
\right\}
\approx
\lambda\hat{\bm x}_A(t)+(1-\lambda)\hat{\bm x}_C(t),
\label{eq:estimation_fusion_local_approximation}
\end{equation}
where \(U(\bm x,t)\) is the CANN activity bump and \(\lambda\in[0,1]\) is the fusion weight.
Thus, in the local unimodal regime, combining the two scalar fields before decoding recovers the convex additive estimate used in the MSE analysis.

When a remote spurious peak appears, this local approximation no longer holds.
If the ANN argmax estimate jumps to the distractor while the CANN center remains near the target-side activity bump, directly averaging the decoded estimates can produce an intermediate state that is not supported by either branch.
We therefore form the fused score before final decoding:
\begin{equation}
S_H(\bm x,t)
=
\lambda S_A(\bm x,t)+(1-\lambda)U(\bm x,t),
\qquad
\hat{\bm x}_H(t)
=
\arg\max_{\bm x}S_H(\bm x,t).
\label{eq:estimation_fusion_stable_tracking}
\end{equation}
The remote ANN peak receives weak support from the CANN activity bump because of their small spatial overlap, whereas the target-side response is reinforced near the current CANN center.
Thus, this estimation-fusion rule inherits the remote-peak suppression and continuous center evolution analyzed above, providing a stable final state estimate under distractor-induced response distortions.

Overall, this continuous-time analysis extends the complementarity and synergy in one-step analysis to trajectory-level stable tracking.
In the local regime, the ANN--CANN pair retains complementary error profiles that can be balanced by convex fusion, while in the spurious remote-peak regime, the CANN activity bump suppresses distractor-induced pulls through spatial-overlap gating and maintains continuous center evolution.
These mechanisms support the fusion of the ANN response map and the CANN activity bump for stable state estimation.
The detailed derivations are given in Appendix~\ref{app:continuous_time_response_map_dynamics}.

\subsection{Architecture of HTNN}
\label{sec:htnn_architecture}
Driven by the theoretical insights into bias--variance complementarity, functional synergy, and stable-tracking mechanisms, we propose the HTNN model for visual tracking.
The HTNN comprises a SiamFC-based ANN branch and a training-free CANN branch.
The ANN branch converts visual observations from the current frame into response maps over the target-center state space. 
The CANN branch receives external drives constructed from these response maps across frames and evolves an activity bump over the shared state space.
As shown in Figure~\ref{fig:htnn_architecture}, the two branches are connected through a theory-guided two-stage fusion design. 
Representation-fusion combines the ANN response map with a motion cue to construct a more reliable CANN external drive. 
Estimation-fusion combines the ANN response map and the CANN activity bump before final target-state decoding.
This design follows the state-space alignment in Section~\ref{sec:framework_setting}, the additive-fusion analysis in Section~\ref{subsec:mse_of_additive_hybridization}, and the continuous-time tracking analysis in Section~\ref{subsec:continuous_time_response_map_dynamics}.

\begin{figure}[pos=!htbp]
    \centering
    \includegraphics[width=1.0\linewidth]{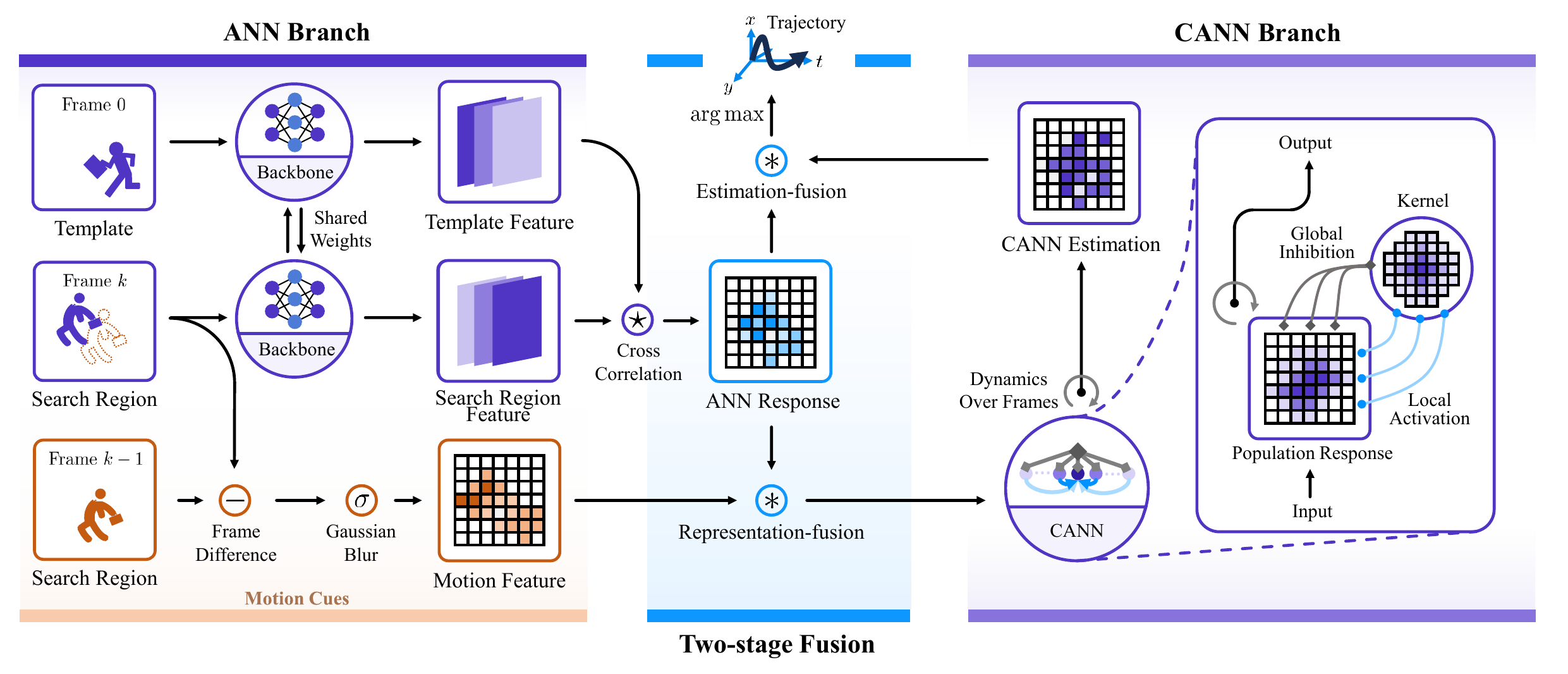}
    \caption{
    Architecture of HTNN. 
    The ANN branch follows a SiamFC-style tracking pipeline, where the template and search region are processed by shared backbones and their cross-correlation produces an ANN response map. 
    Representation-fusion combines the ANN response map with a motion cue extracted from adjacent search regions to construct the CANN external drive. 
    The CANN branch evolves an activity bump through local excitation and global inhibition over the same target-center state space. 
    Estimation-fusion then combines the ANN response map and the evolved CANN activity bump before final target-state decoding.
    The final state estimate is obtained by applying an argmax operation to the fused score.
    }
    \label{fig:htnn_architecture}
\end{figure}

\paragraph{ANN branch and representation-fusion.}
The ANN branch provides current-frame visual evidence over candidate target-center states, and representation-fusion converts this evidence into the external drive for CANN evolution.
Following the SiamFC tracking paradigm in~\ref{sec:prelim_siamfc}, the initial template image \(\bm V_{\mathrm{temp}}^{(0)}\) and the search region \(\bm V_{\mathrm{search}}^{(k)}\) at frame \(k\) are processed by a shared backbone \(\varphi(\cdot)\).
The ANN response map is computed by cross-correlation:
\begin{equation}
S_A^{(k)}
=
\varphi(\bm V_{\mathrm{temp}}^{(0)})
\star
\varphi(\bm V_{\mathrm{search}}^{(k)}).
\end{equation}
Here, \(S_A^{(k)}(\bm x)\) represents the response map that matches evidence whether the target center is located at candidate state \(\bm x\in\mathbb R^2\).
The ANN estimate can be decoded as \(\hat{\bm x}_A^{(k)}=\arg\max_{\bm x} S_A^{(k)}(\bm x)\).
In HTNN, the same response map also provides the visual evidence used to construct the CANN external drive.

Representation-fusion constructs this external drive by combining appearance-based matching evidence with local motion information.
In practical tracking, \(S_A^{(k)}\) may contain scattered activations or spurious peaks caused by background texture, target deformation, or appearance ambiguity.
Although CANN dynamics can suppress remote spurious peaks, an overly scattered drive may still introduce additional perturbations into bump evolution.
A motion cue \(M^{(k)}:\mathbb R^2\rightarrow\mathbb R\) is therefore introduced to modulate the ANN response using local inter-frame changes.
In implementation, \(M^{(k)}\) is computed from the frame difference between \(\bm V_{\mathrm{search}}^{(k)}\) and \(\bm V_{\mathrm{search}}^{(k-1)}\), followed by Gaussian smoothing.
The CANN external drive is then constructed as
\begin{equation}
I_{\mathrm{ext}}^{(k)}
=
\nu_A S_A^{(k)}
+
\nu_M M^{(k)}
+
\nu_{AM}\big(S_A^{(k)}\odot M^{(k)}\big),
\label{eq:representation_fusion}
\end{equation}
where \(\nu_A\), \(\nu_M\), and \(\nu_{AM}\) are scalar representation-fusion coefficients, and \(\odot\) denotes element-wise multiplication.
This construction keeps the drive in the same target-center state space and provides the CANN branch with evidence that is more concentrated around visually and temporally supported target regions.

\paragraph{CANN branch and estimation-fusion.}
The CANN branch evolves an activity bump over the same target-center state space as the ANN response map.
Following the discrete CANN implementation introduced in Section~\ref{sec:2d-cann}, the external drive \(I_{\mathrm{ext}}^{(k)}\) is mapped onto the \(N\times N\) toroidal grid.
The CANN state is then iteratively updated according to Eq.~\eqref{equ:CANN-discrete-update}.
For each video frame, \(\bm I_{\mathrm{ext}}^{(k)}\) is held fixed during a finite response window of \(L\) update steps.
This response window corresponds to response time \(T=L\Delta t\).
After CANN evolution, we denote the evolved activity bump as \(S_C^{(k)}(\bm x):=U^{(k)}(\bm x,T)\), and the CANN estimate can be decoded as \(\hat{\bm x}_C^{(k)}=\arg\max_{\bm x} S_C^{(k)}(\bm x)\).

Estimation-fusion combines the ANN response map and the evolved CANN activity bump before final decoding.
Since both scalar fields are defined over the same target-center state space, HTNN computes
\begin{equation}
S_H^{(k)}
=
\lambda S_A^{(k)}
+
(1-\lambda)S_C^{(k)},
\qquad
\lambda\in[0,1],
\label{eq:estimation_fusion}
\end{equation}
and decodes the final target center as
\begin{equation}
\hat{\bm x}_H^{(k)}
=
\arg\max_{\bm x} S_H^{(k)}(\bm x).
\label{eq:htnn_final_decode}
\end{equation}
This fusion uses current-frame visual evidence from the ANN branch and the temporally evolved activity bump from the CANN branch.
As discussed in Sections~\ref{subsec:mse_of_additive_hybridization} and~\ref{subsec:continuous_time_response_map_dynamics}, estimation-fusion follows the functional-synergy analysis of convex additive fusion and the continuous-time stable-tracking mechanism of CANN dynamics.
By combining the ANN response map with the evolved CANN activity bump before decoding, it can reduce estimation error while preserving response support under remote distractors, yielding a more stable and accurate final tracking estimate.

\paragraph{Overall training procedure and inference acceleration approach.}

\begin{figure}[pos=!htbp]
    \centering
    \includegraphics[width=0.7\linewidth]{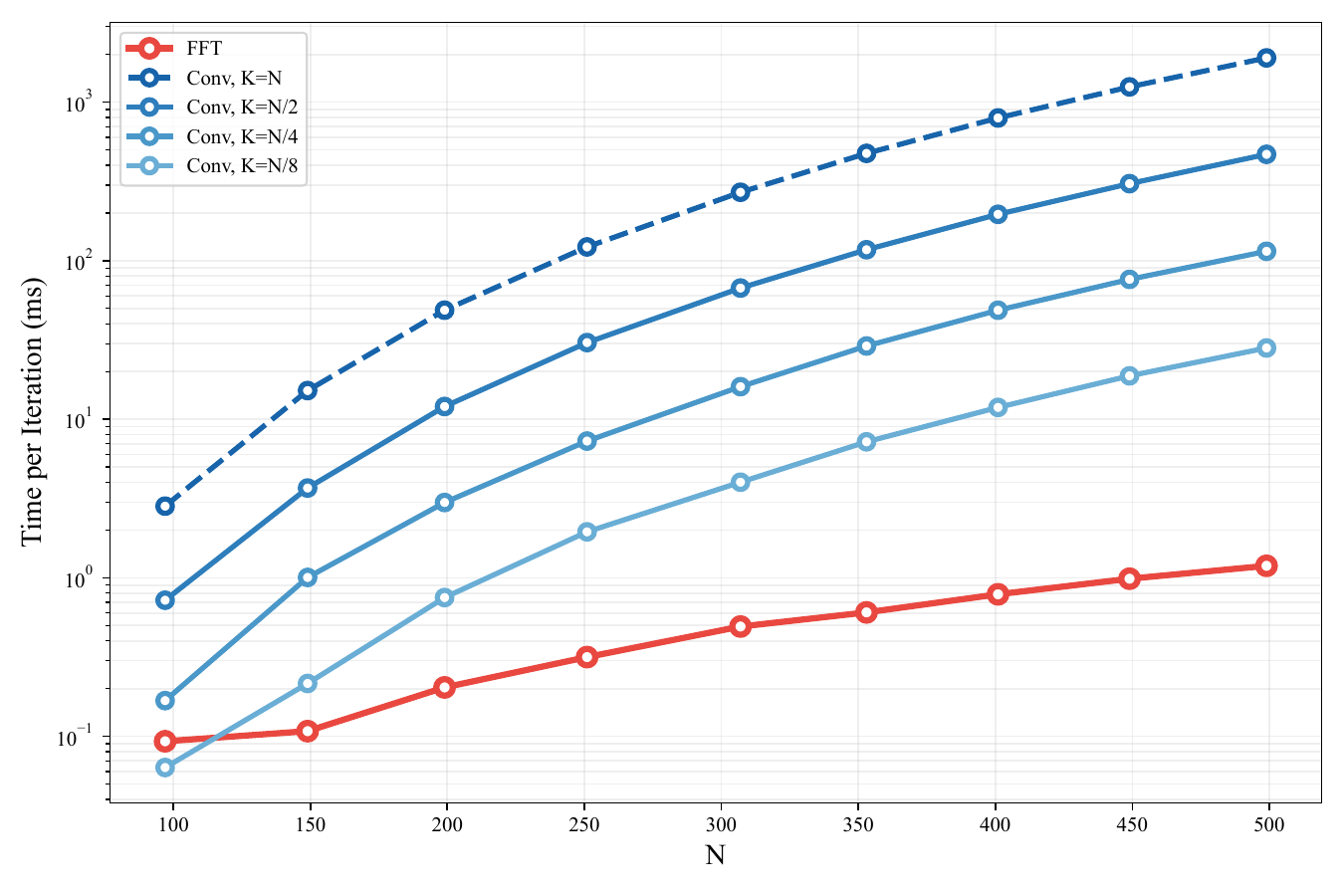}
    \caption{
    Efficiency comparison of recurrent convolution implementations in CANN inference.
    Runtime is measured and averaged for one recurrent convolution with batch size 32.
    The red curve denotes the FFT-based implementation, which computes the exact circular convolution on the toroidal grid.
    The blue curves denote direct or truncated convolution with different kernel sizes \(K\).
    The benchmark deliberately uses prime grid sizes to avoid cuFFT-favorable sizes of the form \(2^a3^b5^c7^d\), giving a conservative setting for FFT acceleration.
    }
    \label{fig:fft_speedup}
\end{figure}

In HTNN, the ANN branch is trained from template--search pairs with dense response-map supervision, whereas the CANN branch is integrated as a training-free dynamical module.
Its recurrent structure is specified by the CANN dynamics and contains a small set of dynamical parameters, including the bump width \(a\), the time constant \(\tau_c\), the inhibition strength $k$, the input gain $A$, and the response time \(T\).
These parameters control the lag term, the variance term, and the response-map dynamics analyzed in Sections~\ref{subsec:bias_variance_complementarity}--\ref{subsec:continuous_time_response_map_dynamics}.
They are selected designed parameters rather than learned by backpropagation.
Therefore, HTNN does not require end-to-end optimization of the full ANN--CANN system.

For efficient inference, we introduce an acceleration approach that uses the Fast Fourier Transform (FFT) to compute recurrent CANN updates exactly on the toroidal grid.
The computational bottleneck is the recurrent convolution term \( \bm J\circledast \bm r \) in Eq.~\eqref{equ:CANN-discrete-update}, which requires \(O(N^4)\) time per update step under direct convolution on an \(N\times N\) response grid.
This cost becomes significant because the CANN state is updated for multiple recurrent steps within each video frame.

A common alternative is truncated convolution, which reduces the runtime to \(O(N^2K^2)\) by limiting the effective kernel size to \(K\).
However, truncation changes the original recurrent dynamics by restricting the spatial range of interaction, and its efficiency depends on the chosen kernel size.
In contrast, the toroidal-grid implementation makes the recurrent term a circular convolution, which can be computed in the frequency domain as
\begin{equation}
\bm J \circledast \bm r
=
\mathcal F^{-1}
\!\left(
\mathcal F(\bm J)\odot \mathcal F(\bm r)
\right),
\label{eq:fft_cann}
\end{equation}
where \(\mathcal F\) and \(\mathcal F^{-1}\) denote the two-dimensional FFT and its inverse.
This reduces each CANN update to \(O(N^2\log N)\) time while preserving the discrete circular-convolution dynamics on the toroidal grid.
Since the recurrent kernel \(\bm J\) is fixed after the CANN parameters are selected, \(\mathcal F(\bm J)\) is precomputed and reused during inference.

Figure~\ref{fig:fft_speedup} compares FFT-based convolution with direct and truncated convolution implementations.
Direct convolution becomes rapidly more expensive as the response-grid size increases.
Truncated convolution reduces the cost by using a smaller kernel, but even the strongly truncated \(K=N/8\) setting only approaches the FFT runtime in the small-\(N\) range and does so by substantially restricting the original recurrent interaction range.
Across the tested response-grid sizes, the FFT-based implementation is approximately \(10^1\)--\(10^3\) times faster than the truncated and direct alternatives while preserving the exact circular convolution.
Thus, HTNN uses FFT-based circular convolution as its exact acceleration route for efficient recurrent CANN inference.

\FloatBarrier

\section{Results}

Here, we present a systematic experimental evaluation of the VOT task.
In Section~\ref{subsec:experimental_setup}, we first introduce the datasets, evaluation metrics, baseline methods, and implementation details used in our evaluation.
In Section~\ref{subsec:main_results}, we compare HTNN with different models on nine benchmark datasets and various challenging conditions.
We then analyze the impact of the temporal continuity of input videos and examine the effectiveness of the bias--variance complementarity.
We further provide qualitative explanations for the tracking stability of HTNN in Section~\ref{sec:visual_analytics} and conduct ablation experiments in Section~\ref{subsec:ablation_study}.
Finally, Section~\ref{subsec:parameter_sensitivity} provides parameter sensitivity analysis to estimate how HTNN responds to changes in fusion parameters and CANN dynamical parameters.

\subsection{Experimental Setup}
\label{subsec:experimental_setup}
\paragraph{Datasets.}
We evaluate HTNN on nine standard single-object tracking benchmarks, including \textbf{OTB100} and \textbf{OTB50}~\citep{Dataset_OTB},
\textbf{GOT-10k}~\citep{Dataset_GOT-10k}, 
\textbf{LaSOT}~\citep{Dataset_LaSOT},
\textbf{TColor128}~\citep{Dataset_TColor128},
\textbf{UAV123}~\citep{Dataset_UAV},
\textbf{NfS}~\citep{Dataset_NfS}, 
\textbf{VOT2019}~\citep{Dataset_VOT2019},
and \textbf{TrackingNet}~\citep{Dataset_TrackingNet}.
These datasets cover diverse tracking scenarios, including high-frame-rate tracking in NfS, regular-frame-rate tracking in OTB and UAV123, long-term tracking in LaSOT, diverse scenes in TrackingNet, and sparse-sampling tracking in GOT-10k.
Detailed dataset statistics are shown in Table~\ref{tab:testing_dataset_stats}.

\begin{table}[htbp]
\centering
\caption{Statistics of the testing datasets used in our evaluation.}
\begin{tabular}{@{}lccc@{}}
\toprule
\textbf{Dataset} & \textbf{Videos} & \textbf{Frames} & \textbf{FPS} \\
\midrule
OTB50        & 50  & 27K  & Not specified\textsuperscript{a} \\
OTB100       & 100 & 59K  & Not specified \\
GOT-10k      & 180 & 21K  & 10fps \\
LaSOT        & 279 & 685K & Not specified \\
TColor128    & 128 & 56K  & Not specified \\
UAV123       & 123 & 113K & 30fps \\
NfS          & 100 & 383K & 240fps \\
VOT2019      & 106 & 20K  & Not specified \\
TrackingNet  & 511 & 226K & Not specified \\
\bottomrule
\end{tabular}
\par\vspace{3pt}\footnotesize\raggedright
\textsuperscript{a} ``Not specified'' indicates that the original dataset paper does not report a fixed frame rate, or that the videos are collected from the Internet or multiple sources and may therefore have different frame rates.
\par
\label{tab:testing_dataset_stats}
\end{table}

\paragraph{Evaluation metrics.}
We primarily adopt two groups of standard evaluation metrics.
The first group evaluates frame-level localization accuracy.
Following common evaluation protocols~\citep{Dataset_GOT-10k}, we use \textbf{Precision} (\(\mathrm{Pr}\)) defined as the proportion of frames in which the Euclidean distance between the predicted target center and the ground-truth target center is smaller than 20 pixels.
We also use \textbf{Success Rate} (\(\mathrm{SR}\)) defined as the proportion of frames in which the intersection-over-union (IoU) between the predicted and ground-truth bounding boxes is greater than 0.5.
Collectively, Pr and SR measure whether a tracker correctly localizes the target.

The second group evaluates trajectory-level stability.
We use \textbf{Velocity Norm Error} (\(\mathrm{VNE}\))~\citep{Metric_Trajectory} to measure the consistency between the predicted motion magnitude and the ground-truth motion magnitude.
Let \(\bm x_t^\ast\) and \(\hat{\bm x}_t\) denote the ground-truth target center and the predicted target center at frame \(t\), respectively.
We compute the ground-truth and predicted velocity magnitudes as \(v_t=\|\bm x_t^\ast-\bm x_{t-1}^\ast\|_2\) and \(\hat v_t=\|\hat{\bm x}_t-\hat{\bm x}_{t-1}\|_2\), respectively.
\(\mathrm{VNE}\) is then defined as the average of \(|v_t - \hat v_t|\) over the evaluated frames.
A lower VNE indicates that the velocity magnitude of the predicted trajectory is closer to that of the ground-truth trajectory.

\paragraph{Baseline methods.}
We compare HTNN with five related methods to separate the roles of different hybridization mechanisms.

\begin{enumerate}
  \item \(\textbf{ANN}^{\mathrm{SiamFC}}\) serves as the data-driven ANN baseline. 
  It directly estimates the current target state through Siamese appearance matching and does not include explicit continuous state evolution.

  \item \(\textbf{CANN}^{\mathrm{M}}\) serves as the motion-driven CANN baseline. 
  It uses motion information as the external drive and obtains the final estimate entirely from CANN evolution, without using the ANN appearance response map.

  \item \(\textbf{CANN}^{\mathrm{DH}}\) serves as the direct ANN--CANN hybridization baseline. 
  It directly feeds the ANN response map into the CANN branch and decodes the final target state from the CANN response, without the representation-fusion and estimation-fusion stages used in HTNN.

  \item \textbf{FlyNet}~\citep{HNN4_FlyNet} serves as a heuristic ANN-CANN hybridization baseline. 
  Although the original model was designed for visual place recognition, we adapt it to VOT by only replacing the neural-network backbone while keeping the tracking pipeline unchanged for fair comparison.
  
  \item \textbf{HSTNN} serves as a recent neuron-scale hybridization baseline.
  In this work, we adapt the Stage-1 hybridization of HSTNN~\citep{HNN12_SNNRNN} to VOT by inserting a lightweight RNN--SNN temporal response head after the SiamFC response map and before final peak decoding.
  We use the Stage-1 variant because it preserves the main response-refinement capacity of HSTNN, while the later selection and restoration stages mainly target neuron selection, compactness, and retraining of the hybrid model.

\end{enumerate}

\paragraph{Implementation details.}
HTNN is implemented in PyTorch and evaluated on a single NVIDIA RTX 3090 GPU.
For the data-driven components, \(\mathrm{ANN}^{\mathrm{SiamFC}}\), FlyNet, and HSTNN are trained on GOT-10k.
For the training-free CANN branch, we selected the dynamical parameters and fusion weights by grid search on the LaSOT training set, using tracking accuracy and stability as the selection criteria.
The selected parameters were then fixed for all benchmark evaluations without dataset-specific tuning.
For CANN dynamics, we use an \(N\times N\) toroidal grid with \(N=85\).
The CANN response window contains \(L=8\) discrete update steps with \(\Delta t=1\), corresponding to response time \(T=L\Delta t\).
The dynamical parameters are set to \(\tau_c=0.71\), \(A=2.01\), \(k=0.08\), and \(a=0.40\).
For the two-stage fusion design, the weights are set to \(\nu_A=0.008\), \(\nu_M=0.47\), \(\nu_{AM}=1.10\), and \(\lambda=0.92\).
More implementation details are given in Appendix~\ref{app:implementation_details}.
\subsection{Main Results}
\label{subsec:main_results}

\paragraph{Overall performance.}

\begin{table}[htbp]
\centering
\caption{Overall comparison across nine benchmarks in terms of localization accuracy and tracking stability.}
\label{tab:overall_pr_sr_vne}
\setlength{\tabcolsep}{2pt}
\begin{tabular}{c ccc ccc ccc ccc ccc ccc}
\toprule
\textbf{Method}\textsuperscript{a}
& \multicolumn{3}{c}{\textbf{HTNN}}
& \multicolumn{3}{c}{\(\textbf{CANN}^{\mathrm{DH}}\)}
& \multicolumn{3}{c}{\textbf{FlyNet}}
& \multicolumn{3}{c}{\textbf{HSTNN}}
& \multicolumn{3}{c}{\(\textbf{ANN}^{\mathrm{SiamFC}}\)}
& \multicolumn{3}{c}{\(\textbf{CANN}^{\mathrm{M}}\)} \\
\cmidrule(lr){2-4}\cmidrule(lr){5-7}\cmidrule(lr){8-10}\cmidrule(lr){11-13}\cmidrule(lr){14-16}\cmidrule(lr){17-19}
\textbf{Metric}\textsuperscript{a}
& \textbf{Pr}\(\uparrow\) & \textbf{SR}\(\uparrow\) & \textbf{VNE}\(\downarrow\)
& \textbf{Pr}\(\uparrow\) & \textbf{SR}\(\uparrow\) & \textbf{VNE}\(\downarrow\)
& \textbf{Pr}\(\uparrow\) & \textbf{SR}\(\uparrow\) & \textbf{VNE}\(\downarrow\)
& \textbf{Pr}\(\uparrow\) & \textbf{SR}\(\uparrow\) & \textbf{VNE}\(\downarrow\)
& \textbf{Pr}\(\uparrow\) & \textbf{SR}\(\uparrow\) & \textbf{VNE}\(\downarrow\)
& \textbf{Pr}\(\uparrow\) & \textbf{SR}\(\uparrow\) & \textbf{VNE}\(\downarrow\) \\
\midrule
OTB50
& \textbf{75.1} & \textbf{68.7} & \textbf{2.58}
& 38.8 & 32.7 & 5.18
& 17.7 & 8.3 & 5.29
& 71.6 & 64.4 & 3.66
& 71.3 & 66.2 & 2.72
& 16.4 & 9.6 & 6.11 \\

OTB100
& \textbf{80.9} & \textbf{75.3} & \textbf{2.23}
& 48.9 & 44.6 & 4.37
& 19.0 & 11.1 & 4.53
& 76.4 & 71.1 & 3.23
& 79.0 & 73.8 & 2.38
& 18.6 & 13.0 & 5.33 \\

GOT-10k
& 35.8 & 59.4 & \textbf{12.54}
& 23.1 & 50.2 & 15.00
& 6.1 & 19.8 & 15.49
& 31.4 & 57.7 & 16.02
& \textbf{35.9} & \textbf{61.6} & 14.07
& 6.7 & 23.8 & 18.00 \\

LaSOT
& \textbf{32.0} & \textbf{35.4} & 6.71
& 20.2 & 26.2 & 6.06
& 5.7 & 6.5 & \textbf{5.98}
& 29.4 & 33.7 & 9.79
& 31.0 & 35.2 & 7.66
& 6.4 & 8.2 & 6.33 \\

TColor128
& \textbf{71.2} & \textbf{64.4} & \textbf{2.84}
& 39.8 & 33.0 & 4.01
& 13.8 & 7.5 & 4.09
& 65.3 & 58.2 & 3.99
& 69.3 & 63.2 & 3.07
& 13.0 & 7.8 & 4.72 \\

UAV123
& \textbf{70.1} & \textbf{61.7} & \textbf{2.27}
& 36.7 & 35.1 & 2.67
& 9.3 & 4.6 & 2.83
& 68.0 & 59.9 & 2.85
& 68.2 & 59.9 & 2.64
& 5.4 & 3.4 & 2.99 \\

NfS
& \textbf{64.8} & \textbf{65.5} & \textbf{2.81}
& 38.7 & 42.7 & 2.84
& 5.5 & 4.0 & 2.98
& 58.9 & 61.9 & 4.20
& 63.7 & 65.1 & 3.25
& 3.9 & 3.4 & 3.17 \\

VOT2019
& \textbf{64.1} & \textbf{43.4} & \textbf{6.54}
& 25.8 & 20.5 & 8.97
& 8.2 & 5.7 & 8.82
& 56.6 & 39.6 & 6.79
& 60.9 & 41.4 & 6.74
& 10.3 & 5.8 & 9.61 \\

TrackingNet\textsuperscript{b}
& \textbf{51.4} & - & -
& 47.7 & - & -
& 28.2 & - & -
& 49.7 & - & -
& 51.2 & - & -
& 28.2 & - & - \\
\bottomrule
\end{tabular}
\par\vspace{3pt}\footnotesize\raggedright
\textsuperscript{a} For Pr and SR, higher is better; for VNE, lower is better. Best results are shown in bold.
\textsuperscript{b} TrackingNet can only be evaluated online, with no SR or VNE returns.
\par
\end{table}

Table~\ref{tab:overall_pr_sr_vne} reports the overall comparison across nine benchmarks and six methods. 
For localization accuracy, HTNN achieves the best Pr and SR on eight out of nine datasets, with GOT-10k being the only exception.  
In contrast, \(\mathrm{CANN}^{\mathrm{DH}}\), FlyNet, and HSTNN perform worse than \(\mathrm{ANN}^{\mathrm{SiamFC}}\) in most cases.
It indicates that direct ANN--CANN coupling, heuristic HNN adaptation, or inherently discrete neuronal-scale hybridization is insufficient for reliable target tracking.
With theoretical guidance from functional synergy, HTNN achieves more accurate tracking.

The stability results show a consistent tendency. 
Among the eight benchmarks with available ground-truth trajectories for offline VNE computation, HTNN obtains the best VNE on most datasets. 
On LaSOT, HTNN does not achieve the lowest VNE, but it still improves over \(\mathrm{ANN}^{\mathrm{SiamFC}}\) while maintaining substantially stronger localization accuracy than \(\mathrm{CANN}^{\mathrm{M}}\), \(\mathrm{CANN}^{\mathrm{DH}}\), and FlyNet. 
Since VNE measures the consistency between predicted and ground-truth velocity magnitudes, these results suggest that HTNN improves tracking stability without sacrificing localization accuracy in most evaluated settings.

The result on GOT-10k should not be viewed as a simple exception. 
Admittedly, on this low-frame-rate benchmark (10 FPS), \(\mathrm{ANN}^{\mathrm{SiamFC}}\) obtains higher Pr and SR. 
This result instead suggests an operating regime of HTNN, where CANN-based continuous dynamics may become less suitable when frame sampling is too sparse. 
We examine this boundary in Section~\ref{failure_case_analysis}, showing that a reduced frame rate and stronger sequence discontinuity weaken the advantage of CANN-based state evolution.

\paragraph{Condition-wise performance.}

\begin{figure}[pos=!htbp]
    \centering
    \includegraphics[width=1.0\linewidth]{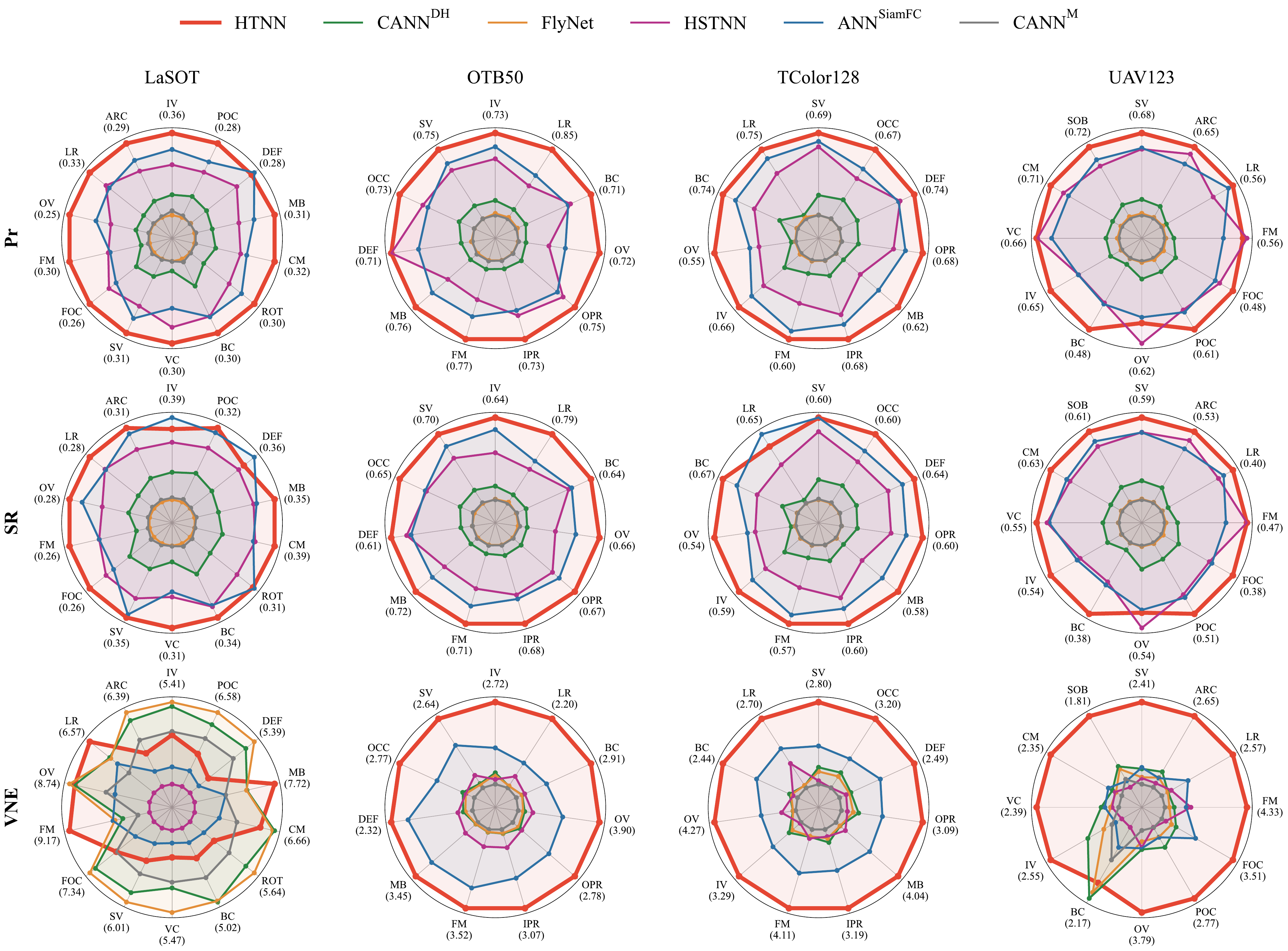}
    \caption{
    Condition-wise comparison on LaSOT, OTB50, TColor128, and UAV123. 
    Each radar chart summarizes Pr, SR, or VNE across the standard challenging condition of the corresponding dataset.
    For visualization, each axis is normalized by the best method on that attribute with a log-scaled gap-to-best radius, so a larger radius always indicates better performance, while the number in parentheses reports the raw best score.
    Condition abbreviations: ARC, aspect ratio change; BC, background clutter; CM, camera motion; DEF, deformation; FM, fast motion; FOC, full occlusion; IPR, in-plane rotation; IV, illumination variation; LR, low resolution; MB, motion blur; OCC, occlusion; OPR, out-of-plane rotation; OV, out-of-view; POC, partial occlusion; ROT, rotation; SOB, similar object; SV, scale variation; VC, viewpoint change. 
    }
    \label{fig:attribute_wise_performance}
\end{figure}

Figure~\ref{fig:attribute_wise_performance} examines HTNN under challenging tracking conditions.
These conditions separate different sources of tracking difficulty, including appearance variation, target motion, occlusion, and background interference, thereby testing whether a method remains reliable across diverse tracking scenarios.
Consistent with Table~\ref{tab:overall_pr_sr_vne}, HTNN maintains favorable Pr, SR, and VNE across most conditions.
This indicates that its gains are broadly distributed across diverse tracking challenges, supporting the role of ANN--CANN fusion in combining accurate localization with stable target-state evolution.

\paragraph{Temporal-continuity boundary analysis.}
\label{failure_case_analysis}

\begin{figure}[pos=!htbp]
    \centering
    \includegraphics[width=1.0\linewidth]{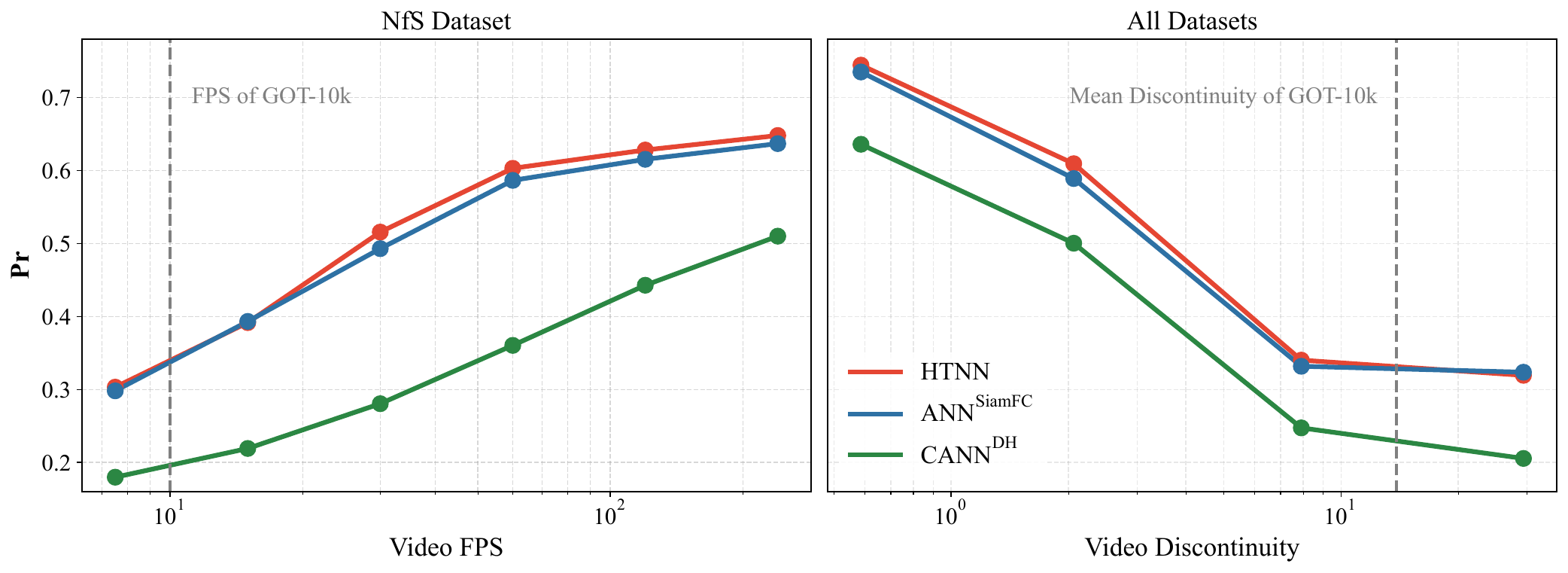}
    \caption{
    Controlled continuity-degradation analyses. 
    Left: average Precision on NfS under temporal subsampling from 240 FPS to lower frame rates. 
    Right: average Precision versus sequence discontinuity aggregated across datasets. 
    The dashed vertical lines indicate the FPS and the mean discontinuity level corresponding to GOT-10k, respectively. Sequence discontinuity is evaluated by Velocity Smoothness Error (VSE)~\citep{Metric_Trajectory}.
    }
    \label{fig:failure_case_analysis}
\end{figure}

The GOT-10k result raises the question of which cases HTNN could benefit from CANN-based state evolution.
Figure~\ref{fig:failure_case_analysis} examines this through two degradation analyses.
Both panels measure how the advantage of HTNN over \(\mathrm{ANN}^{\mathrm{SiamFC}}\) changes as temporal continuity weakens.
The left panel controls temporal continuity by reducing the sampling density of the same video source, while the right panel compares datasets with different motion discontinuity levels.
For the cross-dataset analysis, sequence discontinuity is quantified by Velocity Smoothness Error (VSE)~\citep{Metric_Trajectory}, which measures the average deviation between the ground-truth velocity magnitude sequence and its Savitzky--Golay-smoothed counterpart.
A larger VSE indicates stronger high-frequency variation in target motion and therefore weaker local temporal continuity.

In the NfS subsampling experiment, the performance gap between HTNN and \(\mathrm{ANN}^{\mathrm{SiamFC}}\) narrows as the effective frame rate decreases.
When the subsampled frame rate approaches the GOT-10k regime, the Pr advantage of HTNN disappears.
The cross-dataset analysis shows the same tendency.
As sequence discontinuity increases, the advantage of HTNN over \(\mathrm{ANN}^{\mathrm{SiamFC}}\) becomes smaller, and no clear gain remains beyond the discontinuity level associated with GOT-10k.

These observations characterize the operating regime of the proposed ANN--CANN hybridization.
When adjacent frames preserve sufficient local continuity, CANN dynamics can capture target motion through finite-response state evolution.
When frame sampling becomes sparse or adjacent-frame displacement becomes too large, the target may move beyond the effective range of local bump tracking.
Under this condition, CANN evolution provides less benefit for localization and may introduce a stronger finite-response lag.
The GOT-10k behavior is therefore consistent with the operating regime implied by the continuous-state dynamics of HTNN.
It is also worth noting that GOT-10k represents an extremely low-frame-rate setting (10 FPS), which may be more discontinuous than common real-world applications operating at 30 FPS or higher.

\paragraph{Drift analysis.}

\begin{figure}[pos=!htbp]
    \centering
    \includegraphics[width=1.0\linewidth]{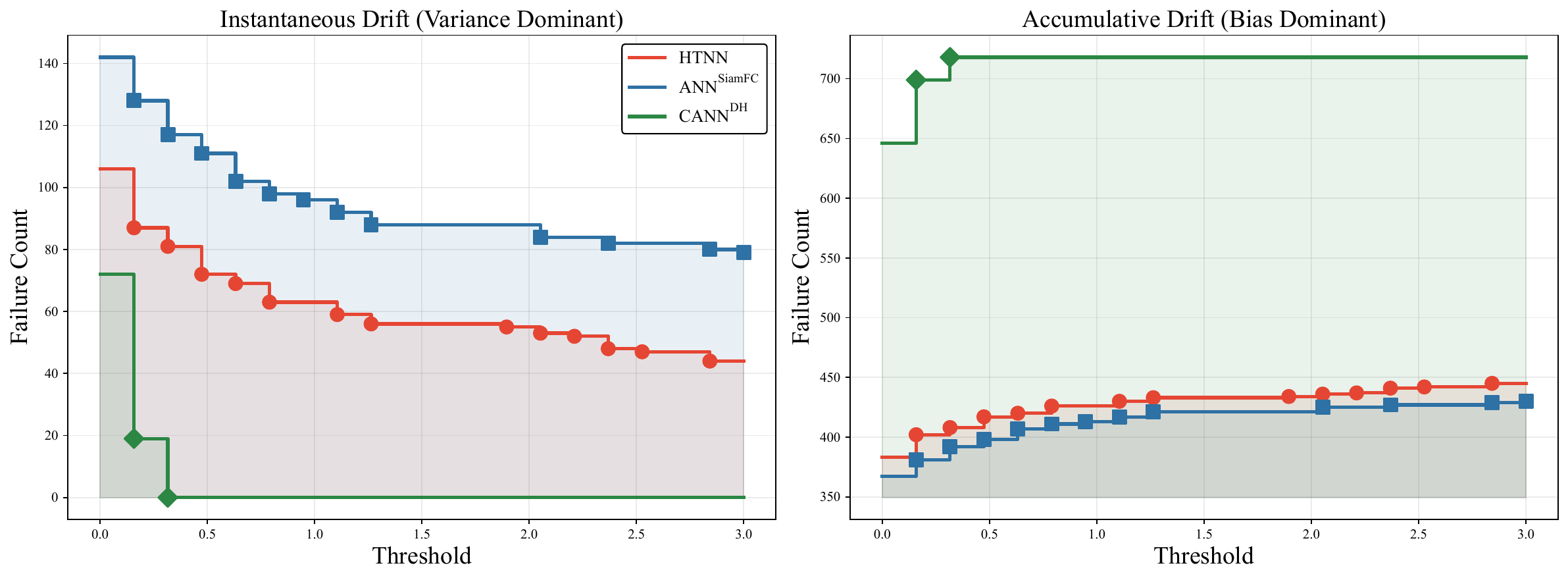}
    \caption{
    Threshold-aggregated drift failure curves over all evaluated sequences. 
    Left: instantaneous drift. 
    Right: accumulative drift. 
    The horizontal axis is the displacement-ratio threshold used for drift classification, and the vertical axis is the aggregated number of failures.
    }
    \label{fig:drift_failure_curves}
\end{figure}

\begin{table}[htbp]
\centering
\caption{Threshold-aggregated drift failure statistics over all evaluated sequences.}
\label{tab:drift_failure_statistics}
\setlength{\tabcolsep}{4pt}
\resizebox{1.0\linewidth}{!} {%
\begin{tabular}{c ccc ccc ccc}
\toprule
\textbf{Failure type}
& \multicolumn{3}{c}{\textbf{Inst. Drift}\textsuperscript{a}\(\downarrow\)} 
& \multicolumn{3}{c}{\textbf{Acc. Drift}\textsuperscript{b}\(\downarrow\)} 
& \multicolumn{3}{c}{\textbf{Total Drift}\textsuperscript{c}\(\downarrow\)} \\
\cmidrule(lr){2-4}\cmidrule(lr){5-7}\cmidrule(lr){8-10}
\textbf{Method}\textsuperscript{d}
& \textbf{ANN}\textsuperscript{\(\mathrm{SiamFC}\)} & \textbf{HTNN} & \textbf{CANN}\textsuperscript{\(\mathrm{DH}\)}
& \textbf{ANN}\textsuperscript{\(\mathrm{SiamFC}\)} & \textbf{HTNN} & \textbf{CANN}\textsuperscript{\(\mathrm{DH}\)}
& \textbf{ANN}\textsuperscript{\(\mathrm{SiamFC}\)} & \textbf{HTNN} & \textbf{CANN}\textsuperscript{\(\mathrm{DH}\)} \\
\midrule
OTB50
& 7.75 & 3.75 & \textbf{0.20}
& \textbf{14.25} & 15.25 & 35.80
& 22.0 & \textbf{19.0} & 36.0 \\

OTB100
& 7.75 & 3.75 & \textbf{0.25}
& \textbf{20.25} & 21.25 & 60.75
& 28.0 & \textbf{25.0} & 61.0 \\

GOT-10k
& 4.30 & 4.35 & \textbf{1.00}
& \textbf{32.70} & 34.65 & 47.00
& \textbf{37.0} & 39.0 & 48.0 \\

LaSOT
& 41.35 & 29.55 & \textbf{1.20}
& \textbf{186.65} & 193.45 & 242.80
& 228.0 & \textbf{223.0} & 244.0 \\

TColor128
& 12.00 & 8.40 & \textbf{0.80}
& \textbf{36.00} & 37.60 & 95.20
& 48.0 & \textbf{46.0} & 96.0 \\

UAV123
& 7.65 & 3.40 & \textbf{0.75}
& \textbf{52.35} & 52.60 & 98.25
& 60.0 & \textbf{56.0} & 99.0 \\

NfS
& 9.00 & 3.00 & \textbf{0.00}
& \textbf{38.00} & 43.00 & 76.00
& 47.0 & \textbf{46.0} & 76.0 \\

VOT2019
& 5.15 & 4.50 & \textbf{0.35}
& 33.85 & \textbf{30.50} & 57.65
& 39.0 & \textbf{35.0} & 58.0 \\

\midrule
All
& 94.95 & 60.70 & \textbf{4.55}
& \textbf{414.05} & 428.30 & 713.45
& 509.0 & \textbf{489.0} & 718.0 \\
\bottomrule
\end{tabular}%
}
\par\vspace{3pt}\footnotesize\raggedright
\textsuperscript{a} Inst. Drift denotes instantaneous-drift failures.
\textsuperscript{b} Acc. Drift denotes accumulative-drift failures.
\textsuperscript{c} Total Drift is the sum of Inst. Drift and Acc. Drift.
Each value is the average failure count over 20 displacement-ratio thresholds uniformly sampled from 0 to 3, which leads to non-integer entries.
\textsuperscript{d} The method columns are ordered as \(\mathrm{ANN}^{\mathrm{SiamFC}}\)--HTNN--\(\mathrm{CANN}^{\mathrm{DH}}\) to make the aggregate failure-count patterns easy to read: Inst. Drift follows \(\mathrm{ANN}^{\mathrm{SiamFC}} > \mathrm{HTNN} > \mathrm{CANN}^{\mathrm{DH}}\), Acc. Drift follows \(\mathrm{ANN}^{\mathrm{SiamFC}} < \mathrm{HTNN} < \mathrm{CANN}^{\mathrm{DH}}\), and Total Drift follows \(\mathrm{ANN}^{\mathrm{SiamFC}} > \mathrm{HTNN} < \mathrm{CANN}^{\mathrm{DH}}\).
\par
\end{table}

We now turn to drift analysis to examine whether the performance gains of HTNN arise from the ANN--CANN complementarity predicted by the one-step analysis in Section~\ref{subsec:bias_variance_complementarity} and the continuous-time analysis in Section~\ref{subsec:continuous_time_response_map_dynamics}.
In full VOT trajectories, however, the corresponding bias and variance terms cannot be directly estimated under the standard tracking protocol.
Only the first frame is initialized by ground truth, and each subsequent search region depends on previous predictions.

Thus, in Figure~\ref{fig:drift_failure_curves} and Table~\ref{tab:drift_failure_statistics}, we analyze two tracking-failure patterns as empirical proxies that connect the theory with tracking behavior. Here, instantaneous drift refers to failures dominated by abrupt erroneous displacement, such as jumps to remote distractors, and reflects sensitivity to instantaneous response-map fluctuations, or variance.
Accumulative drift refers to failures in which the prediction gradually moves away from the target, for example, when the object moves too fast, and reflects sensitivity to finite-response lag or bias.
For each sequence, we first identify the last stage of complete tracking failure, defined as the final frame of a period in which IoU remains below 0.1 for 30 consecutive frames.
For the interval in which tracking degrades from success to failure, a failure is classified as instantaneous drift if, under a given threshold, the ratio between the predicted displacement and the ground-truth displacement exceeds the threshold, and this displacement causes a sharp IoU drop.
Otherwise, the failure is classified as accumulative drift.

Figure~\ref{fig:drift_failure_curves} visualizes how the number of the two drift failures changes with the displacement-ratio threshold.
For instantaneous drift, the curves consistently follow the ordering \(\mathrm{ANN}^{\mathrm{SiamFC}} > \mathrm{HTNN} > \mathrm{CANN}^{\mathrm{DH}}\).
For accumulative drift, the ordering becomes \(\mathrm{ANN}^{\mathrm{SiamFC}} < \mathrm{HTNN} < \mathrm{CANN}^{\mathrm{DH}}\).
This is consistent with the bias--variance complementarity expected from the theoretical analysis: ANN is more sensitive to instantaneous response-map fluctuations, whereas CANN is more sensitive to finite-response lag.
Compared with \(\mathrm{ANN}^{\mathrm{SiamFC}}\), HTNN substantially reduces instantaneous-drift failures at the cost of introducing only a small amount of accumulative drift.
Compared with \(\mathrm{CANN}^{\mathrm{DR}}\), HTNN substantially reduces accumulative-drift failures while introducing only a small amount of instantaneous drift.
Thus, by exploiting the complementary error structures of the two branches, HTNN achieves stronger tracking performance.
Table~\ref{tab:drift_failure_statistics} reports the drift failure statistics across individual datasets and is broadly consistent with the above observations.
This analysis completes the theory--experiment loop at the level of failure patterns.
At the same time, the reduction of instantaneous-drift failures relative to \(\mathrm{ANN}^{\mathrm{SiamFC}}\) further indicates that HTNN achieves more stable tracking.

\subsection{Visual Analytics}
\label{sec:visual_analytics}

The quantitative results above show that HTNN improves tracking stability and localization accuracy. 
We further examine this behavior through visual analytics of response maps and tracking trajectories. 
These visualizations are used to illustrate how HTNN behaves in concrete tracking sequences.

\begin{figure}[pos=!htbp]
    \centering
    \includegraphics[width=1.0\linewidth]{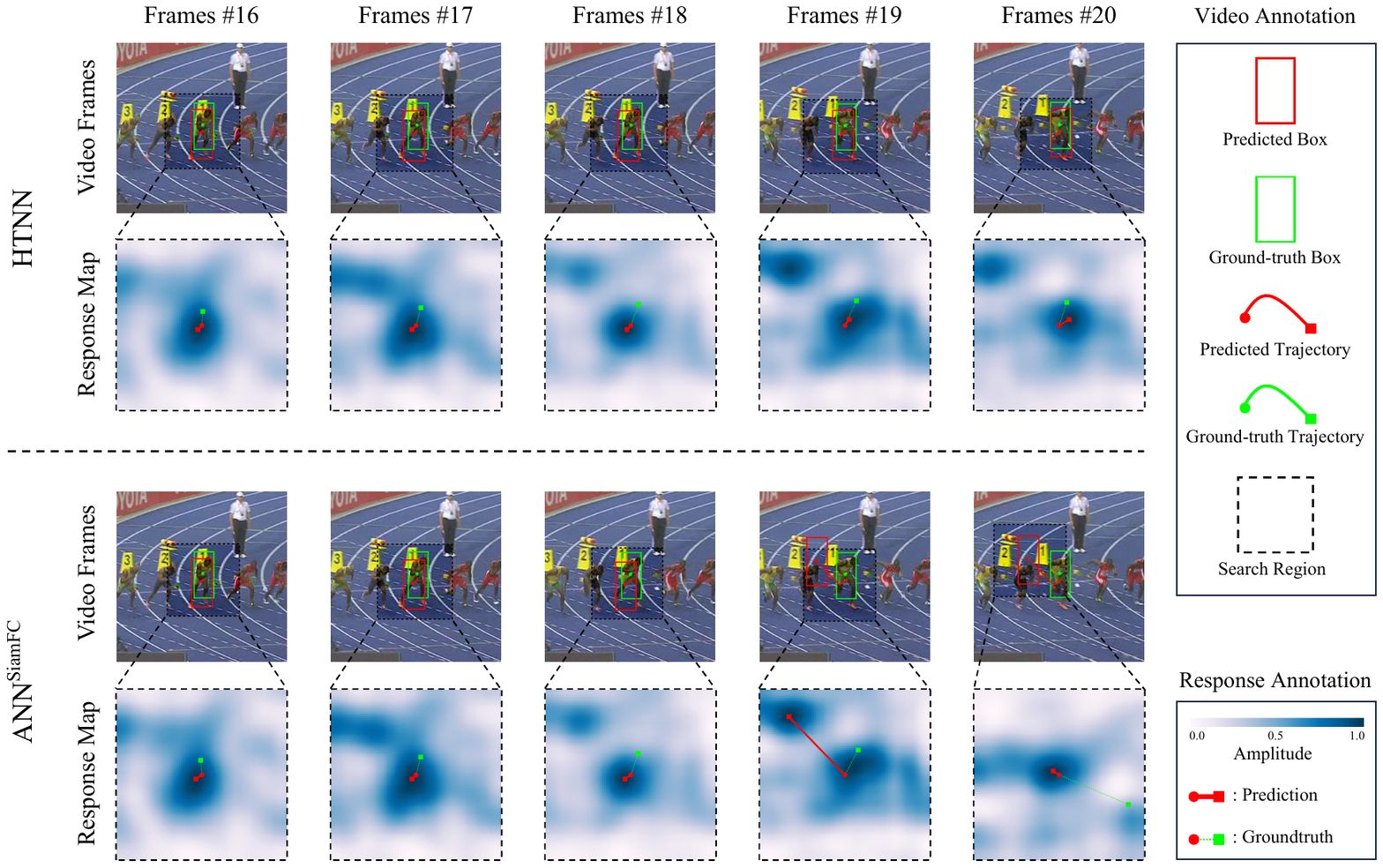}
    \caption{
    Response-map visualization comparing HTNN with \(\mathrm{ANN}^{\mathrm{SiamFC}}\). 
    Each column shows consecutive frames, response maps, predicted boxes, ground-truth boxes, and short predicted / ground-truth trajectories. 
    \(\mathrm{ANN}^{\mathrm{SiamFC}}\) exhibits an abrupt response shift in later frames, leading to instantaneous drift from the target. 
    HTNN keeps the response more concentrated around the target-side region and maintains a more stable prediction across the same interval.
    }
    \label{fig:response_map_1}
\end{figure}

\paragraph{Visual analytics of response maps.}

Figure~\ref{fig:response_map_1} illustrates a case where \(\mathrm{ANN}^{\mathrm{SiamFC}}\) suffers from instantaneous drift. 
In the later frames, its response map shifts toward a distractor-side region, and the predicted box moves away from the ground-truth box. 
HTNN suppresses this remote response peak through CANN response-map dynamics, keeping the dominant response closer to the target-side region. 
The resulting prediction remains close to the ground-truth target, which visualizes the remote-peak suppression mechanism analyzed in Section~\ref{subsec:continuous_time_response_map_dynamics}.

\begin{figure}[pos=!htbp]
    \centering
    \includegraphics[width=1.0\linewidth]{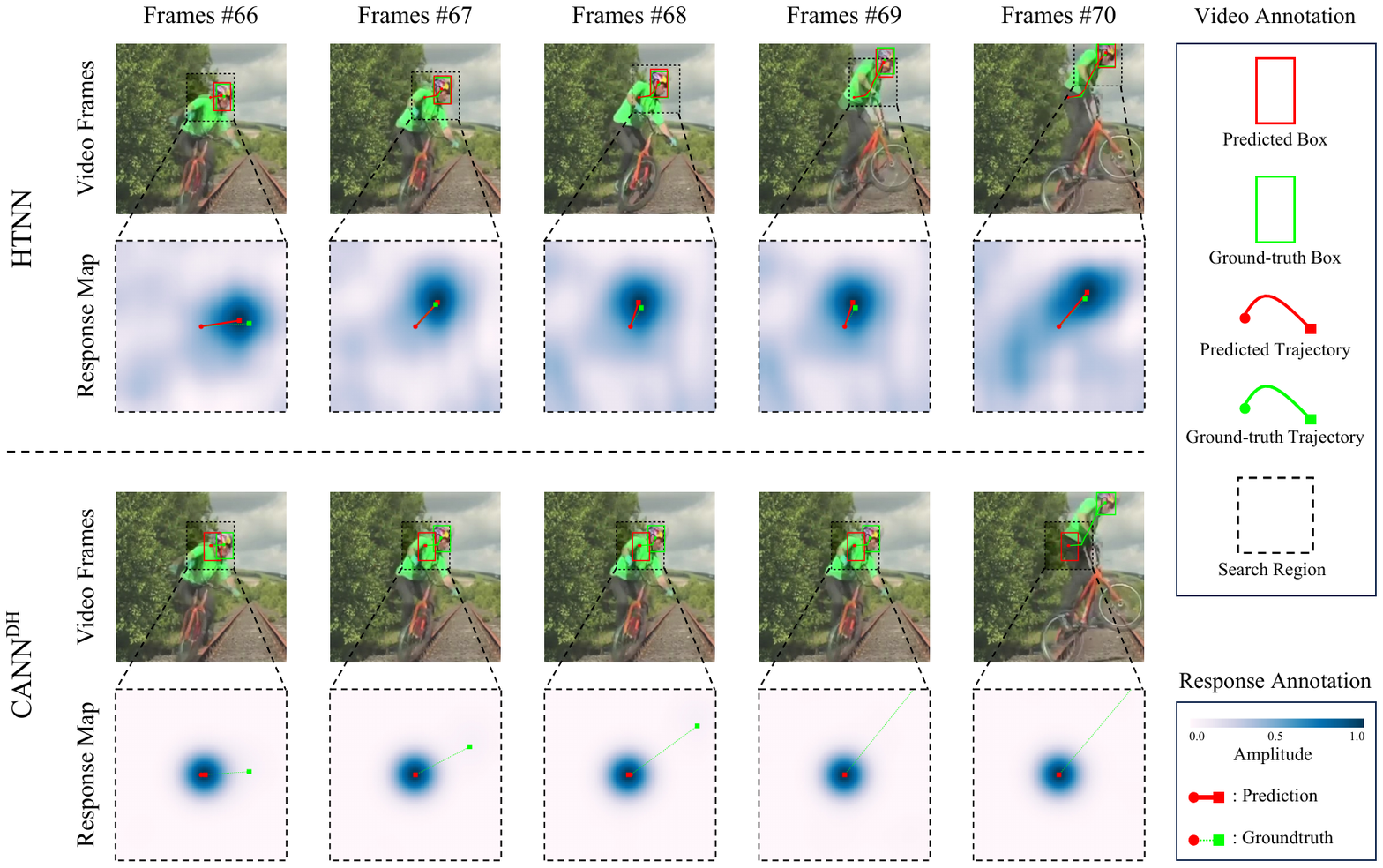}
    \caption{
    Response-map visualization comparing HTNN with \(\mathrm{CANN}^{\mathrm{DH}}\). 
    Each column shows consecutive frames, response maps, predicted boxes, ground-truth boxes, and short predicted / ground-truth trajectories. 
    \(\mathrm{CANN}^{\mathrm{DH}}\) produces a smooth response but lags behind rapid target motion. 
    HTNN follows the changing target state more promptly while preserving a concentrated response around the target.
    }
    \label{fig:response_map_2}
\end{figure}

Figure~\ref{fig:response_map_2} shows the opposite failure tendency of \(\mathrm{CANN}^{\mathrm{DH}}\). 
Its response remains smooth, but the predicted box lags behind the rapidly changing target state. 
HTNN responds more promptly to target motion and keeps the response map concentrated near the ground-truth location. 
This example indicates that the stability of HTNN is not obtained by excessive smoothing alone; it also depends on retaining current-frame ANN evidence through estimation-fusion.

Together, these two examples show the two-sided role of HTNN in response-map evolution. 
Compared with \(\mathrm{ANN}^{\mathrm{SiamFC}}\), HTNN suppresses abrupt response-induced drift. 
Compared with \(\mathrm{CANN}^{\mathrm{DH}}\), HTNN avoids excessive lag when the target state changes rapidly. 
The response-map visualizations, therefore, provide qualitative support for the stable but not over-smoothed tracking behavior observed in the quantitative results.

\begin{figure}[pos=!htbp]
    \centering
    \includegraphics[width=1.0\linewidth]{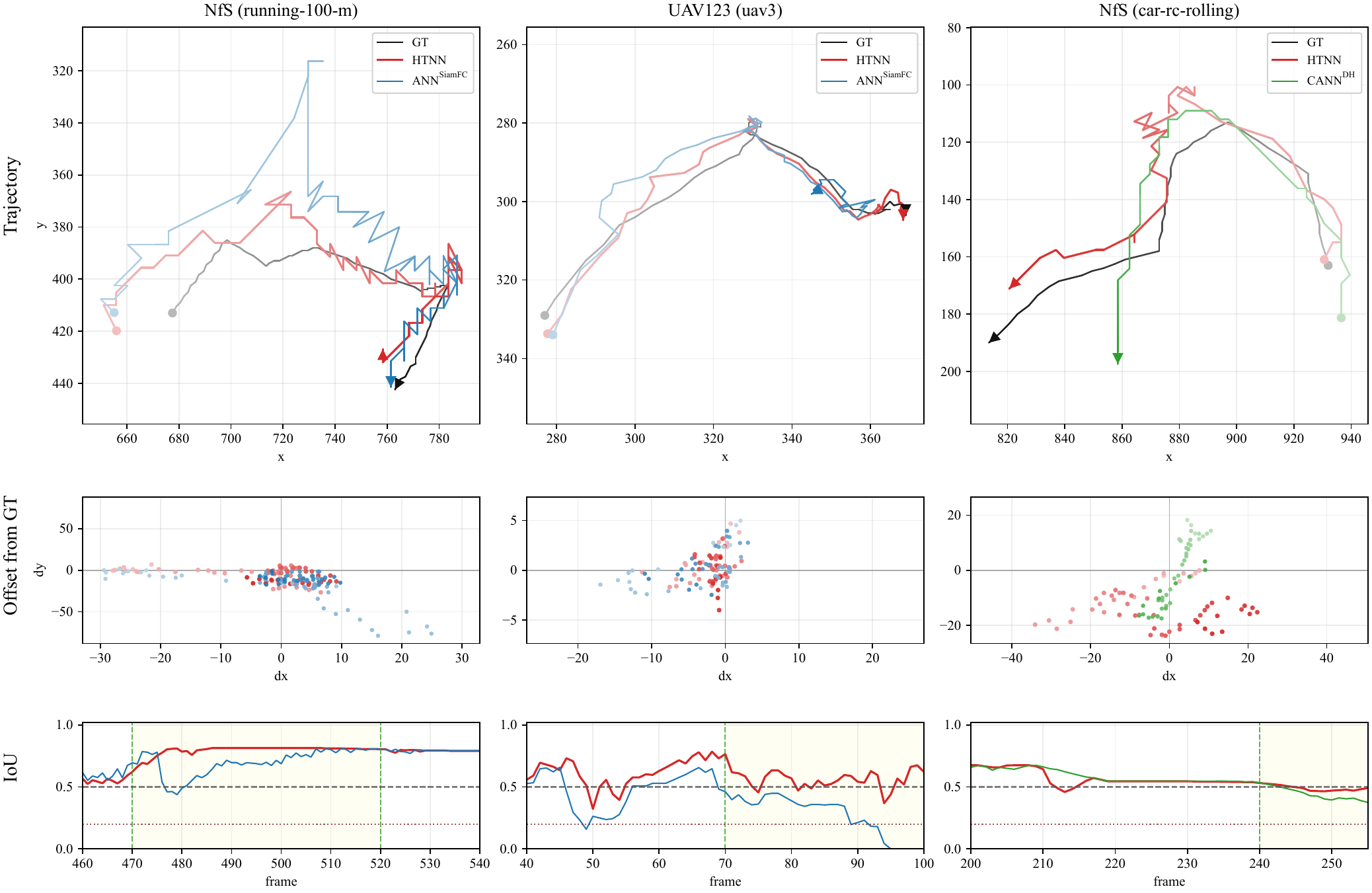}
    \caption{
    Visual analytics of tracking stability. 
    Each column shows a representative sequence. 
    The first two columns compare HTNN with \(\mathrm{ANN}^{\mathrm{SiamFC}}\), and the third column compares HTNN with \(\mathrm{CANN}^{\mathrm{DH}}\). 
    The first row shows predicted center trajectories and the ground-truth trajectory. 
    The second row shows offsets from the ground-truth center. 
    The third row shows IoU curves over the selected frame interval. 
    The yellow shaded region in the third row marks the frame interval where the stability difference between the compared methods becomes evident.
    In the first two sequences, HTNN improves stability by reducing abrupt deviations produced by \(\mathrm{ANN}^{\mathrm{SiamFC}}\). 
    In the third sequence, HTNN maintains stability without the lagged response observed in \(\mathrm{CANN}^{\mathrm{DH}}\) after a rapid target-state change.
    }
    \label{fig:trajectory_stability}
\end{figure}

\paragraph{Visual analytics of tracking trajectories.}

Figure~\ref{fig:trajectory_stability} visualizes tracking stability at the level of predicted trajectories. The first two columns compare HTNN with \(\mathrm{ANN}^{\mathrm{SiamFC}}\). 
In \textit{running-100-m}, \(\mathrm{ANN}^{\mathrm{SiamFC}}\) produces abrupt trajectory deviations, and its offsets become widely dispersed from the ground truth. 
HTNN keeps the predicted centers closer to the ground-truth trajectory and maintains a higher IoU in the highlighted interval. 
In \textit{uav3}, \(\mathrm{ANN}^{\mathrm{SiamFC}}\) becomes unstable in the later frames, which is accompanied by a sharp IoU decrease. 
HTNN keeps the prediction closer to the target and avoids this drift pattern. 
These examples show that HTNN improves tracking stability by suppressing abrupt deviations in the final prediction.

The third column compares HTNN with \(\mathrm{CANN}^{\mathrm{DH}}\). 
In \textit{car-rc-rolling}, \(\mathrm{CANN}^{\mathrm{DH}}\) produces a smooth trajectory, but it lags behind the ground-truth trajectory after the target changes direction. 
HTNN shows local fluctuations, but it follows the target-state change more promptly and maintains better overlap. 
This example shows that stable tracking in HTNN is not equivalent to maximal smoothing. 
It is achieved by combining current-frame visual evidence with CANN-based state evolution, which helps suppress abrupt drift while avoiding the large lag observed in direct CANN-based estimation.

\subsection{Ablation Study}
\label{subsec:ablation_study}

\begin{table}[htbp]
\centering
\caption{Ablation study of representation-fusion and estimation-fusion on eight benchmarks.}
\label{tab:ablation_prefusion_postfusion}
\setlength{\tabcolsep}{3.5pt}
\begin{tabular}{l cc cc cc cc}
\toprule
\textbf{Method}
& \multicolumn{2}{c}{\textbf{Direct Hybrid}\textsuperscript{a}}
& \multicolumn{2}{c}{\textbf{W/ Representation-fusion}}
& \multicolumn{2}{c}{\textbf{W/ Estimation-fusion}}
& \multicolumn{2}{c}{\textbf{Full HTNN}} \\
\cmidrule(lr){2-3}\cmidrule(lr){4-5}\cmidrule(lr){6-7}\cmidrule(lr){8-9}
\textbf{Representation-fusion}
& \multicolumn{2}{c}{\(\times\)}
& \multicolumn{2}{c}{\(\checkmark\)}
& \multicolumn{2}{c}{\(\times\)}
& \multicolumn{2}{c}{\(\checkmark\)} \\
\textbf{Estimation-fusion}
& \multicolumn{2}{c}{\(\times\)}
& \multicolumn{2}{c}{\(\times\)}
& \multicolumn{2}{c}{\(\checkmark\)}
& \multicolumn{2}{c}{\(\checkmark\)} \\
\textbf{Metric}\textsuperscript{b}
& \textbf{Pr}\(\uparrow\) & \textbf{SR}\(\uparrow\)
& \textbf{Pr}\(\uparrow\) & \textbf{SR}\(\uparrow\)
& \textbf{Pr}\(\uparrow\) & \textbf{SR}\(\uparrow\)
& \textbf{Pr}\(\uparrow\) & \textbf{SR}\(\uparrow\) \\
\midrule
OTB50
& 38.8 & 32.7
& +53.4\% & +67.0\%
& +88.9\% & +108.0\%
& \textbf{+93.6\%} & \textbf{+110.1\%} \\

OTB100
& 48.9 & 44.6
& +34.8\% & +39.7\%
& +63.8\% & +67.5\%
& \textbf{+65.4\%} & \textbf{+68.8\%} \\

GOT-10k
& 23.1 & 50.2
& +19.9\% & +11.2\%
& +53.7\% & \textbf{+20.9\%}
& \textbf{+55.0\%} & +18.3\% \\

LaSOT
& 20.2 & 26.2
& +30.7\% & +21.4\%
& +55.4\% & +34.4\%
& \textbf{+58.4\%} & \textbf{+35.1\%} \\

TColor128
& 39.8 & 33.0
& +42.2\% & +50.9\%
& +77.1\% & +92.4\%
& \textbf{+78.9\%} & \textbf{+95.2\%} \\

UAV123
& 36.7 & 35.1
& +64.9\% & +55.3\%
& +83.4\% & +68.1\%
& \textbf{+91.0\%} & \textbf{+75.8\%} \\

NfS
& 38.7 & 42.7
& +31.8\% & +25.5\%
& +66.9\% & +53.2\%
& \textbf{+67.4\%} & \textbf{+53.4\%} \\

VOT2019
& 25.8 & 20.5
& +49.2\% & +45.4\%
& +146.5\% & \textbf{+127.3\%}
& \textbf{+148.4\%} & +111.7\% \\

\midrule
Overall\textsuperscript{c}
& 30.8 & 35.1
& +39.4\% & +32.2\%
& +74.6\% & +56.3\%
& \textbf{+77.3\%} & \textbf{+56.4\%} \\
\bottomrule
\end{tabular}
\par\vspace{3pt}\footnotesize\raggedright
\textsuperscript{a} In this table, Direct Hybrid denotes the \(\mathrm{CANN}^{\mathrm{DH}}\) baseline introduced above.
\textsuperscript{b} The Direct Hybrid columns report Pr and SR, and the remaining columns report relative changes with respect to Direct Hybrid. 
\textsuperscript{c} The Overall row is computed by first taking the video-number-weighted average of each model's absolute performance across the eight datasets and then measuring the relative change against Direct Hybrid.
\par
\end{table}

Table~\ref{tab:ablation_prefusion_postfusion} examines the contribution of the two fusion operations in HTNN. 
The reference setting is Direct Hybrid, where the ANN response map is directly used as the CANN input, and the final target state is decoded from the CANN response. 
This setting implements state-space alignment alone, without the two fusion operations introduced in HTNN. 
The other variants add representation-fusion, estimation-fusion, or both under the same tracking protocol.

Representation-fusion improves the external drive before CANN evolution. 
Combining the ANN response with a motion cue and their interaction makes the CANN input more consistent with local inter-frame changes. 
This design is connected to the theoretical setting where the CANN drive is expected to be dominated by target-centered evidence with bounded response fluctuations. 
Adding representation-fusion alone improves the overall Pr and SR by 39.4\% and 32.2\%, respectively, showing that shaping the CANN drive is already beneficial before changing the final decoding rule.

Estimation-fusion produces a larger improvement. 
Instead of relying only on the CANN response after finite-time evolution, it combines the ANN and CANN response maps before \(\arg\max\). 
This directly follows the additive-fusion analysis: the ANN response supplies current-frame visual evidence, while the CANN response supplies a temporally evolved state-space response. 
Adding estimation-fusion alone improves the overall Pr and SR by 74.6\% and 56.3\%, respectively, indicating that the main gain comes from using both aligned responses in the final estimation.

The full HTNN achieves an overall improvement of 77.3\% in Pr and 56.4\% in SR over Direct Hybrid, exceeding all ablated variants in the overall average.
The result shows that the two fusion operations serve different roles. 
Representation-fusion constructs a more suitable CANN drive, while estimation-fusion uses the complementary ANN and CANN responses for final decoding. 
Their combination supports the framework-level claim that systematic and synergistic ANN--CANN hybridization is necessary for effective tracking.

\subsection{Parameter Sensitivity Analysis}
\label{subsec:parameter_sensitivity}

\begin{figure}[pos=!htbp]
    \centering
    \includegraphics[width=1.0\linewidth]{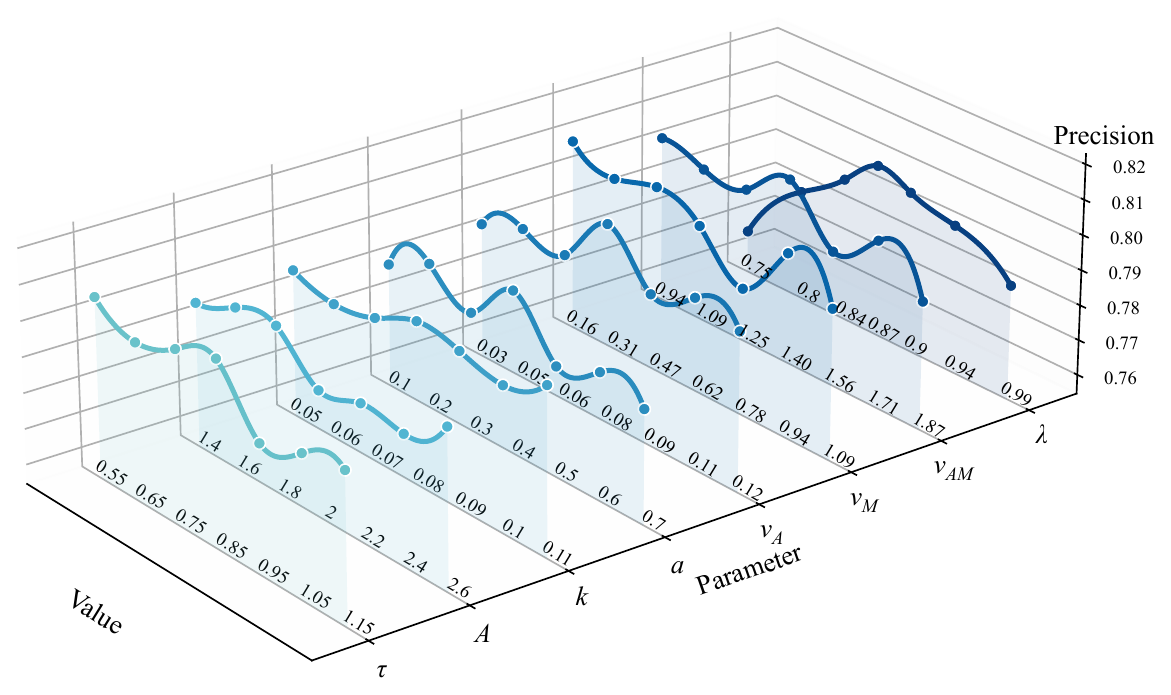}
    \caption{
    Parameter sensitivity analysis of HTNN. 
    Precision is evaluated by sweeping one parameter at a time while keeping the remaining parameters fixed at the default setting. 
    The swept parameters include CANN dynamical parameters \(\tau_c,A,k,a\), representation-fusion weights \(\nu_A,\nu_M,\nu_{AM}\), and the estimation-fusion weight \(\lambda\). 
    The \(\lambda\) curve exhibits a unimodal trend with a peak near the ANN-dominant side, while the other parameters show moderate sensitivity across the tested ranges.
    }
    \label{fig:parameters}
\end{figure}

Figure~\ref{fig:parameters} examines the sensitivity of HTNN to fusion parameters and CANN dynamical parameters. 
A particularly informative trend appears in the estimation-fusion weight \(\lambda\). 
The Precision curve shows a clear unimodal pattern, with its peak located near the ANN-dominant side rather than at either endpoint. 
This indicates that final localization should rely primarily on current-frame ANN evidence, while a nonzero CANN contribution is still needed to improve tracking stability. 
This empirical pattern is consistent with the MSE analysis in Section~\ref{subsec:mse_of_additive_hybridization}, where additive fusion admits an interior optimum under explicit dynamical-parameter conditions. 
In this sense, the \(\lambda\) sweep provides experimental support for the theoretical role of estimation-fusion.

The remaining parameters show moderate variation across the tested ranges. 
The CANN dynamical parameters \(\tau_c\), \(A\), \(k\), and \(a\) do not exhibit a sharp collapse under small perturbations, and the representation-fusion weights \(\nu_A\), \(\nu_M\), and \(\nu_{AM}\) also remain within a practical operating range. 
Although the theory gives a sufficient condition involving the bump width \(a\), this condition is derived under local leading-order assumptions and should be interpreted as a design guide rather than a necessary empirical boundary. 
Overall, the parameter study suggests that HTNN is not dependent on a fragile parameter choice; its performance is mainly shaped by the framework-level design of state-space alignment and two-stage fusion.

\section{Discussion and Conclusions}

In this work, we present a theory-grounded ANN--CANN hybridization framework for continuous-state estimation, extending current HNNs from neuron-scale coupling toward a population-scale hybridization paradigm grounded in neural population dynamics.
Through state-space alignment, the framework unifies heterogeneous ANN and CANN representations over a shared space, characterizes their bias--variance complementarity, and derives a fusion mechanism to realize the functional synergy. 
To instantiate this framework, we propose HTNN for VOT tasks. Systematic experiments on nine VOT benchmarks yield three core conclusions.
First, this hybridization consistently improves localization accuracy and tracking stability compared to baseline methods across diverse, challenging tracking conditions.
Second, drift analyses confirm that HTNN maintains stability by preventing abrupt jumps and controlling the accumulation of lag-induced errors as predicted in our theory.
Third, temporal-continuity analysis characterizes the framework's operating regime, showing that the benefits of CANN-based state evolution weaken when sparse frame sampling or discontinuous motion violates local continuity.

As a proof-of-concept implementation of the population-scale hybridization framework, HTNN also points to several directions for further development.
On the ANN side, we employ SiamFC as a lightweight testbed to validate the feasibility of the framework; adapting this architecture to more advanced AI models (such as SiamCAR~\cite{Future_ANN_1}, TransT~\cite{Future_ANN_2}, and SGLATrack~\cite{Future_ANN_3}) to achieve stronger tracking performance is a promising avenue for future work.
On the CANN side, the present model uses a single two-dimensional attractor dynamics, whereas real tracking may require adaptive dynamical mechanisms across different spatial and temporal scales, and may further benefit from learnable components.
At the framework level, the current analysis explains ANN--CANN hybridization through one-step estimation and local continuous-time dynamics.
A more complete theory should model how response maps, CANN state evolution, and search-region updates interact throughout the full tracking process.
Such a theory would also support adaptive fusion, allowing the fusion weight and CANN response parameters to be adjusted according to confidence and temporal continuity.

The experimental results highlight the potential of ANN--CANN hybridization for open-world tracking applications, where a tracker must accurately localize the target while avoiding unstable jumps under unknown environmental variations.
The integration of CANN dynamics fulfills this requirement by mechanisms of suppressing remote response peaks and maintaining continuous state evolution. 
Overall, this unified framework provides a powerful algorithmic solution for stable target tracking in autonomous and embodied systems, paving the way for low-power, dependable perception when deployed on neuromorphic hardware.

\clearpage

\appendix
\startappendixtoc
\appendixtableofcontents

\setcounter{equation}{0}
\renewcommand{\theequation}{\thesection.\arabic{subsection}.\arabic{equation}}
\makeatletter
\@addtoreset{equation}{subsection}
\makeatother

\setcounter{table}{0}
\renewcommand{\thetable}{\thesection.\arabic{table}}

\section{Proofs and Derivations for the ANN--CANN Hybridization Framework}
\label{app:theoretical_derivations}

\setcounter{assumption}{0}
\renewcommand{\theassumption}{(\alph{assumption})}

This appendix provides the mathematical details for the theoretical analysis in Sections~\ref{subsec:bias_variance_complementarity}--\ref{subsec:continuous_time_response_map_dynamics}.
Its organization follows the three theoretical components of the Method section: one-step bias--variance complementarity , functional synergy through convex additive fusion, and continuous-time stable tracking.
The state-space alignment in Section~\ref{sec:framework_setting} provides the shared target-state representation used throughout the analysis, while the HTNN model in Section~\ref{sec:htnn_architecture} instantiates the resulting framework.

Throughout the appendix, the notation is aligned with the main text.
All asymptotic notation in this appendix is taken with respect to the effective sample size \(n\to\infty\).
For deterministic quantities, \(r_n=o(a_n)\) means \(r_n/a_n\to0\).
For random quantities, \(R_n=o_p(a_n)\) means \(R_n/a_n\) converges to \(0\) in probability.
Accordingly, expansions of random estimators use \(o_p(\cdot)\), whereas deterministic statistical quantities obtained after taking expectations or covariances, such as bias, covariance, and MSE, use \(o(\cdot)\).

\subsection{Derivations for Bias--Variance Complementarity}
\label{app:bias_variance_complementarity}

This section provides the assumptions and derivations for the one-step bias--variance analysis in Section~\ref{subsec:bias_variance_complementarity}.
We first specify the local statistical assumptions for the ANN response map and the local dynamical assumptions for the CANN activity bump.
We then derive the Bayes response induced by Gaussian heatmap supervision, prove the ANN and CANN estimator expansions, and compare their leading-order variance terms.

\subsubsection{Local assumptions for the one-step analysis}
\label{app:technical_assumptions_theory}

The assumptions below specify the local statistical and dynamical regime used in Section~\ref{subsec:bias_variance_complementarity}.
They are not intended as a complete model of full-sequence tracking.
Rather, they isolate the leading-order behavior of ANN--CANN hybridization in a single frame-to-frame state transition.

\begin{assumption}[Supervised heatmap and label noise]
\label{assump:heatmap_label_noise}
Let \(\bm x^\ast\in\mathbb R^2\) denote the annotated target center for the current frame.
Conditioned on the true target center \(\bm x_1\), the annotation noise is isotropic Gaussian:
\begin{equation}
\bm x^\ast\mid \bm x_1\sim\mathcal N(\bm x_1,\sigma_{\mathrm{lab}}^2I_2),
\end{equation}
where \(I_2\) is the \(2\times2\) identity matrix.
The supervised heatmap centered at \(\bm x^\ast\) is
\begin{equation}
g_h(\bm x-\bm x^\ast)=
\exp\!\left(-\frac{\|\bm x-\bm x^\ast\|^2}{2h^2}\right),
\end{equation}
where \(h>0\) is the supervision width.
In local coordinates, we write \(\bm u=\bm x-\bm x_1\) and \(\bm\zeta=\bm x^\ast-\bm x_1\), where \(\bm\zeta\) denotes the annotation noise.
\end{assumption}

\begin{assumption}[Local asymptotic regularity of the ANN response map]
\label{assump:ann_asymptotic}
Let \(S_A^{(n)}(\bm u)\) be the trained ANN response map around the current target state, where \(n\) denotes the effective number of supervised training samples.
There exists a deterministic Bayes response \(f^\star(\bm u)\) such that
\begin{equation}
S_A^{(n)}(\bm u)=f^\star(\bm u)+\delta_n(\bm u).
\end{equation}
Here, \(\delta_n\) represents the finite-sample fluctuation of the learned response map around the Bayes response.
Statistically, this corresponds to the local asymptotic regime in which a regular supervised estimator concentrates around its population target, while its residual fluctuation is of order \(n^{-1/2}\).
Specifically, on a local neighborhood \(\mathcal B\subset\mathbb R^2\) of \(\bm u=\bm 0\), we assume
\begin{equation}
\sqrt n\,\delta_n\Rightarrow G
\quad\text{in } C^2(\mathcal B),
\end{equation}
where \(G\) is a zero-mean Gaussian random field.
Here, \(C^2(\mathcal B)\) denotes the space of twice continuously differentiable functions on \(\mathcal B\), so this convergence controls the local asymptotic behavior of the response map and its first- and second-order derivatives.
This assumption abstracts the standard central-limit-type behavior of smooth response estimators and is used only to characterize local fluctuations around the target-centered response.
\end{assumption}

\begin{assumption}[Local covariance model]
\label{assump:local_covariance}
For closed-form covariance constants, the local covariance kernel of \(G\) is modeled as an isotropic Gaussian kernel:
\begin{equation}
K(\bm u,\bm v)
:=
\operatorname{Cov}(G(\bm u),G(\bm v))
=
\sigma_G^2
\exp\!\left(
-\frac{\|\bm u-\bm v\|^2}{2\xi^2}
\right),
\end{equation}
where \(\sigma_G^2>0\) is the pointwise fluctuation variance and \(\xi>0\) is the local correlation length.
With a general covariance operator, the same arguments yield abstract covariance expressions; this local model gives the closed-form constants used in Section~\ref{subsec:bias_variance_complementarity}.
\end{assumption}

\begin{assumption}[CANN single-bump and translational-mode reduction]
\label{assump:cann_modal_reduction}
The CANN branch operates in a local single-bump regime around \(\bm x_1\).
For the one-step transition \(\bm x_1=\bm x_0+\bm\Delta\), the frame-to-frame displacement is assumed to be small relative to the CANN bump width:
\begin{equation}
\frac{\|\bm\Delta\|}{a}=o(1).
\end{equation}
Equivalently, in a finite local regime, one may read this condition as
\(\|\bm\Delta\|\le \varepsilon_\Delta a\) for a small constant \(0<\varepsilon_\Delta\ll1\).
This condition keeps the previous and current target states within the same local bump neighborhood, so that the bump centered near \(\bm x_0\) can be approximated as translating toward \(\bm x_1\) through the translational modes.
Large frame-to-frame jumps comparable to or larger than the bump width are outside this local one-step analysis.

Its stationary synaptic-input bump prototype is
\begin{equation}
\tilde U(\bm u)
=
U_0
\exp\!\left(-\frac{\|\bm u\|^2}{4a^2}\right),
\end{equation}
where \(U_0>0\) is the bump amplitude and \(a>0\) is the bump width.
Writing \(\bm u=(u_1,u_2)^\top\), the two normalized translational modes are
\begin{equation}
r_i(\bm u)=
\frac{u_i}{\sqrt{2\pi}a^2}
\exp\!\left(-\frac{\|\bm u\|^2}{4a^2}\right),
\qquad i\in\{1,2\},
\end{equation}
with the standard \(L^2(\mathbb R^2)\) inner product
\begin{equation}
\langle \phi,\psi\rangle
=
\int_{\mathbb R^2}\phi(\bm u)\psi(\bm u)\,\mathrm d\bm u.
\end{equation}
Define
\begin{equation}
a_{i,n}:=\langle\delta_n,r_i\rangle,
\qquad
\bm a_n=(a_{1,n},a_{2,n})^\top.
\end{equation}
Within the weak-stimulus, small-distortion, and finite-response regime, the CANN center dynamics is dominated by these two translational modes.
The non-translational deformation modes contribute only higher-order corrections to the center equation over the response window \([0,T]\).
This reduction follows the standard perturbation picture of 2D CANNs, where the translational modes are neutral and the remaining distortion modes decay under recurrent dynamics~\citep{SiWuCANN1,SiWuCANN2}.
\end{assumption}

\subsubsection{Bayes response under Gaussian heatmap supervision}
\label{app:bayes_response}

This section derives the deterministic Bayes response \(f^\star\) used in the ANN estimator analysis.

\begin{proof}[Derivation of the Bayes response]
Under squared-loss heatmap supervision, the Bayes response is the conditional expectation of the supervised heatmap:
\begin{equation}
f^\star(\bm x)
=
\mathbb E\!\left[g_h(\bm x-\bm x^\ast)\mid \bm x_1\right].
\end{equation}
Using the local coordinate \(\bm u=\bm x-\bm x_1\) and the annotation noise \(\bm\zeta=\bm x^\ast-\bm x_1\), Assumption~\ref{assump:heatmap_label_noise} gives \(\bm\zeta\sim\mathcal N(\bm 0,\sigma_{\mathrm{lab}}^2I_2)\).
Hence
\begin{equation}
f^\star(\bm u)
=
\int_{\mathbb R^2}
\exp\!\left(-\frac{\|\bm u-\bm\zeta\|^2}{2h^2}\right)
\frac{1}{2\pi\sigma_{\mathrm{lab}}^2}
\exp\!\left(-\frac{\|\bm\zeta\|^2}{2\sigma_{\mathrm{lab}}^2}\right)
\,\mathrm d\bm\zeta .
\end{equation}
This integral is the convolution of an unnormalized isotropic Gaussian with a normalized isotropic Gaussian density.
Completing the square gives
\begin{equation}
f^\star(\bm u)
=
M
\exp\!\left(-\frac{\|\bm u\|^2}{2\rho^2}\right),
\qquad
\rho^2=h^2+\sigma_{\mathrm{lab}}^2,
\qquad
M=\frac{h^2}{h^2+\sigma_{\mathrm{lab}}^2}.
\end{equation}
The amplitude \(M\) appears because \(g_h\) has unit peak and is not probability-density normalized.

The gradient and Hessian at the peak are
\begin{equation}
\nabla f^\star(\bm 0)=\bm 0,
\qquad
\nabla^2 f^\star(\bm 0)=-\frac{M}{\rho^2}I_2.
\end{equation}
Thus \(\bm u=\bm 0\) is a strict local maximizer of \(f^\star\), and the local curvature of the Bayes response is controlled by \(M/\rho^2\).
\end{proof}

\subsubsection{Proof of ANN asymptotic unbiasedness and covariance}
\label{app:proof_ann_unbiased_2d}

\begin{proof}[Proof of Proposition~\ref{prop:ann_unbiased_2d_main}]
Let \(\hat{\bm u}_A\) be the local maximizer of
\begin{equation}
S_A^{(n)}(\bm u)=f^\star(\bm u)+\delta_n(\bm u)
\end{equation}
near \(\bm u=\bm 0\).
By Appendix~\ref{app:bayes_response},
\begin{equation}
\nabla^2 f^\star(\bm 0)=-\frac{M}{\rho^2}I_2,
\end{equation}
which is strictly negative definite.
Together with the \(C^2(\mathcal B)\) convergence in Assumption~\ref{assump:ann_asymptotic}, this implies that, with probability tending to one, \(S_A^{(n)}\) has a unique local maximizer in a sufficiently small neighborhood of \(\bm u=\bm 0\).

The first-order optimality condition is
\begin{equation}
\nabla S_A^{(n)}(\hat{\bm u}_A)=\bm 0.
\end{equation}
Expanding around \(\bm u=\bm 0\) gives
\begin{equation}
\bm 0
=
\nabla f^\star(\bm 0)
+
\nabla^2 f^\star(\bm 0)\hat{\bm u}_A
+
\nabla\delta_n(\bm 0)
+
o_p(n^{-1/2}).
\end{equation}
Since \(\nabla f^\star(\bm 0)=\bm 0\) and \(\nabla^2 f^\star(\bm 0)=-(M/\rho^2)I_2\), we obtain
\begin{equation}
\hat{\bm u}_A
=
\frac{\rho^2}{M}\nabla\delta_n(\bm 0)
+
o_p(n^{-1/2}).
\end{equation}
Therefore,
\begin{equation}
\hat{\bm x}_A
=
\bm x_1+
\frac{\rho^2}{M}\nabla\delta_n(\bm 0)
+
o_p(n^{-1/2}).
\end{equation}

Because \(\sqrt n\,\delta_n\Rightarrow G\) in \(C^2(\mathcal B)\), we have
\begin{equation}
\sqrt n\,\nabla\delta_n(\bm 0)
\Rightarrow
\nabla G(\bm 0).
\end{equation}
The field \(G\) is zero mean, so the leading term of \(\hat{\bm x}_A-\bm x_1\) has zero mean.
Hence
\begin{equation}
\operatorname{Bias}_A
=
\mathbb E[\hat{\bm x}_A-\bm x_1]
=
o(n^{-1/2}).
\end{equation}

It remains to compute the covariance.
By Assumption~\ref{assump:local_covariance},
\begin{equation}
K(\bm u,\bm v)
=
\operatorname{Cov}(G(\bm u),G(\bm v))
=
\sigma_G^2\exp\!\left(-\frac{\|\bm u-\bm v\|^2}{2\xi^2}\right).
\end{equation}
For \(i,j\in\{1,2\}\), the \((i,j)\)-th entry of the gradient covariance is
\begin{equation}
\big[\operatorname{Cov}(\nabla G(\bm 0))\big]_{ij}
=
\operatorname{Cov}(\partial_iG(\bm 0),\partial_jG(\bm 0))
=
\left.
\partial_{u_i}\partial_{v_j}K(\bm u,\bm v)
\right|_{\bm u=\bm v=\bm 0}.
\end{equation}
Writing \(\bm w=\bm u-\bm v\), we have
\begin{equation}
\partial_{u_i}\partial_{v_j}K(\bm u,\bm v)
=
-\partial_{w_i}\partial_{w_j}
\left[
\sigma_G^2\exp\!\left(-\frac{\|\bm w\|^2}{2\xi^2}\right)
\right].
\end{equation}
Evaluating at \(\bm w=\bm 0\) gives
\begin{equation}
\operatorname{Cov}(\partial_iG(\bm 0),\partial_jG(\bm 0))
=
\frac{\sigma_G^2}{\xi^2}\delta_{ij}.
\end{equation}
Therefore,
\begin{equation}
\operatorname{Cov}(\nabla\delta_n(\bm 0))
=
\frac{\sigma_G^2}{n\xi^2}I_2
+
o(n^{-1}).
\end{equation}
Substituting the local maximizer expansion gives
\begin{equation}
\Sigma_A
:=
\operatorname{Cov}(\hat{\bm x}_A\mid \bm x_1)
=
\frac{\sigma_G^2}{n}
\frac{\rho^4}{M^2\xi^2}I_2
+
o(n^{-1}).
\end{equation}
Taking the trace gives
\begin{equation}
s_A
:=
\operatorname{tr}(\Sigma_A)
=
\frac{2\sigma_G^2\rho^4}{nM^2\xi^2}
+
o(n^{-1}).
\end{equation}
This proves Proposition~\ref{prop:ann_unbiased_2d_main}.
\end{proof}

\subsubsection{Proof of CANN finite-response bias and covariance}
\label{app:proof_cann_bv_2d}

\begin{proof}[Proof of Proposition~\ref{prop:cann_bv_2d_main}]
Let \(\bm y(t)\in\mathbb R^2\) denote the displacement of the CANN bump center relative to the current target state \(\bm x_1\).
At the beginning of the response window, the bump is centered at the previous state, so
\begin{equation}
\bm y(0)=\bm x_0-\bm x_1=-\bm\Delta.
\end{equation}
The CANN external drive is the ANN response map in the aligned formulation:
\begin{equation}
I_{\mathrm{ext}}(\bm u,t)=f^\star(\bm u)+\delta_n(\bm u).
\end{equation}
Under the single-bump approximation, the CANN synaptic input is written as
\begin{equation}
U(\bm u,t)
=
\tilde U(\bm u-\bm y(t))
+
R_\perp(\bm u,t),
\end{equation}
where \(R_\perp\) denotes the stable non-translational deformation component.

Projecting the CANN dynamics onto the shifted translational modes gives
\begin{equation}
\dot y_i(t)
=
\gamma_{\mathrm{dyn}}
\left\langle I_{\mathrm{ext}}(\cdot,t),r_i(\cdot-\bm y(t))\right\rangle
+
o(\|\bm y(t)\|)
+
o_p(n^{-1/2}),
\qquad
i\in\{1,2\},
\end{equation}
where
\begin{equation}
\gamma_{\mathrm{dyn}}
=
\frac{\sqrt{2/\pi}}{\tau_cU_0}.
\end{equation}
This coefficient follows from
\begin{equation}
\partial_{u_i}\tilde U(\bm u)
=
-\frac{u_i}{2a^2}U_0
\exp\!\left(-\frac{\|\bm u\|^2}{4a^2}\right)
=
-U_0\sqrt{\frac{\pi}{2}}\,r_i(\bm u),
\end{equation}
which converts the external-drive projection onto \(r_i\) into a center-velocity term.

We first compute the deterministic contribution of \(f^\star\).
For \(i=1\), using
\begin{equation}
r_1(\bm u)=
\frac{u_1}{\sqrt{2\pi}a^2}
\exp\!\left(-\frac{\|\bm u\|^2}{4a^2}\right),
\end{equation}
we have
\begin{equation}
\left\langle f^\star,r_1(\cdot-\bm y)\right\rangle
=
\frac{M}{\sqrt{2\pi}a^2}
\int_{\mathbb R^2}
(u_1-y_1)
\exp\!\left(
-\frac{\|\bm u\|^2}{2\rho^2}
-\frac{\|\bm u-\bm y\|^2}{4a^2}
\right)
\,\mathrm d\bm u .
\end{equation}
Expanding this Gaussian integral to first order in \(\bm y\) gives
\begin{equation}
\left\langle f^\star,r_1(\cdot-\bm y)\right\rangle
=
-
\frac{8Ma^2\rho^2}{\sqrt{2/\pi}(\rho^2+2a^2)^2}
y_1
+
O(\|\bm y\|^2).
\end{equation}
The same calculation holds for \(i=2\).
Multiplying by \(\gamma_{\mathrm{dyn}}\), the deterministic restoring coefficient is
\begin{equation}
\beta
=
\gamma_{\mathrm{dyn}}
\frac{8Ma^2\rho^2}{\sqrt{2/\pi}(\rho^2+2a^2)^2}
=
\frac{8Ma^2\rho^2}{\tau_cU_0(\rho^2+2a^2)^2}.
\end{equation}

For the fluctuation term, Assumption~\ref{assump:cann_modal_reduction} gives
\begin{equation}
\left\langle \delta_n,r_i(\cdot-\bm y(t))\right\rangle
=
a_{i,n}
+
o_p(n^{-1/2})
\end{equation}
in the local tracking regime.
Thus, after first-order local linearization, the CANN center dynamics is
\begin{equation}
\dot y_i(t)
=
-\beta y_i(t)
+
\gamma_{\mathrm{dyn}}a_{i,n}
+
o_p(n^{-1/2}),
\qquad
i\in\{1,2\}.
\end{equation}

Solving this linear equation on \([0,T]\), and writing
\begin{equation}
\alpha=e^{-\beta T},
\end{equation}
yields
\begin{equation}
\bm y(T)
=
-\alpha\bm\Delta
+
\frac{1-\alpha}{\beta}\gamma_{\mathrm{dyn}}\bm a_n
+
o_p(n^{-1/2}).
\end{equation}
Since \(\hat{\bm x}_C=\bm x_1+\bm y(T)\), we obtain
\begin{equation}
\hat{\bm x}_C
=
\bm x_1-\alpha\bm\Delta
+
\frac{1-\alpha}{\beta}\gamma_{\mathrm{dyn}}\bm a_n
+
o_p(n^{-1/2}).
\label{eq:cann_estimator_expansion_appendix}
\end{equation}
Because the leading finite-sample fluctuation is zero mean, \(\mathbb E[\bm a_n]=\bm 0+o(n^{-1/2})\).
Taking expectations in Eq.~\eqref{eq:cann_estimator_expansion_appendix} gives
\begin{equation}
\operatorname{Bias}_C
=
\mathbb E[\hat{\bm x}_C-\bm x_1]
=
-\alpha\bm\Delta
+
o(n^{-1/2}).
\end{equation}

We now compute the covariance of \(\bm a_n\).
From Assumption~\ref{assump:local_covariance}, for \(a_{1,n}=\langle\delta_n,r_1\rangle\),
\begin{equation}
\operatorname{Var}(a_{1,n})
=
\frac{1}{n}
\int_{\mathbb R^2}
\int_{\mathbb R^2}
r_1(\bm u)K(\bm u,\bm v)r_1(\bm v)
\,\mathrm d\bm u\,\mathrm d\bm v
+
o(n^{-1}).
\end{equation}
The same expression holds for \(a_{2,n}\), and the cross term vanishes by symmetry.
Using the Fourier convention
\begin{equation}
\widehat\phi(\bm k)
=
\int_{\mathbb R^2}e^{-i\bm k^\top\bm u}\phi(\bm u)\,\mathrm d\bm u,
\end{equation}
whose inverse transform contains the factor \((2\pi)^{-2}\), we have
\begin{equation}
\widehat K(\bm k)
=
2\pi\sigma_G^2\xi^2
\exp\!\left(-\frac{\xi^2\|\bm k\|^2}{2}\right),
\end{equation}
and
\begin{equation}
\widehat r_1(\bm k)
=
-i\,4\sqrt{2\pi}a^2 k_1
\exp(-a^2\|\bm k\|^2).
\end{equation}
Therefore,
\begin{equation}
\int\!\!\int r_1(\bm u)K(\bm u,\bm v)r_1(\bm v)
\,\mathrm d\bm u\,\mathrm d\bm v
=
\frac{1}{(2\pi)^2}
\int_{\mathbb R^2}
|\widehat r_1(\bm k)|^2\widehat K(\bm k)
\,\mathrm d\bm k .
\end{equation}
Substituting the two Fourier transforms and using
\begin{equation}
\int_{\mathbb R^2} k_1^2 e^{-B\|\bm k\|^2}\,\mathrm d\bm k
=
\frac{\pi}{2B^2},
\qquad
B=2a^2+\frac{\xi^2}{2},
\end{equation}
gives
\begin{equation}
\operatorname{Var}(a_{1,n})
=
\frac{\sigma_G^2}{n}
\frac{32\pi a^4\xi^2}{(\xi^2+4a^2)^2}
+
o(n^{-1}).
\end{equation}
Thus,
\begin{equation}
\operatorname{Cov}(\bm a_n)
=
\frac{\sigma_G^2}{n}J(a,\xi)I_2
+
o(n^{-1}),
\qquad
J(a,\xi)
=
\frac{32\pi a^4\xi^2}{(\xi^2+4a^2)^2}.
\label{eq:cov_an_appendix}
\end{equation}

From Eq.~\eqref{eq:cann_estimator_expansion_appendix},
\begin{equation}
\hat{\bm x}_C-\mathbb E\hat{\bm x}_C
=
\frac{1-\alpha}{\beta}\gamma_{\mathrm{dyn}}\bm a_n
+
o_p(n^{-1/2}).
\end{equation}
Using Eq.~\eqref{eq:cov_an_appendix}, we obtain
\begin{equation}
\Sigma_C
:=
\operatorname{Cov}(\hat{\bm x}_C\mid\bm\Delta)
=
\left(\frac{1-\alpha}{\beta}\gamma_{\mathrm{dyn}}\right)^2
\frac{\sigma_G^2}{n}J(a,\xi)I_2
+
o(n^{-1}).
\end{equation}
Since
\begin{equation}
\frac{\gamma_{\mathrm{dyn}}}{\beta}
=
\frac{\sqrt{2/\pi}(\rho^2+2a^2)^2}{8Ma^2\rho^2},
\end{equation}
we get
\begin{equation}
\Sigma_C
=
\frac{\sigma_G^2}{n}
(1-\alpha)^2
\frac{\xi^2(\rho^2+2a^2)^4}
{M^2\rho^4(\xi^2+4a^2)^2}
I_2
+
o(n^{-1}).
\end{equation}
Taking the trace gives
\begin{equation}
s_C
:=
\operatorname{tr}(\Sigma_C)
=
\frac{2\sigma_G^2}{n}
(1-\alpha)^2
\frac{\xi^2(\rho^2+2a^2)^4}
{M^2\rho^4(\xi^2+4a^2)^2}
+
o(n^{-1}).
\end{equation}
This proves Proposition~\ref{prop:cann_bv_2d_main}.
\end{proof}

\subsubsection{Variance comparison and lower-variance condition}
\label{app:proof_variance_comparison_2d}

\begin{proof}[Derivation of the variance comparison]
From Eq.~\eqref{eq:sigma_A_2d_main} and Eq.~\eqref{eq:sigma_C_2d_main}, the leading covariance ratio is
\begin{equation}
\frac{\Sigma_C}{\Sigma_A}
=
(1-\alpha)^2
\frac{\xi^4(\rho^2+2a^2)^4}
{\rho^8(\xi^2+4a^2)^2}
+
o(1).
\end{equation}
Equivalently,
\begin{equation}
\Sigma_C
=
\bar\kappa^2\Sigma_A
+
o(n^{-1}),
\qquad
\bar\kappa
=
(1-\alpha)
\frac{\xi^2(\rho^2+2a^2)^2}
{\rho^4(\xi^2+4a^2)}.
\end{equation}
The same relation holds after taking traces:
\begin{equation}
s_C
=
\bar\kappa^2s_A
+
o(n^{-1}).
\end{equation}
Thus, \(\bar\kappa<1\) means that the CANN branch has a smaller leading-order variance than the ANN branch.

Since \(1-\alpha\le1\), a sufficient condition independent of the response time is
\begin{equation}
\frac{\xi^2(\rho^2+2a^2)^2}
{\rho^4(\xi^2+4a^2)}
<1.
\end{equation}
Let \(t=a^2\).
After expanding and canceling common terms, this inequality becomes
\begin{equation}
\xi^2t
<
\rho^2(\rho^2-\xi^2).
\end{equation}
Therefore, when \(\rho^2>\xi^2\), choosing
\begin{equation}
a^2
<
\frac{\rho^2(\rho^2-\xi^2)}{\xi^2}
\end{equation}
is sufficient for \(\bar\kappa<1\).
This proves the lower-variance condition used in Section~\ref{subsec:bias_variance_complementarity}.
\end{proof}

\subsection{Derivations for Functional Synergy through Additive Fusion}
\label{app:functional_synergy_additive_fusion}
\label{app:proof_hybrid_mse}

This section provides the derivations for Section~\ref{subsec:mse_of_additive_hybridization}.
The goal is to justify how convex additive fusion uses the bias--variance complementarity established in Section~\ref{subsec:bias_variance_complementarity} to reduce the leading-order MSE.
We first derive the ANN--CANN cross covariance induced by the shared response-map fluctuation.
We then derive the MSE of the convex additive estimator, the endpoint expansions, and the convex-fusion theorem.

\subsubsection{ANN--CANN cross covariance under shared response-map fluctuations}
\label{app:cross_covariance_additive_fusion}

\begin{proof}[Derivation of the ANN--CANN cross covariance]
Define the centered stochastic errors of the ANN and CANN estimators as
\begin{equation}
\tilde{\bm e}_A
:=
\hat{\bm x}_A-\mathbb E[\hat{\bm x}_A],
\qquad
\tilde{\bm e}_C
:=
\hat{\bm x}_C-\mathbb E[\hat{\bm x}_C].
\end{equation}
Following the notation in Section~\ref{subsec:mse_of_additive_hybridization}, define
\begin{equation}
\Sigma_{AC}
:=
\mathbb E[\tilde{\bm e}_A\tilde{\bm e}_C^\top],
\qquad
s_{AC}:=\operatorname{tr}(\Sigma_{AC}).
\end{equation}
From Proposition~\ref{prop:ann_unbiased_2d_main}, the leading ANN fluctuation is
\begin{equation}
\tilde{\bm e}_A
=
\frac{\rho^2}{M}\nabla\delta_n(\bm 0)
+
o_p(n^{-1/2}).
\end{equation}
From Proposition~\ref{prop:cann_bv_2d_main}, the centered CANN fluctuation is
\begin{equation}
\tilde{\bm e}_C
=
\frac{1-\alpha}{\beta}\gamma_{\mathrm{dyn}}\bm a_n
+
o_p(n^{-1/2}).
\end{equation}
Thus the leading cross covariance is determined by
\(\operatorname{Cov}(\partial_i\delta_n(\bm 0),a_{i,n})\), or equivalently by
\(\operatorname{Cov}(\partial_iG(\bm 0),\langle G,r_i\rangle)\) at the Gaussian-field limit.

For \(i=1\), using the covariance kernel \(K(\bm u,\bm v)=\operatorname{Cov}(G(\bm u),G(\bm v))\), we have
\begin{equation}
\operatorname{Cov}(\partial_1G(\bm 0),\langle G,r_1\rangle)
=
\int_{\mathbb R^2}
\left.
\partial_{u_1}K(\bm u,\bm v)
\right|_{\bm u=\bm 0}
r_1(\bm v)\,\mathrm d\bm v .
\end{equation}
Under Assumption~\ref{assump:local_covariance},
\begin{equation}
K(\bm u,\bm v)
=
\sigma_G^2
\exp\!\left(
-\frac{\|\bm u-\bm v\|^2}{2\xi^2}
\right),
\end{equation}
and therefore
\begin{equation}
\left.
\partial_{u_1}K(\bm u,\bm v)
\right|_{\bm u=\bm 0}
=
\frac{v_1\sigma_G^2}{\xi^2}
\exp\!\left(-\frac{\|\bm v\|^2}{2\xi^2}\right).
\end{equation}
Using
\begin{equation}
r_1(\bm v)
=
\frac{v_1}{\sqrt{2\pi}a^2}
\exp\!\left(-\frac{\|\bm v\|^2}{4a^2}\right),
\end{equation}
direct Gaussian integration gives
\begin{equation}
\operatorname{Cov}(\partial_1G(\bm 0),\langle G,r_1\rangle)
=
\frac{4\sqrt{2\pi}\sigma_G^2a^2\xi^2}
{(\xi^2+2a^2)^2}.
\end{equation}
The same expression holds for \(i=2\), and the off-diagonal terms vanish by symmetry.

Substituting the leading ANN and CANN fluctuations gives the per-coordinate cross covariance
\begin{equation}
\operatorname{Cov}(\tilde e_{A,i},\tilde e_{C,i})
=
\frac{\sigma_G^2}{n}
(1-\alpha)
\frac{\xi^2(\rho^2+2a^2)^2}
{M^2(\xi^2+2a^2)^2}
+
o(n^{-1}),
\qquad i\in\{1,2\}.
\end{equation}
Hence
\begin{equation}
s_{AC}
=
\operatorname{tr}(\Sigma_{AC})
=
\frac{2\sigma_G^2}{n}
(1-\alpha)
\frac{\xi^2(\rho^2+2a^2)^2}
{M^2(\xi^2+2a^2)^2}
+
o(n^{-1}).
\label{eq:sac_appendix}
\end{equation}
Using
\begin{equation}
s_A
=
\frac{2\sigma_G^2}{n}
\frac{\rho^4}{M^2\xi^2}
+
o(n^{-1}),
\end{equation}
and
\begin{equation}
s_C
=
\frac{2\sigma_G^2}{n}
(1-\alpha)^2
\frac{\xi^2(\rho^2+2a^2)^4}
{M^2\rho^4(\xi^2+4a^2)^2}
+
o(n^{-1}),
\end{equation}
we obtain
\begin{equation}
s_{AC}
=
\chi_{AC}\sqrt{s_As_C}
+
o(n^{-1}),
\qquad
\chi_{AC}
=
\frac{\xi^2(\xi^2+4a^2)}
{(\xi^2+2a^2)^2}.
\label{eq:sac_chi_appendix}
\end{equation}
Moreover, \(\chi_{AC}\in(0,1)\), because
\begin{equation}
(\xi^2+2a^2)^2-\xi^2(\xi^2+4a^2)
=
4a^4
>
0.
\end{equation}
Thus, the two estimators share the same finite-sample response-map fluctuation source, but the translational-mode projection keeps their leading cross covariance below full scalar correlation.
\end{proof}

\subsubsection{MSE of convex additive fusion}
\label{app:hybrid_mse_derivation}

\begin{proof}[Derivation of Eq.~\eqref{eq:hybrid_mse_main}]
Let
\begin{equation}
\bm e_A:=\hat{\bm x}_A-\bm x_1,
\qquad
\bm e_C:=\hat{\bm x}_C-\bm x_1.
\end{equation}
From Propositions~\ref{prop:ann_unbiased_2d_main} and~\ref{prop:cann_bv_2d_main},
\begin{equation}
\mathbb E[\bm e_A]
=
o(n^{-1/2}),
\qquad
\mathbb E[\bm e_C]
=
-\alpha\bm\Delta
+
o(n^{-1/2}).
\end{equation}
The convex additive estimator is
\begin{equation}
\hat{\bm x}_H(\lambda)
=
\lambda\hat{\bm x}_A+(1-\lambda)\hat{\bm x}_C,
\qquad
\lambda\in[0,1],
\end{equation}
and its error is
\begin{equation}
\bm e_H(\lambda)
:=
\hat{\bm x}_H(\lambda)-\bm x_1
=
\lambda\bm e_A+(1-\lambda)\bm e_C.
\end{equation}
Thus the leading bias is
\begin{equation}
\operatorname{Bias}_H(\lambda)
=
\mathbb E[\bm e_H(\lambda)]
=
-(1-\lambda)\alpha\bm\Delta
+
o(n^{-1/2}).
\end{equation}
Since \(d=\|\bm\Delta\|^2\), the squared bias term is
\begin{equation}
\|\operatorname{Bias}_H(\lambda)\|^2
=
(1-\lambda)^2\alpha^2d
+
o(n^{-1}).
\end{equation}

For the covariance term, write
\begin{equation}
s_A:=\operatorname{tr}(\Sigma_A),
\qquad
s_C:=\operatorname{tr}(\Sigma_C),
\qquad
s_{AC}:=\operatorname{tr}(\Sigma_{AC}).
\end{equation}
Using the centered errors \(\tilde{\bm e}_A\) and \(\tilde{\bm e}_C\), the centered hybrid error is
\begin{equation}
\tilde{\bm e}_H(\lambda)
=
\lambda\tilde{\bm e}_A+(1-\lambda)\tilde{\bm e}_C.
\end{equation}
Therefore,
\begin{equation}
\operatorname{tr}\!\left(\operatorname{Cov}(\bm e_H(\lambda))\right)
=
\lambda^2s_A
+
(1-\lambda)^2s_C
+
2\lambda(1-\lambda)s_{AC}
+
o(n^{-1}).
\end{equation}
Adding the squared bias and covariance trace gives
\begin{equation}
\operatorname{MSE}_H(\lambda)
=
(1-\lambda)^2\alpha^2d
+
\lambda^2s_A
+
(1-\lambda)^2s_C
+
2\lambda(1-\lambda)s_{AC}
+
o(n^{-1}),
\end{equation}
which is Eq.~\eqref{eq:hybrid_mse_main}.
\end{proof}

\subsubsection{Endpoint expansions and local improvement conditions}
\label{app:endpoint_expansions_additive_fusion}

\begin{proof}[Derivation of the endpoint expansions]
Let \(L_H(\lambda)\) denote the leading quadratic part of \(\operatorname{MSE}_H(\lambda)\):
\begin{equation}
L_H(\lambda)
=
(1-\lambda)^2\alpha^2d
+
\lambda^2s_A
+
(1-\lambda)^2s_C
+
2\lambda(1-\lambda)s_{AC}.
\end{equation}
Near the ANN endpoint, set \(\lambda=1-\epsilon\) with \(0<\epsilon\ll1\).
Then
\begin{equation}
\begin{aligned}
L_H(1-\epsilon)-L_H(1)
&=
\epsilon^2\alpha^2d
+
(1-\epsilon)^2s_A
+
\epsilon^2s_C
+
2(1-\epsilon)\epsilon s_{AC}
-
s_A \\
&=
-2\epsilon(s_A-s_{AC})
+
O(\epsilon^2).
\end{aligned}
\end{equation}
Thus, adding a small CANN contribution improves the ANN endpoint when \(s_{AC}<s_A\).
Under the lower-variance condition \(\bar\kappa<1\), we have \(s_C<s_A\).
Together with Eq.~\eqref{eq:sac_chi_appendix} and \(\chi_{AC}\in(0,1)\), this gives \(s_{AC}<s_A\) at leading order.

Near the CANN endpoint, set \(\lambda=\epsilon\) with \(0<\epsilon\ll1\).
Then
\begin{equation}
\begin{aligned}
L_H(\epsilon)-L_H(0)
&=
(1-\epsilon)^2\alpha^2d
+
\epsilon^2s_A
+
(1-\epsilon)^2s_C
+
2\epsilon(1-\epsilon)s_{AC}
-
(\alpha^2d+s_C) \\
&=
-2\epsilon(\alpha^2d+s_C-s_{AC})
+
O(\epsilon^2).
\end{aligned}
\end{equation}
A sufficient condition for local improvement at the CANN endpoint, independent of the motion magnitude \(d\), is
\begin{equation}
s_{AC}<s_C.
\end{equation}

Using \(s_C=\bar\kappa^2s_A+o(n^{-1})\) and Eq.~\eqref{eq:sac_chi_appendix}, the leading-order condition \(s_{AC}<s_C\) is equivalent to
\begin{equation}
\chi_{AC}<\bar\kappa.
\end{equation}
Substituting
\begin{equation}
\chi_{AC}
=
\frac{\xi^2(\xi^2+4a^2)}
{(\xi^2+2a^2)^2},
\qquad
\bar\kappa
=
(1-\alpha)
\frac{\xi^2(\rho^2+2a^2)^2}
{\rho^4(\xi^2+4a^2)},
\end{equation}
gives
\begin{equation}
\frac{\xi^2(\xi^2+4a^2)}
{(\xi^2+2a^2)^2}
<
(1-\alpha)
\frac{\xi^2(\rho^2+2a^2)^2}
{\rho^4(\xi^2+4a^2)}.
\end{equation}
After canceling the common positive factor \(\xi^2\) and taking square roots, this becomes
\begin{equation}
\frac{\xi^2+4a^2}{\xi^2+2a^2}
<
\sqrt{1-\alpha}
\frac{\rho^2+2a^2}{\rho^2},
\end{equation}
which is Eq.~\eqref{eq:sac_less_sc_condition_main}.
\end{proof}

\subsubsection{Proof of convex additive fusion theorem}
\label{app:proof_convex_additive_fusion_theorem}

\begin{proof}[Proof of Theorem~\ref{thm:optimal_lambda_main}]
The leading quadratic part of the MSE is
\begin{equation}
L_H(\lambda)
=
(1-\lambda)^2\alpha^2d
+
\lambda^2s_A
+
(1-\lambda)^2s_C
+
2\lambda(1-\lambda)s_{AC}.
\end{equation}
Expanding it as a quadratic function of \(\lambda\) gives
\begin{equation}
L_H(\lambda)
=
\alpha^2d+s_C
+
2(s_{AC}-\alpha^2d-s_C)\lambda
+
Q_H\lambda^2,
\end{equation}
where
\begin{equation}
Q_H
=
\alpha^2d+s_A+s_C-2s_{AC}.
\end{equation}
Under the conditions \(s_{AC}<s_A\) and \(s_{AC}<s_C\), we have
\begin{equation}
Q_H
=
(\alpha^2d+s_C-s_{AC})
+
(s_A-s_{AC})
>
0.
\end{equation}
Thus \(L_H(\lambda)\) is strictly convex, and its unconstrained minimizer is obtained from \(\partial L_H(\lambda)/\partial\lambda=0\):
\begin{equation}
\lambda^\star
=
\frac{\alpha^2d+s_C-s_{AC}}
{\alpha^2d+s_A+s_C-2s_{AC}}.
\end{equation}
This gives Eq.~\eqref{eq:lambda_star_main}.

The same two inequalities imply that this minimizer lies inside the convex interval.
The numerator is positive because
\begin{equation}
\alpha^2d+s_C-s_{AC}>0,
\end{equation}
and
\begin{equation}
1-\lambda^\star
=
\frac{s_A-s_{AC}}{Q_H}
>
0.
\end{equation}
Therefore, \(\lambda^\star\in(0,1)\).
Since \(L_H(\lambda)\) is strictly convex and \(\lambda^\star\) is an interior minimizer, its leading MSE is smaller than the leading MSE at both pure-estimator endpoints \(\lambda=0\) and \(\lambda=1\).
The \(o(n^{-1})\) remainder in Eq.~\eqref{eq:hybrid_mse_main} gives the corresponding statement for the leading-order MSE expansion.

It remains to relate the two inequalities \(s_{AC}<s_A\) and \(s_{AC}<s_C\) to the dynamical-parameter conditions.
The inequality \(s_{AC}<s_A\) follows from the lower-variance regime \(\bar\kappa<1\), because this gives \(s_C<s_A\), and Eq.~\eqref{eq:sac_chi_appendix} gives \(s_{AC}=\chi_{AC}\sqrt{s_As_C}+o(n^{-1})\) with \(\chi_{AC}\in(0,1)\).
The inequality \(s_{AC}<s_C\) is guaranteed by Eq.~\eqref{eq:sac_less_sc_condition_main}.

Finally, we show that these requirements can be jointly satisfied.
As \(T\to\infty\), \(\alpha=e^{-\beta T}\to0\), so \(\sqrt{1-\alpha}\to1\).
Thus Eq.~\eqref{eq:sac_less_sc_condition_main} can be satisfied for sufficiently large \(T\) whenever
\begin{equation}
\frac{\xi^2+4a^2}{\xi^2+2a^2}
<
\frac{\rho^2+2a^2}{\rho^2}.
\end{equation}
This inequality is equivalent to
\begin{equation}
a^2
>
\frac{\rho^2-\xi^2}{2}.
\end{equation}
Together with the sufficient condition for \(\bar\kappa<1\),
\begin{equation}
a^2
<
\frac{\rho^2(\rho^2-\xi^2)}{\xi^2},
\end{equation}
we obtain the feasible design range
\begin{equation}
\frac{\rho^2-\xi^2}{2}
<
a^2
<
\frac{\rho^2(\rho^2-\xi^2)}{\xi^2}.
\end{equation}
When \(\rho^2>\xi^2\), this interval is nonempty.
Thus, by choosing \(a\) in this range and then choosing \(T\) sufficiently large, the two inequalities \(s_{AC}<s_A\) and \(s_{AC}<s_C\) can be jointly satisfied.
This proves Theorem~\ref{thm:optimal_lambda_main}.
\end{proof}

\subsection{Derivations for Continuous-Time Stable Tracking}
\label{app:continuous_time_response_map_dynamics}

This section gives the derivations supporting Section~\ref{subsec:continuous_time_response_map_dynamics}.
The derivations follow the three parts of the main text: local continuous-time complementarity and synergy, spurious remote-peak suppression with continuous center evolution, and estimation fusion before final decoding.

\subsubsection{Local continuous-time complementarity and synergy}
\label{app:ct_bv_synergy}

\begin{proof}[Derivation of Eqs.~\eqref{eq:ct_bias_variance_result}--\eqref{eq:ct_lambda_star}]
We consider the local continuous-time regime in which the ANN estimate fluctuates around the ground-truth trajectory:
\begin{equation}
\hat{\bm x}_A(t)
=
\bm x^*(t)+\bm\varepsilon(t),
\qquad
\mathbb E[\bm\varepsilon(t)]=\bm 0,
\end{equation}
with temporal covariance
\begin{equation}
\operatorname{Cov}(\bm\varepsilon(t),\bm\varepsilon(s))
=
\sigma_A^2e^{-\eta_{\varepsilon}|t-s|}I_2.
\end{equation}
In the local small-offset regime, the translational-mode reduction of the CANN center dynamics is approximated by the first-order system
\begin{equation}
\tau_{\mathrm{ct}}\dot{\hat{\bm x}}_C(t)+\hat{\bm x}_C(t)=\hat{\bm x}_A(t),
\end{equation}
where \(\tau_{\mathrm{ct}}>0\) is the effective center-dynamics time constant.

Using the stationary causal solution of this linear system, we have
\begin{equation}
\hat{\bm x}_C(t)
=
\int_0^\infty
\frac{1}{\tau_{\mathrm{ct}}}
e^{-u/\tau_{\mathrm{ct}}}
\hat{\bm x}_A(t-u)\,\mathrm du .
\end{equation}
Substituting \(\hat{\bm x}_A(t)=\bm x^*(t)+\bm\varepsilon(t)\) gives
\begin{equation}
\hat{\bm x}_C(t)
=
\int_0^\infty
\frac{1}{\tau_{\mathrm{ct}}}
e^{-u/\tau_{\mathrm{ct}}}
\bm x^*(t-u)\,\mathrm du
+
\bm\eta_C(t),
\end{equation}
where
\begin{equation}
\bm\eta_C(t)
=
\int_0^\infty
\frac{1}{\tau_{\mathrm{ct}}}
e^{-u/\tau_{\mathrm{ct}}}
\bm\varepsilon(t-u)\,\mathrm du .
\end{equation}

Under the local constant-velocity approximation
\begin{equation}
\bm x^*(t)=\bm x^*(0)+\bm vt,
\end{equation}
we have \(\bm x^*(t-u)=\bm x^*(t)-u\bm v\). Therefore,
\begin{equation}
\begin{aligned}
\int_0^\infty
\frac{1}{\tau_{\mathrm{ct}}}
e^{-u/\tau_{\mathrm{ct}}}
\bm x^*(t-u)\,\mathrm du
&=
\bm x^*(t)
-
\bm v
\int_0^\infty
\frac{u}{\tau_{\mathrm{ct}}}
e^{-u/\tau_{\mathrm{ct}}}
\,\mathrm du \\
&=
\bm x^*(t)-\tau_{\mathrm{ct}}\bm v .
\end{aligned}
\end{equation}
Since \(\mathbb E[\bm\eta_C(t)]=\bm 0\), the ANN and CANN biases relative to \(\bm x^*(t)\) are
\begin{equation}
\operatorname{Bias}_A(t)=\bm 0,
\qquad
\operatorname{Bias}_C(t)=-\tau_{\mathrm{ct}}\bm v .
\end{equation}

We next compute the covariance of the CANN stochastic term.
For one coordinate, write \(\eta_C(t)\) for the corresponding scalar component of \(\bm\eta_C(t)\). Then
\begin{equation}
\operatorname{Var}(\eta_C(t))
=
\sigma_A^2
\int_0^\infty
\int_0^\infty
\frac{1}{\tau_{\mathrm{ct}}^2}
e^{-(u+w)/\tau_{\mathrm{ct}}}
e^{-\eta_{\varepsilon}|u-w|}
\,\mathrm du\,\mathrm dw .
\end{equation}
Let \(U\) and \(W\) be independent exponential random variables with rate \(1/\tau_{\mathrm{ct}}\). The double integral equals
\begin{equation}
\mathbb E\!\left[e^{-\eta_{\varepsilon}|U-W|}\right].
\end{equation}
For two independent exponential variables with the same rate, \(|U-W|\) is exponential with rate \(1/\tau_{\mathrm{ct}}\). Thus,
\begin{equation}
\mathbb E\!\left[e^{-\eta_{\varepsilon}|U-W|}\right]
=
\int_0^\infty
\frac{1}{\tau_{\mathrm{ct}}}
e^{-r/\tau_{\mathrm{ct}}}
e^{-\eta_{\varepsilon}r}
\,\mathrm dr
=
\frac{1}{1+\eta_{\varepsilon}\tau_{\mathrm{ct}}}.
\end{equation}
Hence
\begin{equation}
\operatorname{Cov}(\hat{\bm x}_C(t))
=
\operatorname{Cov}(\bm\eta_C(t))
=
\frac{\sigma_A^2}{1+\eta_{\varepsilon}\tau_{\mathrm{ct}}}I_2.
\end{equation}
Together with
\begin{equation}
\operatorname{Cov}(\hat{\bm x}_A(t))=\sigma_A^2I_2,
\end{equation}
this proves Eq.~\eqref{eq:ct_bias_variance_result}.

The cross covariance between the ANN estimation error and the CANN stochastic term is
\begin{equation}
\operatorname{Cov}(\bm\varepsilon(t),\bm\eta_C(t))
=
\int_0^\infty
\frac{1}{\tau_{\mathrm{ct}}}
e^{-u/\tau_{\mathrm{ct}}}
\operatorname{Cov}(\bm\varepsilon(t),\bm\varepsilon(t-u))
\,\mathrm du .
\end{equation}
Using the exponential temporal covariance gives
\begin{equation}
\operatorname{Cov}(\bm\varepsilon(t),\bm\eta_C(t))
=
\sigma_A^2
\int_0^\infty
\frac{1}{\tau_{\mathrm{ct}}}
e^{-u/\tau_{\mathrm{ct}}}
e^{-\eta_{\varepsilon}u}
\,\mathrm du\, I_2
=
\frac{\sigma_A^2}{1+\eta_{\varepsilon}\tau_{\mathrm{ct}}}I_2.
\end{equation}
Therefore, the scalar variance and cross term in the continuous-time local regime are
\begin{equation}
s_A^{\mathrm{ct}}=2\sigma_A^2,
\qquad
s_C^{\mathrm{ct}}=\frac{2\sigma_A^2}{1+\eta_{\varepsilon}\tau_{\mathrm{ct}}},
\qquad
s_{AC}^{\mathrm{ct}}=\frac{2\sigma_A^2}{1+\eta_{\varepsilon}\tau_{\mathrm{ct}}}.
\end{equation}

For the convex additive estimate
\begin{equation}
\hat{\bm x}_H(t;\lambda)
=
\lambda\hat{\bm x}_A(t)+(1-\lambda)\hat{\bm x}_C(t),
\qquad
\lambda\in[0,1],
\end{equation}
the leading bias is
\begin{equation}
\operatorname{Bias}_H(t;\lambda)
=
-(1-\lambda)\tau_{\mathrm{ct}}\bm v .
\end{equation}
Thus the continuous-time MSE is
\begin{equation}
\operatorname{MSE}^{\mathrm{ct}}_H(\lambda)
=
(1-\lambda)^2\tau_{\mathrm{ct}}^2\|\bm v\|^2
+
\lambda^2s_A^{\mathrm{ct}}
+
(1-\lambda)^2s_C^{\mathrm{ct}}
+
2\lambda(1-\lambda)s_{AC}^{\mathrm{ct}}.
\end{equation}
Using \(s_C^{\mathrm{ct}}=s_{AC}^{\mathrm{ct}}\), this can be written as
\begin{equation}
\operatorname{MSE}^{\mathrm{ct}}_H(\lambda)
=
(1-\lambda)^2\tau_{\mathrm{ct}}^2\|\bm v\|^2
+
\lambda^2
\left(
s_A^{\mathrm{ct}}-s_C^{\mathrm{ct}}
\right)
+
s_C^{\mathrm{ct}}.
\end{equation}
Differentiating with respect to \(\lambda\) gives
\begin{equation}
\frac{\mathrm d}{\mathrm d\lambda}
\operatorname{MSE}^{\mathrm{ct}}_H(\lambda)
=
-2(1-\lambda)\tau_{\mathrm{ct}}^2\|\bm v\|^2
+
2\lambda
\left(
s_A^{\mathrm{ct}}-s_C^{\mathrm{ct}}
\right).
\end{equation}
The stationary point therefore satisfies
\begin{equation}
\lambda^\star_{\mathrm{ct}}
=
\frac{\tau_{\mathrm{ct}}^2\|\bm v\|^2}
{\tau_{\mathrm{ct}}^2\|\bm v\|^2+s_A^{\mathrm{ct}}-s_C^{\mathrm{ct}}}.
\end{equation}
Since
\begin{equation}
s_A^{\mathrm{ct}}-s_C^{\mathrm{ct}}
=
2\sigma_A^2
-
\frac{2\sigma_A^2}{1+\eta_{\varepsilon}\tau_{\mathrm{ct}}}
=
\frac{2\sigma_A^2\eta_{\varepsilon}\tau_{\mathrm{ct}}}
{1+\eta_{\varepsilon}\tau_{\mathrm{ct}}},
\end{equation}
we obtain
\begin{equation}
\lambda^\star_{\mathrm{ct}}
=
\frac{\tau_{\mathrm{ct}}^2\|\bm v\|^2}
{
\tau_{\mathrm{ct}}^2\|\bm v\|^2
+
2\sigma_A^2\eta_{\varepsilon}\tau_{\mathrm{ct}}/(1+\eta_{\varepsilon}\tau_{\mathrm{ct}})
}.
\end{equation}
When \(\|\bm v\|>0\) and \(\eta_{\varepsilon}\tau_{\mathrm{ct}}>0\), both the numerator and the additional denominator term are positive, so \(\lambda^\star_{\mathrm{ct}}\in(0,1)\).
\end{proof}

\subsubsection{Spurious remote-peak suppression and continuous center evolution}
\label{app:remote_peak_stability}

\begin{proof}[Derivation of the remote-peak pull]
Consider a Gaussian component of the ANN response map:
\begin{equation}
I_{\bm\mu,\sigma,A}(\bm x)
=
A\exp\!\left(
-\frac{\|\bm x-\bm\mu\|^2}{2\sigma^2}
\right),
\end{equation}
where \(A>0\) is the amplitude, \(\bm\mu\in\mathbb R^2\) is the peak center, and \(\sigma>0\) is the peak width.
Under the single-bump approximation,
\begin{equation}
U(\bm x,t)
\approx
\tilde U(\bm x-\hat{\bm x}_C(t)),
\qquad
\tilde U(\bm u)
=
U_0
\exp\!\left(-\frac{\|\bm u\|^2}{4a^2}\right).
\end{equation}
Let \(\bm u=\bm x-\hat{\bm x}_C(t)\) and \(\bm q=\bm\mu-\hat{\bm x}_C(t)\). Then
\begin{equation}
I_{\bm\mu,\sigma,A}(\bm u)
=
A
\exp\!\left(
-\frac{\|\bm u-\bm q\|^2}{2\sigma^2}
\right).
\end{equation}

The center motion is obtained by projecting the external-drive component onto the CANN translational modes.
Using the same normalization as in Assumption~\ref{assump:cann_modal_reduction}, this projection gives a contribution of the form
\begin{equation}
\dot{\hat{\bm x}}_C^{(\bm\mu,\sigma,A)}(t)
=
\beta(A,\sigma)
\exp\!\left(
-\frac{\|\bm\mu-\hat{\bm x}_C(t)\|^2}{2(\sigma^2+2a^2)}
\right)
\bigl(\bm\mu-\hat{\bm x}_C(t)\bigr).
\label{eq:gaussian_component_projection_appendix}
\end{equation}
The coefficient is
\begin{equation}
\beta(A,\sigma)
=
\frac{8Aa^2\sigma^2}
{\tau_cU_0(\sigma^2+2a^2)^2}.
\end{equation}

We now derive Eq.~\eqref{eq:gaussian_component_projection_appendix}.
The overlap between the Gaussian response component and the CANN activity bump is
\begin{equation}
F_\sigma(\bm q)
=
\int_{\mathbb R^2}
I_{\bm\mu,\sigma,A}(\bm u)
\tilde U(\bm u)
\,\mathrm d\bm u .
\end{equation}
Multiplying the two Gaussian functions gives
\begin{equation}
F_\sigma(\bm q)
=
\frac{4\pi a^2\sigma^2AU_0}{\sigma^2+2a^2}
\exp\!\left(
-\frac{\|\bm q\|^2}{2(\sigma^2+2a^2)}
\right).
\end{equation}
The product of the two Gaussian functions has mean
\begin{equation}
\frac{2a^2}{\sigma^2+2a^2}\bm q.
\end{equation}
Using
\begin{equation}
\partial_{u_i}\tilde U(\bm u)
=
-\frac{u_i}{2a^2}\tilde U(\bm u),
\end{equation}
the projection onto the \(i\)-th translational direction is proportional to
\begin{equation}
\int_{\mathbb R^2}
u_i
I_{\bm\mu,\sigma,A}(\bm u)
\tilde U(\bm u)
\,\mathrm d\bm u
=
\frac{2a^2}{\sigma^2+2a^2}q_iF_\sigma(\bm q).
\end{equation}
Substituting \(F_\sigma(\bm q)\) and the translational-mode normalization gives Eq.~\eqref{eq:gaussian_component_projection_appendix}.

For the two-peak ANN response map in Eq.~\eqref{eq:two_peak_response_map}, the spurious component has
\begin{equation}
A=A_s,
\qquad
\sigma=\omega_s,
\qquad
\bm\mu=\bm x_s(t).
\end{equation}
Let
\begin{equation}
r_s(t)=\|\bm x_s(t)-\hat{\bm x}_C(t)\|.
\end{equation}
Then the spurious peak contributes
\begin{equation}
\dot{\hat{\bm x}}_{C,s}(t)
=
\beta_s
\exp\!\left(
-\frac{r_s(t)^2}{2(\omega_s^2+2a^2)}
\right)
\bigl(\bm x_s(t)-\hat{\bm x}_C(t)\bigr),
\end{equation}
where
\begin{equation}
\beta_s
=
\frac{8A_sa^2\omega_s^2}
{\tau_cU_0(\omega_s^2+2a^2)^2}.
\end{equation}
Its magnitude is
\begin{equation}
\|\dot{\hat{\bm x}}_{C,s}(t)\|
=
\beta_s r_s(t)
\exp\!\left(
-\frac{r_s(t)^2}{2(\omega_s^2+2a^2)}
\right).
\end{equation}
Define
\begin{equation}
g(r)
=
r
\exp\!\left(
-\frac{r^2}{2(\omega_s^2+2a^2)}
\right).
\end{equation}
Then
\begin{equation}
g'(r)
=
\exp\!\left(
-\frac{r^2}{2(\omega_s^2+2a^2)}
\right)
\left(
1-\frac{r^2}{\omega_s^2+2a^2}
\right).
\end{equation}
Thus \(g(r)\) increases on \(0<r<\sqrt{\omega_s^2+2a^2}\), reaches its maximum at \(r=\sqrt{\omega_s^2+2a^2}\), and decreases for \(r>\sqrt{\omega_s^2+2a^2}\).
This proves that a remote spurious peak affects the CANN estimate through a spatial-overlap-gated velocity contribution rather than through a direct coordinate jump.
\end{proof}

\begin{proof}[Derivation of continuous center evolution]
The reduced CANN center dynamics is an ordinary differential equation whose dominant external-drive components have the form
\begin{equation}
\bm q
\mapsto
\beta
\exp\!\left(
-\frac{\|\bm q\|^2}{2S}
\right)
\bm q,
\qquad
S>0.
\end{equation}
This vector field is smooth in \(\bm q\).
Its Jacobian is
\begin{equation}
D_{\bm q}
\left[
\exp\!\left(
-\frac{\|\bm q\|^2}{2S}
\right)
\bm q
\right]
=
\exp\!\left(
-\frac{\|\bm q\|^2}{2S}
\right)
\left(
I-\frac{\bm q\bm q^\top}{S}
\right),
\end{equation}
which is locally bounded.
Therefore, the reduced center vector field is locally Lipschitz in \(\hat{\bm x}_C\) whenever the response amplitudes and peak locations are locally bounded.

If the response amplitudes and peak locations are locally bounded and measurable in time, and the mode-reduction remainder is locally integrable, the reduced center equation satisfies the standard Carath\'{e}odory conditions.
Hence the CANN center trajectory is locally absolutely continuous:
\begin{equation}
\hat{\bm x}_C\in AC_{\mathrm{loc}}.
\end{equation}
In particular, \(\hat{\bm x}_C(t)\) remains continuous and differentiable almost everywhere, even if the ANN argmax estimate switches discontinuously between response peaks.
\end{proof}

\subsubsection{Estimation fusion before final decoding}
\label{app:estimation_fusion_stable_tracking}

\begin{proof}[Derivation of the local approximation in Eq.~\eqref{eq:estimation_fusion_local_approximation}]
In the local unimodal regime, suppose that the ANN response map and the CANN activity bump are locally concentrated around nearby centers.
Their second-order approximations can be written as
\begin{equation}
S_A(\bm x,t)
\approx
C_A-\frac{h_A}{2}\|\bm x-\hat{\bm x}_A(t)\|^2,
\end{equation}
and
\begin{equation}
U(\bm x,t)
\approx
C_C-\frac{h_C}{2}\|\bm x-\hat{\bm x}_C(t)\|^2,
\end{equation}
where \(h_A>0\) and \(h_C>0\) are local curvatures.
For estimation fusion before decoding, define
\begin{equation}
S_H(\bm x,t)
=
\lambda S_A(\bm x,t)+(1-\lambda)U(\bm x,t),
\qquad
\lambda\in[0,1].
\end{equation}
The local maximizer satisfies
\begin{equation}
\nabla_{\bm x}S_H(\bm x,t)
=
-\lambda h_A(\bm x-\hat{\bm x}_A(t))
-(1-\lambda)h_C(\bm x-\hat{\bm x}_C(t))
=
\bm 0.
\end{equation}
Solving for the local maximizer gives
\begin{equation}
\hat{\bm x}_H(t)
\approx
\frac{
\lambda h_A\hat{\bm x}_A(t)
+
(1-\lambda)h_C\hat{\bm x}_C(t)
}
{
\lambda h_A+(1-\lambda)h_C
}.
\end{equation}
When the two local curvatures are comparable, or when their ratio is absorbed into the fusion weight, this reduces to the convex additive approximation
\begin{equation}
\hat{\bm x}_H(t)
\approx
\lambda\hat{\bm x}_A(t)+(1-\lambda)\hat{\bm x}_C(t).
\end{equation}
This gives the local approximation used in Eq.~\eqref{eq:estimation_fusion_local_approximation}.
\end{proof}

\begin{proof}[Support analysis under a remote spurious peak]
The local approximation above does not apply when the ANN response contains a remote spurious peak.
If the ANN argmax estimate jumps to the distractor while the CANN center remains near the target-side activity bump, directly averaging decoded estimates gives
\begin{equation}
\hat{\bm x}_{\mathrm{dec}}(t)
=
\lambda\hat{\bm x}_A(t)+(1-\lambda)\hat{\bm x}_C(t)
\approx
\lambda\bm x_s(t)+(1-\lambda)\hat{\bm x}_C(t).
\end{equation}
This point can lie between the target-side bump and the remote distractor peak, where neither branch assigns strong response support.

Estimation fusion before final decoding instead forms
\begin{equation}
S_H(\bm x,t)
=
\lambda S_A(\bm x,t)+(1-\lambda)U(\bm x,t),
\qquad
\hat{\bm x}_H(t)
=
\arg\max_{\bm x}S_H(\bm x,t).
\end{equation}
To see how this preserves response support, assume that the CANN activity bump around its current center is locally
\begin{equation}
U(\bm x,t)
=
C_U
\exp\!\left(
-\frac{\|\bm x-\hat{\bm x}_C(t)\|^2}{2\eta_U^2}
\right),
\end{equation}
where \(C_U>0\) is the bump amplitude and \(\eta_U>0\) is the local bump width.
Let
\begin{equation}
R_s(t)=\|\bm x_s(t)-\hat{\bm x}_C(t)\|.
\end{equation}
When \(\hat{\bm x}_C(t)\) remains close to the target-side peak and \(R_s(t)\) is large, ignoring exponentially small cross-tail terms gives
\begin{equation}
S_H(\hat{\bm x}_C(t),t)
\approx
\lambda A_\ast+(1-\lambda)C_U,
\end{equation}
whereas
\begin{equation}
S_H(\bm x_s(t),t)
\approx
\lambda A_s+(1-\lambda)C_U
\exp\!\left(
-\frac{R_s(t)^2}{2\eta_U^2}
\right).
\end{equation}
Thus the target-side fused score remains higher than the remote spurious peak whenever
\begin{equation}
\lambda A_\ast+(1-\lambda)C_U
>
\lambda A_s
+
(1-\lambda)C_U
\exp\!\left(
-\frac{R_s(t)^2}{2\eta_U^2}
\right).
\end{equation}
Equivalently,
\begin{equation}
\lambda(A_s-A_\ast)
<
(1-\lambda)C_U
\left[
1-
\exp\!\left(
-\frac{R_s(t)^2}{2\eta_U^2}
\right)
\right].
\end{equation}
This sufficient inequality shows why the remote ANN peak receives weak support from the CANN activity bump, whereas the target-side response is reinforced near the current CANN center.
Therefore, fusing \(S_A(\bm x,t)\) and \(U(\bm x,t)\) before final decoding preserves response support under remote distractors, while direct averaging of decoded coordinates can produce an unsupported intermediate state.
\end{proof}

\section{Additional Implementation Details}
\label{app:implementation_details}

\paragraph{SiamFC feature branch.}
The SiamFC feature branch follows a fully convolutional Siamese matching design.
Given the initial template patch
\(\bm V_{\mathrm{temp}}^{(0)}\in\mathbb{R}^{3\times127\times127}\)
and the search patch at frame \(k\),
\(\bm V_{\mathrm{search}}^{(k)}\in\mathbb{R}^{3\times255\times255}\),
the two patches are processed by the same convolutional feature extractor \(\varphi(\cdot)\).
We describe the convolutional backbone explicitly in Table~\ref{tab:siamfc_backbone}.

\begin{table}[!htbp]
\centering
\caption{Architecture of the SiamFC feature extractor. BN denotes batch normalization.}
\label{tab:siamfc_backbone}
\begin{tabular}{@{}lcccc@{}}
\toprule
\textbf{Layer} & \textbf{Channels} & \textbf{Kernel} & \textbf{Stride} & \textbf{Extra operations} \\
\midrule
Conv1 & \(3 \rightarrow 96\) & \(11\times11\) & 2 & BN, ReLU, \(3\times3\) max pool, stride 2 \\
Conv2 & \(96 \rightarrow 256\) & \(5\times5\) & 1 & groups=2, BN, ReLU, \(3\times3\) max pool, stride 2 \\
Conv3 & \(256 \rightarrow 384\) & \(3\times3\) & 1 & BN, ReLU \\
Conv4 & \(384 \rightarrow 384\) & \(3\times3\) & 1 & groups=2, BN, ReLU \\
Conv5 & \(384 \rightarrow 256\) & \(3\times3\) & 1 & groups=2 \\
\bottomrule
\end{tabular}
\end{table}

For the template patch, the backbone produces
\(\varphi(\bm V_{\mathrm{temp}}^{(0)})\in\mathbb{R}^{256\times6\times6}\).
For the search patch, it produces
\(\varphi(\bm V_{\mathrm{search}}^{(k)})\in\mathbb{R}^{256\times22\times22}\).
The ANN response map is computed by fully convolutional cross-correlation:
\[
S_A^{(k)}
=
\varphi(\bm V_{\mathrm{temp}}^{(0)})
\star
\varphi(\bm V_{\mathrm{search}}^{(k)}).
\]
This operation is implemented as grouped convolution over the batch dimension and yields a
\(17\times17\) response map.
Following the SiamFC implementation, \(S_A^{(k)}\) is multiplied by a scalar output factor \(10^{-3}\).
During inference, \(S_A^{(k)}\) is upsampled and mapped onto the \(N\times N\) toroidal grid used by the CANN branch.

\paragraph{FlyNet feature branch.}
FlyNet is adapted from \citep{HNN4_FlyNet} as a heuristic HNN feature branch under the same tracking
interface. It replaces the convolutional feature extractor with a sparse
projection and winner-take-all coding module. Each input patch is first converted
to grayscale using luminance weights \((0.299,0.587,0.114)\). The exemplar patch
is center padded to \(255\times255\), and both exemplar and search patches are
area-downsampled by a factor of 4. This gives a flattened vector of dimension
\(63\times63=3969\) for each branch.

The exemplar and search vectors are concatenated into a pair representation of
dimension \(7938\). This vector is projected by a fixed sparse binary projection
matrix into a 512-dimensional representation. For each projected unit, 10\% of
the input dimensions are sampled. A winner-take-all operation then keeps the
largest 50\% of the projected activations and binarizes the code. The resulting
binary code is passed through a three-layer MLP:
\[
  512 \rightarrow 1024 \rightarrow 1024 \rightarrow 17^2 .
\]
The output is reshaped into a \(17\times17\) response map. Thus, FlyNet shares
the same downstream response-map interface as the SiamFC feature branch.

\paragraph{HSTNN temporal branch.}
HSTNN is implemented as a SiamFC tracker augmented with a direct hybrid
RNN-SNN temporal response head. The SiamFC feature extractor and
cross-correlation head are kept fixed. At each frame, the tracker evaluates the
same three search scales as SiamFC, upsamples the response maps to the
\(85\times85\) inference grid, applies the standard scale penalty, and selects
the scale with the largest peak response. The selected response
\(r_t\in\mathbb{R}^{1\times85\times85}\) is min-max normalized before entering
the temporal head.

The temporal head contains two parallel recurrent cells. The first cell is a
convolutional GRU with a \(3\times3\) kernel. The second cell is a convolutional
leaky integrate-and-fire module:
\[
  u_t = \lambda u_{t-1} + W_x * r_t + W_s * s_{t-1}, \quad
  s_t = \mathbb{I}[u_t \geq \theta],
\]
where the reported HSTNN setting uses \(\lambda=0.85\), \(\theta=1.0\), and
16 hidden channels. During training, the spike function is optimized with a
surrogate gradient, and the membrane is reset after spikes. We summarize the
temporal-head architecture in Table~\ref{tab:hstnn_s1_temporal_head}.

\begin{table}[!htbp]
\centering
\caption{Architecture of the HSTNN temporal response head.}
\label{tab:hstnn_s1_temporal_head}
\begin{tabular}{@{}lcccc@{}}
\toprule
\textbf{Module} & \textbf{Input/state} & \textbf{Kernel} & \textbf{Channels} & \textbf{Extra operations} \\
\midrule
ConvGRU & response, hidden state & \(3\times3\) & \(1 \rightarrow 16\) & update/reset gates, tanh candidate \\
ConvLIF & response, membrane/spike state & \(3\times3\) & \(1 \rightarrow 16\) & recurrent spike convolution, reset after spike \\
Fuse Conv1 & GRU hidden + LIF spike & \(1\times1\) & \(32 \rightarrow 16\) & ReLU \\
Fuse Conv2 & fused feature & \(1\times1\) & \(16 \rightarrow 1\) & refined response output \\
\bottomrule
\end{tabular}
\end{table}

The fused output is the refined response map. During
tracking, the argmax of this refined response replaces the SiamFC response peak
for target localization, while the scale and target-size update follow the
standard SiamFC rule.

\paragraph{Training details for feature branches.}
The SiamFC and FlyNet feature branches are trained on the GOT-10k training
split. The SiamFC branch is trained for 50 epochs, and the FlyNet branch is
trained for 10 epochs.
Training uses randomly sampled exemplar-search image pairs. The exemplar crop
has size \(127\times127\), and the search crop has size \(255\times255\). The
training augmentation includes random image stretching and random crop
perturbations around the target.

For SiamFC, the training target is a binary logistic response label on the
\(17\times17\) response map. Positive locations are defined around the response
center, and negative locations are defined outside the positive region. The loss
is a class-balanced binary cross-entropy loss over response-map locations.

For FlyNet, the search crop is randomly translated during training. The target
center inside the translated search crop is used to construct a response-map
label centered at the corresponding location. FlyNet is trained with the same
class-balanced binary cross-entropy objective as SiamFC.

The SiamFC and FlyNet feature branches use stochastic gradient descent with
momentum \(0.9\), weight decay \(5\times10^{-4}\), initial learning rate
\(10^{-3}\), and an exponential learning-rate schedule toward \(10^{-5}\).
Unless otherwise specified, the batch size is 64.

\paragraph{Training details for HSTNN.}
HSTNN is trained on GOT-10k temporal segments after loading the corresponding SiamFC checkpoint. The SiamFC network is frozen, and only the temporal head is
optimized. Each training sample is a contiguous sequence segment. At each
temporal step, the template is taken from the first frame, while the search crop
is centered on the previous ground-truth box. The selected SiamFC response is
normalized and passed through the temporal head. The target is a normalized
Gaussian response centered at the current ground-truth displacement on the
\(85\times85\) response grid.

The temporal objective averages losses over the sequence. For each step, it
uses cross-entropy between the target Gaussian distribution and the softmax of
the refined response, plus a weighted center regression term computed from the
soft-argmax response center. The center loss weight is \(0.1\), and the Gaussian
standard deviation is 2 response-grid cells. The reported HSTNN checkpoint
uses 8-frame segments, 16 hidden channels, batch size 4, AdamW with learning
rate \(10^{-3}\), weight decay \(10^{-4}\), gradient clipping at 10, and the
200-iteration checkpoint.

\paragraph{Evaluation details for GOT-10k and VOT2019.}
The official toolkit do not directly report
\(\mathrm{SR}\) for GOT-10k and VOT2019. Therefore, in addition to the official toolkit outputs,
we recompute \(\mathrm{SR}\) from the saved tracking results to make the
success-rate results comparable across datasets. For GOT-10k, we follow the
official validation-set filtering convention and evaluate only visible frames
after the initial frame. For VOT2019, we parse the unsupervised tracking outputs
and compute the same metric over valid evaluated frames. This post-processing
only augments the reported evaluation metrics and does not change the underlying
tracking results.

\stopappendixtoc

\printcredits

\section*{Declaration of competing interest}
The authors declare that they have no known competing financial interests or personal relationships that could have appeared to influence the work reported in this paper.

\section*{Acknowledgements}
This work was partially supported by Brain Science and Brain-like Intelligence Technology-National Science and Technology Major Project 2021ZD0200300, Hong Kong Polytechnic University under Project P0058180, PP0055934, and PolyU15217625 for GRF project funded in 2025/26 Exercise.

\section*{Declaration of generative AI and AI-assisted technologies in the manuscript preparation process}
During the preparation of this work, the authors used OpenAI Codex to assist with code drafting and editing. After using this tool, the authors reviewed, tested where applicable, and edited the output as needed. The authors take full responsibility for the content of the manuscript.

\bibliographystyle{cas-model2-names}
\bibliography{refs}

\end{document}